\documentclass[11pt,reqno,oneside]{amsart}
\usepackage[latin1]{inputenc}
\usepackage[T1]{fontenc}
\usepackage{amsmath,amsfonts,amscd,amssymb,amsthm}
\usepackage{mathtools}
\usepackage{dsfont}
\usepackage{color}
\usepackage[mathscr]{euscript}
\usepackage{comment}
\usepackage{graphicx}
\usepackage{placeins}
\usepackage{extsizes}
\usepackage{float}
\usepackage{multirow}
\usepackage{pdflscape}
\usepackage{caption}
\usepackage{subcaption}
\usepackage{tikz}
\usepackage{nicefrac}
\usepackage{longtable}
\usetikzlibrary{decorations.pathreplacing}
\definecolor{color1bg}{HTML}{f73d28}
\definecolor{color2bg}{HTML}{FA8072}
\definecolor{bblue}{HTML}{00BFFF}
\definecolor{bblue2}{HTML}{00ffff}

	\usetikzlibrary{arrows,calc}
	\tikzset{
>=stealth',
help lines/.style={dashed, thick},
axis/.style={<->},
important line/.style={thick},
connection/.style={thick, dotted},
}

\usetikzlibrary{shadows}
\usetikzlibrary{backgrounds}
\tikzset{
diagonal fill/.style 2 args={fill=#2, path picture={
		\fill[#1, sharp corners] (path picture bounding box.south west) -|
		(path picture bounding box.north east) -- cycle;}},
reversed diagonal fill/.style 2 args={fill=#2, path picture={
		\fill[#1, sharp corners] (path picture bounding box.north west) |- 
		(path picture bounding box.south east) -- cycle;}}
		}
		\usetikzlibrary{arrows.meta}

\usepackage{rotate}
\usepackage{hyperref}

\makeatletter
\@addtoreset{equation}{section}
\makeatother
\newcounter{as}[section]

\DeclareMathOperator*{\argminA}{arg\,min}

\title[Distribution-free Deviation Bounds]{Distribution-free Deviation Bounds and The Role of Domain Knowledge in Learning via Model Selection with Cross-validation Risk Estimation}

\author{Diego Marcondes}
\address{Mathematical Sciences Institute and France-Australia Mathematical Sciences and Interactions ANU-CNRS International Research Lab, The Australian National University}
\email{\href{diego.marcondes@anu.edu.au}{diego.marcondes@anu.edu.au} }

\author{Cl\'audia Peixoto}
\address{Department of Applied Mathematics, Institute of Mathematics and Statistics, Universidade de S\~{a}o
Paulo, R. do Mat\~{a}o, 1010 - Butant\~{a}, S\~{a}o Paulo - SP,
05508-090, Brazil.}

\allowdisplaybreaks
\newtheorem{theorem}{Theorem}[section]
\newtheorem{remark}[theorem]{Remark}

\newtheorem{definition}[theorem]{Definition}

\newtheorem{corollary}[theorem]{Corollary}
\newtheorem{lemma}[theorem]{Lemma}
\newtheorem{proposition}[theorem]{Proposition}
\newtheorem{example}[theorem]{Example}

\usepackage{lineno}

\begin{document}
	\maketitle
	
	\begin{abstract}
			Cross-validation is one of the most widely used tools for risk estimation and model selection in statistics and machine learning, yet its theoretical properties when embedded in a learning procedure remain insufficiently understood. This paper develops a general, distribution-free framework for learning via model selection with cross-validation risk estimation within classical statistical learning theory. We establish VC dimension-based deviation bounds for the entire learning pipeline, providing detailed proofs for both bounded and unbounded loss functions, the latter requiring a novel extension of existing results. A central focus of the analysis is how the structure of the collection of candidate models influences generalization. To this end, we introduce Learning Spaces as collections of candidate models equipped with a partial order whose inclusion structure reflects increasing model complexity. We show how Learning Spaces can be constructed from domain knowledge and analyze how such structural information increases generalization. The framework is illustrated through case studies and a simulation study in high-dimensional linear regression, comparing learning via model selection in two distinct Learning Spaces against ordinary least squares, LASSO, and ridge regression across scenarios of varying alignment between prior knowledge and the true target. The results demonstrate that, when the Learning Space is well-adapted to the target and an efficient search algorithm is employed, learning via model selection can outperform standard methods by orders of magnitude. Through theoretical insights and concrete examples, we provide guidance on selecting the family of candidate models based on domain knowledge to enhance the performance of model selection with cross-validation.
			
	\end{abstract}
		
	\section{Introduction}
	
	In learning problems, properly choosing a hypotheses space is critical to achieving high generalization, which is more likely when it contains hypotheses with low risk and is of limited complexity relative to the sample size, so overfitting is avoided and low-risk hypotheses can be learned. Properly selecting a hypotheses space involves translating prior information and domain knowledge into properties that low-risk hypotheses should satisfy, and then designing a space containing hypotheses with these properties. This can be done, for example, by identifying invariances, such as a group invariance, associated with the learning problem and considering only hypotheses which respect them \cite{marcondesback}.
	
	If prior information was wrong or weak, or if it could not be properly converted into a relatively simple hypotheses space, then generalization may be low. In order to mitigate these issues, one may select the hypotheses space from data, in what is known in the literature as model selection \cite{raschka2018,ding2018,massart2007}. In model selection, one fixes a collection $\{\mathcal{M}_{1},\dots,\mathcal{M}_{n}\}$ of hypotheses spaces, or models, and an empirical error $\hat{L}(\mathcal{M}_{i})$ for each one, which is often given by an independent validation sample, cross-validation, or complexity penalization. A hypotheses space is selected by minimizing the empirical error, and a hypotheses is learned from it.
	
	If the candidate models are nested, i.e., $\mathcal{M}_{1} \subset \cdots \subset \mathcal{M}_{n}$, then a method based on the Structured Risk Minimization (SRM) Inductive Principle, in which the resubstitution error of the estimated hypotheses of $\mathcal{M}_{i}$ is penalized by its complexity, may be applied to solve this problem (see \cite[Chapter~4]{vapnik2000} for more details and \cite{anguita2012} for an example). More generally, model selection may be performed by penalizing the resubstitution error by the complexity of each model, in both nested and non-nested frameworks (see \cite{massart2007} for an in-depth presentation of model selection by penalization and \cite{koltchinskii2001,koltchinskii2011,arlot2011,bartlett2008} for more specific results). Moreover, the classical problem of variable selection \cite{guyon2003,john1994} constitutes another framework for model selection, in which a partially ordered structured family of constrained hypotheses spaces is generated through the elimination of variables.
	
	Although there is a rich literature about model selection, the methods usually have two shortcomings. First, the collection of candidate models is heuristically selected, and strong domain knowledge is not considered. For example, nested candidate models are selected due to their increasing complexity and not necessarily because their hypotheses have some properties that are believed to be satisfied by low-risk hypotheses. Second, even when domain knowledge is considered, the collection of candidate models does not have a rich structure that could be leveraged to enhance the computational efficiency and generalization of model selection. For example, a collection of nested models does not have as many algebraic properties as a Boolean lattice, which is isomorphic to the collection of candidate models in variable selection. Although variable selection is an example that considers domain knowledge and its collection of candidate models has a lattice structure, it is specific to problems in which low-risk hypotheses do not depend on all variables, and it cannot be readily adapted to other scenarios.
	
	In this paper, we analyze how domain knowledge and prior information can be leveraged to increase generalization by learning via model selection. We focus on model selection with cross-validation risk estimation, in which a hypotheses is learned on the selected model with an independent sample or by reusing the sample used to select the model. We formalize model selection within the statistical learning framework and deduce insightful distribution-free deviation bounds for learning via model selection in this instance. These insights are then illustrated through case studies and a simulation study, in which practical ways of increasing generalization by incorporating domain knowledge into the family of candidate models are presented.
	
	\section{Background}
	
	The classical framework of Statistical Learning Theory, in the case of binary classification, is a triple $(\mathcal{H},\mathbb{A},\mathcal{D}_{N})$, composed of a set $\mathcal{H}$ of functions $h: \mathbb{R}^{d} \to \{0,1\},d \geq 1$, called hypotheses space, and a learning algorithm $\mathbb{A}(\mathcal{H},\mathcal{D}_{N})$, which searches $\mathcal{H}$ to minimize the classification error over a training sample $\mathcal{D}_{N} = \{(X_{1},Y_{1}),\dots,(X_{N},Y_{N})\}$ of a random vector $(X,Y)$, with range $\mathbb{R}^{d} \times \{0,1\}$, $Z$, with range $\mathcal{Z} \subset \mathbb{R}^{d}, d \geq 1,$ and unknown probability distribution $P$. We refer to Section \ref{SecPreliminaries} for more details about the general framework of statistical learning.
	
	Let $\ell(h(x),y) \coloneqq \mathds{1}\{h(x) \neq y\}$ be the 0-1 loss function. The risk of a hypotheses $h \in \mathcal{H}$ is an expected value of the local loss $\ell(h(x),y), x \in \mathbb{R}^{d}, y \in \{0,1\}$. If the expectation is the sample mean of $\ell(h(x),y)$ under $\mathcal{D}_{N}$, we have the empirical risk $L_{\mathcal{D}_{N}}(h)$, while if the expectation of $\ell(h(X),Y)$ is under the distribution $P$, we then have the out-of-sample risk $L(h) = \mathbb{P}(h(X) \neq Y)$. A target hypotheses $h^{\star} \in \mathcal{H}$ is such that its out-of-sample risk is minimum in $\mathcal{H}$, i.e., $L(h^{\star}) \leq L(h), \forall h \in \mathcal{H}$, while an empirical risk minimization (ERM) hypotheses $\hat{h}$ is such that its empirical risk is minimum, i.e., $L_{\mathcal{D}_{N}}(\hat{h}) \leq L_{\mathcal{D}_{N}}(h), \forall h \in \mathcal{H}$.
	
	In this context, learning via model selection is a two-step framework in which first a model is selected from a collection of candidates, and then a hypotheses is learned on it. Fix a hypotheses space $\mathcal{H}$, a collection of candidate subsets, or models,
	\begin{align*}
		\mathbb{C}(\mathcal{H}) = \{\mathcal{M}_{1},\dots,\mathcal{M}_{n}\} & & \text{ such that } & & \mathcal{H} = \bigcup_{i=1}^{n} \mathcal{M}_{i},
	\end{align*}
	and an empirical risk estimator $\hat{L}: \mathbb{C}(\mathcal{H}) \mapsto \mathbb{R}_{+}$, which assigns to each candidate model a risk estimated from data. The first step of learning via model selection is to select a minimizer of $\hat{L}$ in $\mathbb{C}(\mathcal{H})$:
	\begin{equation*}
		\hat{\mathcal{M}} \coloneqq \argminA_{\mathcal{M} \in \mathbb{C}(\mathcal{H})} \hat{L}(\mathcal{M}).
	\end{equation*}
	Once a model is selected, one employs a data-driven algorithm $\mathbb{A}$, for example empirical risk minimization under a training sample, to learn a hypotheses $\hat{h}^{\mathbb{A}}_{\hat{\mathcal{M}}}$ in $\hat{\mathcal{M}}$.
	
	An intrinsic characteristic of model selection frameworks is a bias-variance trade-off, which is depicted in Figure \ref{fig_error}. On the one hand, when one selects from data a model $\hat{\mathcal{M}}$ among the candidates to learn on, one adds a bias to the learning process if $h^{\star}$ is not in $\hat{\mathcal{M}}$, since the best hypotheses $h^{\star}_{\hat{\mathcal{M}}}$ one can learn on $\hat{\mathcal{M}}$ may have a greater risk than $h^{\star}$, i.e., $L(h^{\star}_{\hat{\mathcal{M}}}) > L(h^{\star})$. Hence, even when the sample size is large, the learned hypotheses may not well generalize relative to $h^{\star}$. We call this bias type III estimation error, as depicted in Figure \ref{fig_error}.
	
	On the other hand, learning within $\hat{\mathcal{M}}$ may have a smaller error than on the whole $\mathcal{H}$, in the sense of $L(\hat{h}^{\mathbb{A}}_{\hat{\mathcal{M}}})$ being close to $L(h^{\star}_{\hat{\mathcal{M}}})$. We call their difference type II estimation error, as depicted in Figure \ref{fig_error}. The actual error of learning in this instance is the difference between $L(\hat{h}^{\mathbb{A}}_{\hat{\mathcal{M}}})$ and $L(h^{\star})$, also depicted in Figure \ref{fig_error} as type IV estimation error. 
	
	A successful framework for model selection should be such that $L(\hat{h}^{\mathbb{A}}_{\hat{\mathcal{M}}}) < L(\hat{h}^{\mathbb{A}})$ where $\hat{h}^{\mathbb{A}}$ is obtained by learning on the whole space $\mathcal{H}$ with algorithm $\mathbb{A}$. In other words, by adding a bias (III), the learning variance (II) within $\hat{\mathcal{M}}$ should be low enough, so the actual error (IV) committed when learning is smaller than the one committed by learning on the whole space $\mathcal{H}$, that is, $L(\hat{h}^{\mathbb{A}}) - L(h^{\star})$.
	
	\begin{figure}[ht]
		\centering
		\includegraphics[width=0.75\linewidth]{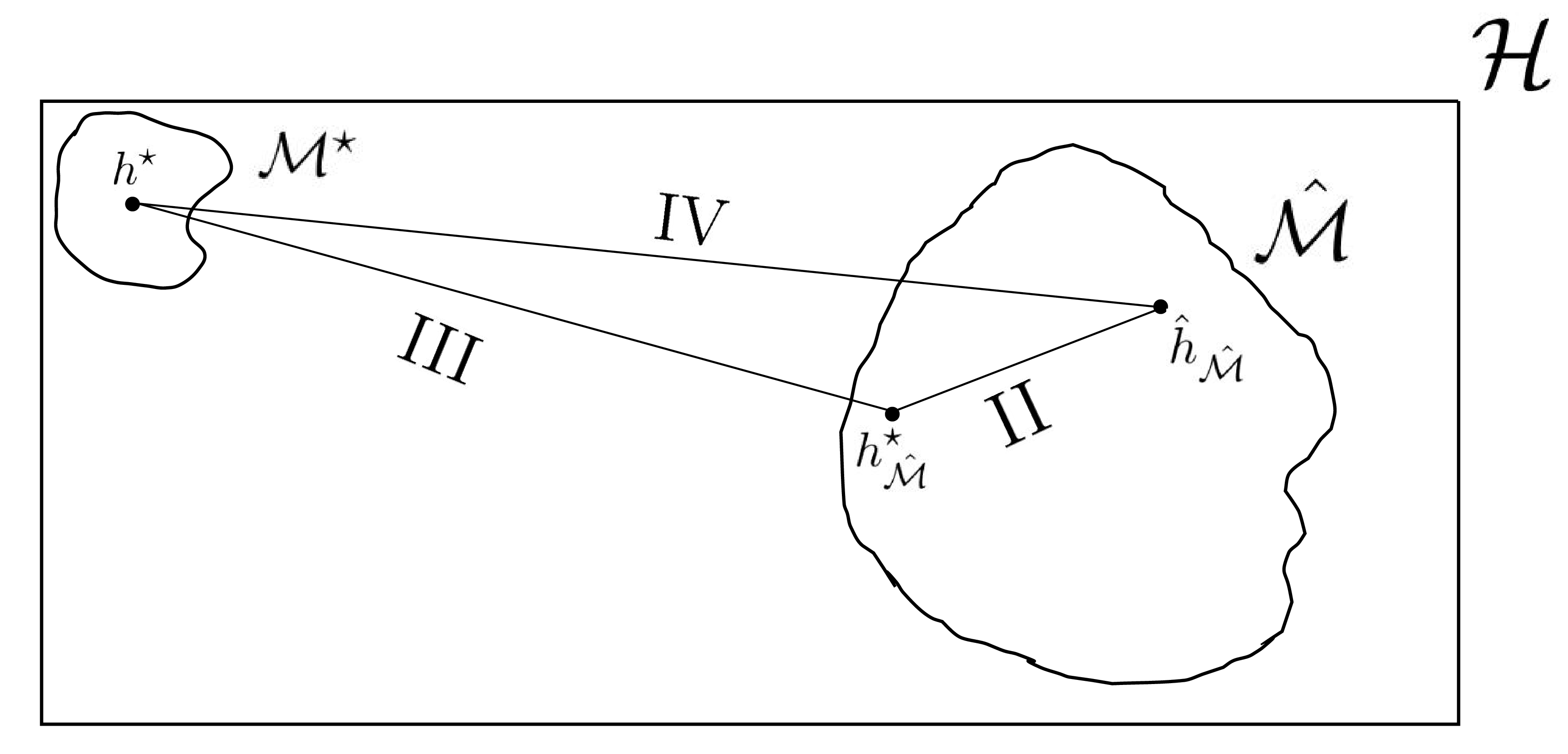}
		\caption{Types II, III, and IV estimation errors when learning on $\hat{\mathcal{M}}$, in which $\hat{h}_{\hat{\mathcal{M}}} \equiv \hat{h}_{\hat{\mathcal{M}}}^{\mathbb{A}}$. These errors are formally defined in Section \ref{SecErrors}.}
		\label{fig_error}
	\end{figure}
	
	In this paper, we study distribution-free asymptotics of learning via model selection, meaning bounding the estimation errors in Figure \ref{fig_error}. We define the target model $\mathcal{M}^{\star}$ among the candidates and establish convergence rates of $\hat{\mathcal{M}}$ to $\mathcal{M}^{\star}$. We focus on frameworks where the risk of each candidate model is estimated via a cross-validation procedure and learning on $\hat{\mathcal{M}}$ is performed with an independent sample. We briefly discuss the case in which the same sample that was used to selected $\hat{\mathcal{M}}$ is reused to learn on it.
	
	\subsection{Related work}
	
	Cross-validation techniques for risk estimation and model selection are widely used in statistics and machine learning, and their asymptotic behavior has been well studied. The asymptotic optimality of specific cross-validation methods has been studied in \cite{andrews1991asymptotic,li1987asymptotic,dudoit2005asymptotics} and so-called oracle inequalities have been established for cross-validation methods in \cite{van2006oracle,lecue2012oracle}. Moreover, some straightforward deviation-bounds in distribution-free scenarios can be found in \cite[Chapter~4]{mohri2018foundations}. However, the understanding of the theoretical properties of cross-validation is quite low in light of its widespread use. This fact has been noted by \cite{austern2020asymptotics,maillard2021local}, which study the asymptotics of cross-validation for specific classes of models. To the best of our knowledge, a systematic study of model selection with cross-validation risk estimation within a distribution-free framework has not been done before, and may bring some insights into the benefits of model selection and have an impact on practical applications.
	
	In this paper, we apply some techniques that have been applied before in the context of model selection via complexity penalization. For instance, results in  \cite[Chapter~4]{mohri2018foundations} present deviation bounds for SRM depending on the VC dimension of the target model, as do ours. However, our results also depend on the minimum discrimination error, that is, informally, the difference between the risk of the target model and the second best.
	
	We note that substantial work has been done in model selection in distribution dependent settings, especially for complexity penalization, as thoroughly presented in \cite{massart2007}. In particular, we highlight \cite{mendelson2004importance}, which, in a distribution-dependent framework, obtains insights about the effect of the complexity of the learned model on the generalization quality of learning via model selection. Distribution-dependent bounds are out of the scope of this paper, but the results presented here might be adapted to distribution-dependent scenarios, a topic we will leave for future studies.
	
	Despite the rich literature about model selection, an important facet of it is rarely treated from a statistical perspective: how the family of candidate models should be chosen. This question, when addressed, is usually done so from a computational perspective by fixing a family which can be efficiently searched. In particular, how generalization can be increased by properly selecting a penalization is a highly studied topic, but how it can be increased by properly selecting the family of candidate models is not. This paper is especially concerned with this neglected topic and how to leverage domain knowledge to increase generalization.
	
	\subsection{Main contributions}
	
	We present learning via model selection with cross-validation risk estimation as a general systematic learning framework within classical Statistical Learning Theory, and establish distribution-free deviation bounds for the estimation errors in terms of VC dimension, giving detailed proofs of the results and considering both bounded and unbounded loss functions. In order to treat the case of bounded loss functions, we extend the classical Statistical Learning Theory \cite{vapnik1998,devroye1996} to learning via model selection, while to treat the case of unbounded loss functions, we apply and extend results of \cite{cortes2019}.
	
	In \cite{cortes2019}, bounds for the tail probabilities of
	\begin{align*}
		\sup\limits_{h \in \mathcal{H}} \frac{L(h) - L_{\mathcal{D}_{N}}(h)}{\sqrt[p]{(L^{p}(h))^{p} + \varsigma}}
	\end{align*}
	are established for unbounded loss functions, for any $\varsigma > 0$, in which $L^{p}(h)$ is the $p$-norm of $\ell(Z,h)$ and $p > 1$ is such that $\sup_{h \in \mathcal{H}} L^{p}(h) < \infty$. We extend this result to $p = 1$ to obtain bounds for the relative type I estimation error
	\begin{align*}
		\sup\limits_{h \in \mathcal{H}} \frac{|L(h) - L_{\mathcal{D}_{N}}(h)|}{L(h) + \varsigma}
	\end{align*}
	so that we can carry out an analysis of model selection with unbounded loss functions. The details about this extension are in Section \ref{ApUnbounded} in Appendix \ref{apVCtheory}.
	
	We also define the Learning Spaces as a class of candidate models in which the partial order by inclusion reflects model complexity. In particular, we focus on Learning Spaces that have a lattice structure and formalize a way of defining them through Learning Space generators. We present some examples and discuss how the Learning Space can be modeled to reflect domain knowledge and prior information about the target hypotheses. We illustrate this through case studies and a simulation study comparing learning via model selection against benchmark methods across scenarios of varying domain knowledge alignment in classification and linear regression. We expect with the theoretical results and these concrete examples to guide practitioners and applied researchers on how to choose the family of candidate models based on domain knowledge to increase generalization.
	
	\subsection{Paper structure}
	
	In Section \ref{SecPreliminaries}, we present the main concepts of Statistical Learning Theory, and define the framework of learning via model selection with cross-validation risk estimation as a systematic general learning framework. In Sections \ref{boundedL} and \ref{SecUnbounded}, we establish deviation bounds for learning via model selection with an independent sample for bounded and unbounded loss functions, respectively. In Section \ref{SecReuse}, we briefly discuss learning via model selection by reusing. In Section \ref{SecEnhanceGen}, we introduce the Learning Spaces, discuss how they can be built based on prior information, and in Section \ref{SecNum} we present numerical simulations to study two concrete examples to better understand the effect on the generalization quality of quantities present in the established deviation-bounds. In particular, we discuss how domain knowledge can be leveraged to increase generalization by learning via model selection and briefly present considerations about computational aspects of learning via model selection. We discuss the main results and perspectives of this paper in Section \ref{FinalRemarks}. In Section \ref{SecProof}, we present the proof of the results. To improve the accessibility of this paper to readers not familiar with VC theory, we present in Appendix \ref{apVCtheory} an overview of its main results that are used in this paper.
	
	\section{Model selection in Statistical Learning}
	\label{SecPreliminaries}
	
	Let $Z$ be a random vector defined on a probability space $(\Omega,\mathcal{S},\mathbb{P})$, with range $\mathcal{Z} \subset \mathbb{R}^{d}, d \geq 1$. Denote $P(z) \coloneqq \mathbb{P}(Z \leq z)$, in which $\leq$ is the component-wise partial order in $\mathbb{R}^{d}$, as the probability distribution of $Z$ at point $z \in \mathcal{Z}$, which we assume unknown, but fixed. Define a sample $\mathcal{D}_{N} = \{Z_{1}, \dots, Z_{N}\}$ as a sequence of independent and identically distributed random vectors, defined on $(\Omega,\mathcal{S},\mathbb{P})$, with distribution $P$.
	
	Let $\mathcal{H}$ be a general set, whose typical element we denote by $h$, which we call hypotheses space. We denote subsets of $\mathcal{H}$ by $\mathcal{M}_{i}$, indexed by the positive integers, i.e., $i \in \mathbb{Z}_{+}$. We may also denote a subset of $\mathcal{H}$ by $\mathcal{M}$ to ease notation. We consider \textit{model} and \textit{subset of} $\mathcal{H}$ as synonyms.
	
	Let $\ell: \mathcal{Z} \times \mathcal{H} \mapsto  \mathbb{R}_{+}$ be a, possibly unbounded, loss function, which represents the loss $\ell(z,h)$ that incurs when one applies a hypotheses $h \in \mathcal{H}$ to \textit{explain} a feature of point $z \in \mathcal{Z}$. Denoting $\ell_{h}(z) \coloneqq \ell(z,h)$ for $z \in \mathcal{Z}$, we assume that, for each $h \in \mathcal{H}$, the composite function $\ell_{h} \circ Z$ is $(\Omega,\mathcal{S})$-measurable.
	
	The risk of a hypotheses $h \in \mathcal{H}$ is defined as
	\begin{equation*}
		L(h) \coloneqq \mathbb{E}[\ell_{h}(Z)] =  \int_{\mathcal{Z}} \ell_{h}(z) \ dP(z),
	\end{equation*}
	in which $\mathbb{E}$ means expectation under $\mathbb{P}$. We define the empirical risk on sample $\mathcal{D}_{N}$ as
	\begin{equation*}
		L_{\mathcal{D}_{N}}(h) \coloneqq \frac{1}{N} \sum_{i=1}^{N} \ell_{h}(Z_{i}),
	\end{equation*}
	that is the empirical mean of $\ell_{h}(Z)$.
	
	In classification problems, when $Z = (X,Y)$ with $X$ taking values in $\mathbb{R}^{d}$ and $Y$ in a finite set of labels, the usual loss function is $\ell_{h}((x,y)) = \mathds{1}\{h(x) \neq y\}$ when $h$ is a function of $x$, so the risk is the respective classification error. In regression problems, when $Y$ takes values in $\mathbb{R}$ the loss function is the quadratic loss $\ell_{h}((x,y)) = (h(x) - y)^2$ and the risk is the respective mean squared error. By considering, for example, the loss function as $\ell_{h}(z) = - \log f(z|h)$ for a probability density function that depends on $h \in \mathcal{H}$, the empirical risk is minus the log-likelihood and empirical risk minimization is equivalent to maximum likelihood.
	
	We denote the set of target hypotheses of $\mathcal{H}$ as
	\begin{equation*}
		h^{\star} \coloneqq \argminA\limits_{h \in \mathcal{H}} L(h),
	\end{equation*}
	that are the hypotheses that minimize $L$ in $\mathcal{H}$, and the set of the target hypotheses of subsets of $\mathcal{H}$ by
	\begin{align*}
		h^{\star}_{i} \coloneqq \argminA\limits_{h \in \mathcal{M}_{i}} L(h) & & h^{\star}_{\mathcal{M}} \coloneqq \argminA\limits_{h \in \mathcal{M}} L(h),
	\end{align*}
	depending on the subset.
	
	The set of hypotheses which minimize the empirical risk in $\mathcal{H}$ is defined as 
	\begin{align}
		\label{ERMH}
		\hat{h}^{\mathcal{D}_{N}} \coloneqq \argminA\limits_{h \in \mathcal{H}} L_{\mathcal{D}_{N}}(h),
	\end{align}	
	while the ones that minimize it in models are denoted by
	\begin{align}
		\label{ERMSub}
		\hat{h}_{i}^{\mathcal{D}_{N}} \coloneqq \argminA\limits_{h \in \mathcal{M}_{i}} L_{\mathcal{D}_{N}}(h) & & \hat{h}_{\mathcal{M}}^{\mathcal{D}_{N}} \coloneqq \argminA\limits_{h \in \mathcal{M}} L_{\mathcal{D}_{N}}(h).
	\end{align}
	We assume the minimum of $L$ and $L_{\mathcal{D}_{N}}$ is achieved in $\mathcal{H}$, and in all subsets of it that we consider throughout this paper, so the sets above are not empty. To ease notation, we may simply denote $\hat{h}$ as a hypotheses that minimizes the empirical risk in $\mathcal{H}$. In general, we denote hypotheses estimated via an algorithm $\mathbb{A}$ in $\mathcal{H}$ and its subsets by $\hat{h}^{\mathbb{A}}, \hat{h}_{i}^{\mathbb{A}}$ and $\hat{h}_{\mathcal{M}}^{\mathbb{A}}$. In the special case when $\mathbb{A}$ is given by empirical risk minimization in sample $\mathcal{D}_{N}$, we have the hypotheses defined in \eqref{ERMH} and \eqref{ERMSub}.
	
	We denote a collection of candidate models by $\mathbb{C}(\mathcal{H}) = \{\mathcal{M}_{i}: i \in \mathcal{J}\}$, for $\mathcal{J} \subset \mathbb{Z}_{+}, \text{\textbar}\mathcal{J}\text{\textbar} < \infty$, and assume that it covers $\mathcal{H}$:
	\begin{equation*}
		\mathcal{H} = \bigcup_{i \in \mathcal{J}} \mathcal{M}_{i}.
	\end{equation*}
	We define the VC dimension under loss function $\ell$ of such a collection as
	\begin{equation*}
		\label{VCdimCand}
		d_{VC}(\mathbb{C}(\mathcal{H}),\ell) \coloneqq \max\limits_{i \in \mathcal{J}} d_{VC}(\mathcal{M}_{i},\ell)
	\end{equation*}
	and assume that $d_{VC}(\mathbb{C}(\mathcal{H}),\ell) < \infty$. Since every model in $\mathbb{C}(\mathcal{H})$ is a subset of $\mathcal{H}$, it follows that $d_{VC}(\mathbb{C}(\mathcal{H}),\ell) \leq d_{VC}(\mathcal{H},\ell)$. In Appendix \ref{apVCtheory}, we review the main concepts of VC theory \cite{vapnik1998}. In particular, we define the VC dimension of a hypotheses space under loss function $\ell$ (cf. Definition \ref{VCdimension}). When the loss function is clear from the context, or not relevant to our argument, we denote the VC dimension simply by $d_{VC}(\mathcal{M})$ for $\mathcal{M} \subseteq \mathcal{H}$ and $d_{VC}(\mathbb{C}(\mathcal{H}))$.
	
	\subsection{Model risk estimation}
	\label{esti_Lhat}
	
	The risk of a model in $\mathbb{C}(\mathcal{H})$ is defined as
	\begin{equation*}
		L(\mathcal{M}) \coloneqq  \min\limits_{h \in \mathcal{M}} L(h) = L(h^{\star}_{\mathcal{M}}),
	\end{equation*}
	for $\mathcal{M} \in \mathbb{C}(\mathcal{H})$, and we consider estimators $\hat{L}(\mathcal{M})$ for $L(\mathcal{M})$ based on cross-validation. We assume that $\hat{L}$ is of the form
	\begin{align}
		\label{form_Lhat}
		\hat{L}(\mathcal{M}) = \frac{1}{m} \ \sum_{j=1}^{m} \ \hat{L}^{(j)}(\hat{h}^{(j)}_{\mathcal{M}}), & & \mathcal{M} \in \mathbb{C}(\mathcal{H}),
	\end{align}
	in which there are $m$ pairs of independent training and validation samples, $\hat{L}^{(j)}$ is the empirical risk under the $j$-th validation sample, and $\hat{h}^{(j)}_{\mathcal{M}}$ is a hypotheses that minimizes the empirical risk in $\mathcal{M}$ under the $j$-th training sample, denoted by $\mathcal{D}_{N}^{(j)}$. We assume independence between samples within a pair $j$, but there may exist dependence between samples of distinct pairs $j,j^{\prime}$. All training and validation samples are subsets of $\mathcal{D}_{N}$. We assume all training samples have a size $N_{t}$ and the validation samples a size $N_{v}$.
	
	In order to exemplify the results obtained in this paper, we consider two estimators of the form \eqref{form_Lhat}, obtained with a validation sample and k-fold cross-validation, which we formally define in Section \ref{SecVal}. When there is no need to specify which estimator of $L(\mathcal{M})$ we are referring to, we simply denote $\hat{L}(\mathcal{M})$ to mean an arbitrary estimator with form \eqref{form_Lhat}.
	
	\subsection{Target model and estimation errors}
	\label{SecErrors}
	
	The motivation for learning via model selection stems from a trade-off inherent to the models in $\mathbb{C}(\mathcal{H})$: simpler models, e.g., those with lower VC dimension, tend to allow more accurate estimation, but at the risk of not containing a hypotheses that performs as well as the best one available in $\mathcal{H}$. In this context, assume that $\mathcal{H}$ is all we have to learn on, and we are not willing to consider any hypotheses outside $\mathcal{H}$. Then, if we could choose, we would like to learn on $\mathcal{M}^{\star}$: \textit{the least complex model in $\mathbb{C}(\mathcal{H})$ which contains a target hypotheses $h^{\star}$}. We call $\mathcal{M}^{\star}$ the target model.
	
	In order to formally define the target model, we need to consider equivalence classes of models, as it is not possible to differentiate some models with the concepts of Statistical Learning Theory. Define in $\mathbb{C}(\mathcal{H})$ the equivalence relation given by
	\begin{align}
		\label{equiv_class}
		\mathcal{M}_{i} \sim \mathcal{M}_{j} \text{ if, and only if, } d_{VC}(\mathcal{M}_{i}) = d_{VC}(\mathcal{M}_{j}) \text{ and } L(\mathcal{M}_{i}) = L(\mathcal{M}_{j}),
	\end{align}
	for $\mathcal{M}_{i}, \mathcal{M}_{j} \in \mathbb{C}(\mathcal{H})$: two models in $\mathbb{C}(\mathcal{H})$ are equivalent if they have the same VC dimension and risk. Let 
	\begin{equation*}
		\mathcal{L}^{\star} = \argminA\limits_{\mathcal{M} \in \ \nicefrac{\mathbb{C}(\mathcal{H})}{\sim}} L(\mathcal{M})
	\end{equation*}
	be the equivalence classes which contain a target hypotheses of $\mathcal{H}$, so their risk is minimum. We define the target model $\mathcal{M}^{\star} \in \nicefrac{\mathbb{C}(\mathcal{H})}{\sim}$ as
	\begin{equation*}
		\mathcal{M}^{\star} = \argminA\limits_{\mathcal{M} \in \mathcal{L}^{\star}} d_{VC}(\mathcal{M}),
	\end{equation*}
	which is the class of the smallest models in $\mathbb{C}(\mathcal{H})$, in the VC dimension sense, that are not disjoint with $h^{\star}$. The target model has the lowest complexity among the unbiased models in $\mathbb{C}(\mathcal{H})$.
	
	The target model is dependent on both $\mathbb{C}(\mathcal{H})$ and the data generating distribution, so we cannot establish beforehand, without looking at data, on which model of $\mathbb{C}(\mathcal{H})$ to learn. Hence, in this context, a model selection procedure should, based on data, learn a model $\hat{\mathcal{M}}$ among the candidates $\mathbb{C}(\mathcal{H})$ as an estimator of $\mathcal{M}^{\star}$.
	
	However, when learning on a model $\hat{\mathcal{M}} \in \mathbb{C}(\mathcal{H})$ selected based on data, one commits three types of errors:
	\begin{align*}
		\label{ee23}
		\textbf{(II)} \ \ L(\hat{h}_{\hat{\mathcal{M}}}^{\mathbb{A}}) - L(h^{\star}_{\hat{\mathcal{M}}}) & & \textbf{(III)} \ \ L(h^{\star}_{\hat{\mathcal{M}}}) - L(h^{\star}) & & \textbf{(IV)} \ \ L(\hat{h}_{\hat{\mathcal{M}}}^{\mathbb{A}}) - L(h^{\star}),
	\end{align*}
	that we call types II, III, and IV estimation errors\footnote{We define type I estimation error in Section \ref{SecLearnOn} (cf. \eqref{typeIe}).}, which are illustrated in Figure \ref{fig_error}. In a broad sense, type III estimation error would represent the bias of learning on $\hat{\mathcal{M}}$, while type II would represent the variance within $\hat{\mathcal{M}}$, and type IV would be the error, with respect to $\mathcal{H}$, committed when learning on $\hat{\mathcal{M}}$ with algorithm $\mathbb{A}$. 
	
	Indeed, type III estimation error compares a target hypotheses $h_{\hat{\mathcal{M}}}^{\star}$ of $\hat{\mathcal{M}}$ with a target hypotheses $h^{\star}$ of $\mathcal{H}$, hence any difference between them would be a systematic bias of learning on $\hat{\mathcal{M}}$ when compared to learning on $\mathcal{H}$. Type II estimation error compares the loss of the estimated hypotheses $\hat{h}_{\hat{\mathcal{M}}}^{\mathbb{A}}$ of $\hat{\mathcal{M}}$ and the loss of its target, assessing how much the estimated hypotheses varies from a target of $\hat{\mathcal{M}}$, while type IV is the effective error committed, since it compares the estimated hypotheses of $\hat{\mathcal{M}}$ with a target of $\mathcal{H}$. 
	
	As is often the case, there will be a bias-variance trade-off that should be minded when learning on $\hat{\mathcal{M}}$, so it is important to guarantee that, when the sample size increases, all the estimation errors tend to zero. Furthermore, it is desired that the learned model $\hat{\mathcal{M}}$ converges to the target model $\mathcal{M}^{\star}$ with probability one, so learning is asymptotically optimal. The proposed learning framework via model selection defined in the next section will take these properties into account.
	
	\subsection{Learning hypotheses via model selection}
	\label{LearningFramework}
	
	Learning via model selection is composed of two steps: first learn a model $\hat{\mathcal{M}}$ from $\mathbb{C}(\mathcal{H})$ and then learn a hypotheses on $\hat{\mathcal{M}}$. In this section, we define $\hat{\mathcal{M}}$ and two algorithms $\mathbb{A}$ to learn on it.
	
	\subsubsection{Learning model $\hat{\mathcal{M}}$}
	
	Model selection is performed by applying a $(\Omega,\mathcal{S})$-measurable function $\mathbb{M}_{\mathbb{C}(\mathcal{H})}$, dependent on $\mathbb{C}(\mathcal{H})$, satisfying
	\begin{equation}
		\label{diagram}
		\omega \in \Omega \xrightarrow{(\mathcal{D}_{N},\hat{L})} (\mathcal{D}_{N}(\omega),\hat{L}(\omega)) \xrightarrow{\ \ \mathbb{M}_{\mathbb{C}(\mathcal{H})} \ \ } \mathcal{\hat{M}}(\omega) \in \mathbb{C}(\mathcal{H}),
	\end{equation}
	which is such that, given $\mathcal{D}_{N}$ and an estimator $\hat{L}$ of the risk of each candidate model, learns a $\mathcal{\hat{M}} \in \mathbb{C}(\mathcal{H})$. Note from (\ref{diagram}) that $\mathcal{\hat{M}}$ is a $(\Omega,\mathcal{S})$-measurable $\mathbb{C}(\mathcal{H})$-valued function, as it is the composition of measurable functions, i.e., $\mathcal{\hat{M}} \coloneqq \mathcal{\hat{M}}_{\mathcal{D}_{N},\hat{L},\mathbb{C}(\mathcal{H})} = \mathbb{M}_{\mathbb{C}(\mathcal{H})}\big(\mathcal{D}_{N},\hat{L}\big)$. Even though $\mathcal{\hat{M}}$ depends on $\mathcal{D}_{N}, \hat{L}$ and $\mathbb{C}(\mathcal{H})$, we drop the subscripts to ease notation. 
	
	We are interested in a $\mathbb{M}_{\mathbb{C}(\mathcal{H})}$ such that
	\begin{equation}
		\label{consistent_LM}
		\mathcal{\hat{M}} = \mathbb{M}_{\mathbb{C}(\mathcal{H})}\big(\mathcal{D}_{N},\hat{L}\big) \xrightarrow{N \rightarrow \infty} \mathcal{M}^{\star} \text{ with probability one}.
	\end{equation}
	In particular, it is desired that the model learned by $\mathbb{M}_{\mathbb{C}(\mathcal{H})}$ be as simple as it can be under the restriction that it converges to the target model. 
	
	A $\mathbb{M}_{\mathbb{C}(\mathcal{H})}$ which satisfies (\ref{consistent_LM}) may be defined by mimicking the definition of $\mathcal{M}^{\star}$, but employing the estimated risk $\hat{L}$ instead of $L$. Define in $\mathbb{C}(\mathcal{H})$ the equivalence relation given by
	\begin{align*}
		\mathcal{M}_{i} \hat{\sim} \mathcal{M}_{j} \text{ if, and only if, } d_{VC}(\mathcal{M}_{i}) = d_{VC}(\mathcal{M}_{j}) \text{ and } \hat{L}(\mathcal{M}_{i}) = \hat{L}(\mathcal{M}_{j}),
	\end{align*}
	for $\mathcal{M}_{i}, \mathcal{M}_{j} \in \mathbb{C}(\mathcal{H})$, which is a random $(\Omega,\mathcal{S})$-measurable equivalence relation. Let
	\begin{equation*}
		\hat{\mathcal{L}} = \argminA\limits_{\mathcal{M} \in \ \nicefrac{\mathbb{C}(\mathcal{H})}{\hat{\sim}}} \hat{L}(\mathcal{M})
	\end{equation*}
	be the classes in $\nicefrac{\mathbb{C}(\mathcal{H})}{\hat{\sim}}$ with the least estimated risk. We call the classes in $\hat{\mathcal{L}}$ the global minima of $\mathbb{C}(\mathcal{H})$. Then, $\mathbb{M}_{\mathbb{C}(\mathcal{H})}$ selects
	\begin{equation}
		\label{Ghat}
		\hat{\mathcal{M}} = \argminA\limits_{\mathcal{M} \in \mathcal{\hat{L}}} d_{VC}(\mathcal{M}),
	\end{equation} 
	the simplest class among the global minima.
	
	\subsubsection{Learning hypotheses on $\hat{\mathcal{M}}$}
	\label{SecLearnOn}
	
	Once $\hat{\mathcal{M}}$ is selected, we need to learn hypotheses on it. In this paper, we focus on learning with an independent sample and briefly discuss learning by reusing, as follows. 
	
	Let $\tilde{\mathcal{D}}_{M} = \{\tilde{Z}_{l}: 1 \leq l \leq M\}$ be a sequence of $M$ independent and identically distributed random vectors with distribution $P$, independent of $\mathcal{D}_{N}$. When learning with an independent sample, we consider
	\begin{align*}
		\hat{h}_{\hat{\mathcal{M}}}^{\tilde{\mathcal{D}}_{M}} \coloneqq \argminA\limits_{h \in \hat{\mathcal{M}}} L_{\tilde{\mathcal{D}}_{M}}(h),
	\end{align*}
	that are the hypotheses which minimize the empirical risk under $\tilde{\mathcal{D}}_{M}$ on $\hat{\mathcal{M}}$. 
	
	Another straightforward way of learning on $\hat{\mathcal{M}}$ is to simply consider
	\begin{align}
		\label{learn_reuse}
		\hat{h}_{\hat{\mathcal{M}}}^{\mathcal{D}_{N}} \coloneqq \argminA\limits_{h \in \hat{\mathcal{M}}} L_{\mathcal{D}_{N}}(h),
	\end{align}
	that are the hypotheses which minimize the empirical error under $\mathcal{D}_{N}$ on $\hat{\mathcal{M}}$. We call this framework \textit{learning by reusing}. Figure \ref{learn_hyp} summarizes these systematic frameworks of learning via model selection.
	
	\begin{figure*}[ht]
		\centering
		\includegraphics[width=\linewidth]{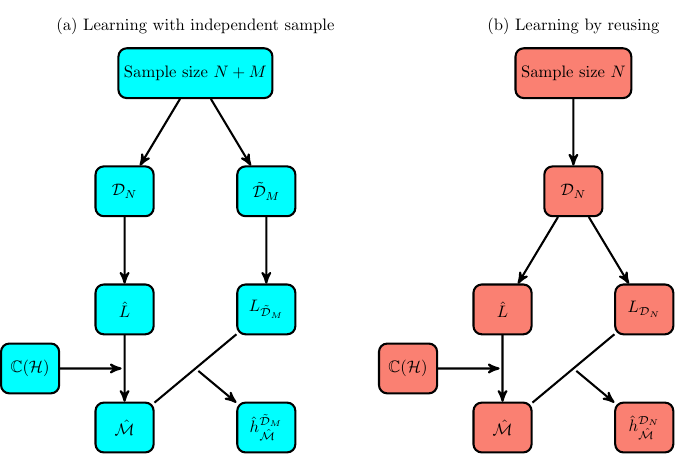}	
		\caption{The systematic frameworks for learning hypotheses via model selection. (a) A sample of size $N+M$ is split into two, one of size $N$ that is used to estimate $\hat{\mathcal{M}}$ by the minimization of $\hat{L}$ on $\mathbb{C}(\mathcal{H})$, and another of size $M$ is used to learn a hypotheses on $\hat{\mathcal{M}}$ by the minimization of the empirical risk. (b) The whole sample of size $N$ is used for estimating $\hat{\mathcal{M}}$ by the minimization of $\hat{L}$ on $\mathbb{C}(\mathcal{H})$, and to estimate hypotheses on $\hat{\mathcal{M}}$ via ERM.}
		\label{learn_hyp}
	\end{figure*}
	
	We define the type I estimation error as
	\begin{align}
		\label{typeIe}
		\textbf{(I)} \begin{cases}
			\sup\limits_{h \in \hat{\mathcal{M}}} \left|L_{\tilde{\mathcal{D}}_{M}}(h) - L(h)\right| & \text{if learning with an independent sample}\\
			\sup\limits_{h \in \hat{\mathcal{M}}} \left|L_{\mathcal{D}_{N}}(h) - L(h)\right| & \text{if learning by reusing}
		\end{cases},
	\end{align}
	which represents how well one can estimate the loss uniformly on $\hat{\mathcal{M}}$ by the empirical risk under $\tilde{\mathcal{D}}_{M}$ and $\mathcal{D}_{N}$, respectively.
	
	\begin{remark}
		We assume that the suprema in \eqref{typeIe} are $(\Omega,\mathcal{S})$-measurable, so it is meaningful to calculate probabilities of events which involve them. We also assume throughout this paper that these suprema, over any fixed $\mathcal{M} \in \mathbb{C}(\mathcal{H}),$ are also $(\Omega,\mathcal{S})$-measurable.
	\end{remark}
	
	\subsection{Deviation bounds and main results}
	
	In classical learning theory, or VC theory, there are two kinds of estimation errors, whose tail probabilities are
	\begin{align}
		\label{GE1}
		\mathbb{P}&\left(\sup\limits_{h \in \mathcal{H}} \text{\textbar}L_{\mathcal{D}_{N}}(h) - L(h)\text{\textbar} > \epsilon\right) 
	\end{align}
	and
	\begin{align}
		\label{GE2}
		\mathbb{P}&\left(L(\hat{h}^{\mathcal{D}_{N}}) - L(h^{\star}) > \epsilon\right),
	\end{align}
	for $\epsilon > 0$. In the terminology of this paper, they are called, respectively, type I and II estimation error, when the target hypotheses of $\mathcal{H}$ are estimated by minimizing the empirical risk under sample $\mathcal{D}_{N}$.
	
	When the loss function is bounded, the rate of convergence of \eqref{GE2} to zero is decreasing on the VC dimension of $\mathcal{H}$. This is the main result of VC theory, which may be stated as follows, and is a consequence of Corollaries \ref{cor3TypeI} and \ref{cor1TypeII}. Observe that the bounds do not depend on $P$, and are valid for any distribution $Z$ may have, that is, are distribution-free. A result analogous to Proposition \ref{propVC} will be stated for unbounded loss functions in Section \ref{SecUnbounded}. In what follows, a.s. stands for almost sure convergence or convergence with probability one.
	
	\begin{proposition}
		\label{propVC}
		Assume the loss function is bounded and fix a hypotheses space $\mathcal{H}$ with $d_{VC}(\mathcal{H}) < \infty$. There exist sequences $\{B^{I}_{N,\epsilon}: N \geq 1\}$ and $\{B^{II}_{N,\epsilon}: N \geq 1\}$ of positive real-valued increasing functions with domain $\mathbb{Z}_{+}$ satisfying
		\begin{equation*}
			\lim\limits_{N \to \infty} B^{I}_{N,\epsilon}(k) = \lim\limits_{N \to \infty} B^{II}_{N,\epsilon}(k) = 0,
		\end{equation*}
		for all $\epsilon > 0$ and $k \in \mathbb{Z}_{+}$ fixed, such that
		\begin{align*}
			\mathbb{P}&\left(\sup\limits_{h \in \mathcal{H}} \text{\textbar}L_{\mathcal{D}_{N}}(h) - L(h)\text{\textbar} > \epsilon\right) \leq B^{I}_{N,\epsilon}(d_{VC}(\mathcal{H}))\\
			\mathbb{P}&\left(L(\hat{h}^{\mathcal{D}_{N}}) - L(h^{\star}) > \epsilon\right) \leq B^{II}_{N,\epsilon}(d_{VC}(\mathcal{H})).
		\end{align*}	
		Furthermore, the following holds:
		\begin{align*}
			&\sup\limits_{h \in \mathcal{H}} \text{\textbar}L_{\mathcal{D}_{N}}(h) - L(h)\text{\textbar} \xrightarrow[N \to \infty]{a.s.} 0 & & L(\hat{h}^{\mathcal{D}_{N}}) - L(h^{\star}) \xrightarrow[N \to \infty]{a.s.} 0.
		\end{align*}
	\end{proposition}
	
	The sequences $\{B^{I}_{N,\epsilon}: N \geq 1\}$ and $\{B^{II}_{N,\epsilon}: N \geq 1\}$ are what we call the deviation bounds of learning via empirical risk minimization in $\mathcal{H}$ and we refer to Appendix \ref{ApTypeI} for explicit formulas for them. The main results of this paper are the convergence of $\hat{\mathcal{M}}$ to $\mathcal{M}^{\star}$ with probability one and deviation bounds for types I, II, III, and IV estimation errors when learning via model selection with bounded or unbounded loss functions.
	
	In particular, it will follow that the established deviation bounds for the type IV estimation error of learning via model selection may be tighter than those for the type II estimation error of learning directly on $\mathcal{H}$ via empirical risk minimization, and hence one may have a lower risk by learning via model selection. This means that by introducing a bias III, which converges to zero, we may decrease the variance II of the learning process, so it is more efficient to learn via model selection. In Section \ref{SecEnhanceGen}, we further discuss when this is the case and how domain knowledge can be leveraged to choose a class of candidate models in which learning via model selection may be, in general, better than learning via ERM in $\mathcal{H}$.
	
	In the following sections, we treat the cases of bounded and unbounded loss functions.
	
	\begin{remark}
		Throughout this paper, we assume that functions such as $B_{N,\epsilon}^{I}(k)$ and $B_{N,\epsilon}^{II}(k)$ are decreasing on $\epsilon$ and $N$ for $k$ fixed.
	\end{remark}
	
	\begin{remark}
		The results of this paper for bounded loss functions hold for any distribution-free complexity, such that Proposition \ref{propVC} remains true. In fact, the results hold by assuming the existence of bounds $B_{N,\epsilon}^{I}(d(\mathcal{H}))$ and $B_{N,\epsilon}^{II}(d(\mathcal{H}))$ for \eqref{GE1} and \eqref{GE2} depending on a complexity measure $d(\mathcal{H})$. For instance, bounds based on the Rademacher and Gaussian complexities \cite{bartlett2002rademacher} or the fat-shattering dimension \cite{bartlett1994fat} could be considered. We stated the results for the VC-dimension to be consistent with the unbounded loss results, since in that case, we could only show that similar bounds hold for the VC dimension (cf. Proposition \ref{propVC2}) by extending results in \cite{cortes2019}.
	\end{remark}
	
	\section{Learning via model selection with bounded loss functions}
	\label{boundedL}
	
	In Section \ref{SecConvTM}, we show the convergence of $\hat{\mathcal{M}}$ to the target model $\mathcal{M}^{\star}$ with probability one, and in Section \ref{SecConvOnMhat} we establish deviation bounds for the estimation errors of learning via model selection when the loss function is bounded. From now on, we assume there exists a constant $C > 0$ such that
	\begin{align*}
		0 \leq \ell(z,h) \leq C & & \text{ for all } z \in \mathcal{Z}, h \in \mathcal{H}.
	\end{align*}
	In this section, we focus on learning with an independent sample and in Section \ref{SecReuse} we briefly present analogous results for learning by reusing.
	
	\subsection{Convergence to the target model}
	\label{SecConvTM}
	
	We start by studying a result weaker than the convergence of $\hat{\mathcal{M}}$ to $\mathcal{M}^{\star}$, that is, the convergence of $L(\hat{\mathcal{M}})$ to $L(\mathcal{M}^{\star})$.
	
	In order to have $L(\hat{\mathcal{M}}) = L(\mathcal{M}^{\star})$, one does not need to know exactly $L(\mathcal{M})$ for all $\mathcal{M} \in \mathbb{C}(\mathcal{H})$, i.e., one does not need $\hat{L}(\mathcal{M}) = L(\mathcal{M})$, for all $\mathcal{M} \in \mathbb{C}(\mathcal{H})$. We argue that it suffices to have $\hat{L}(\mathcal{M})$ close enough to $L(\mathcal{M})$, for all $\mathcal{M} \in \mathbb{C}(\mathcal{H})$, so the global minima of $\mathbb{C}(\mathcal{H})$ will have the same risk as $\mathcal{M}^{\star}$, even if it is not possible to properly estimate their risk. This ``close enough'' depends on $P$, hence is not distribution-free, and is given by the \textit{maximum discrimination error} (MDE) of $\mathbb{C}(\mathcal{H})$ under $P$, which we define as
	\begin{equation*}
		\epsilon^{\star} = \epsilon^{\star}(\mathbb{C}(\mathcal{H}),P) \coloneqq \min\limits_{\substack{\mathcal{M} \in \mathbb{C}(\mathcal{H})\\L(\mathcal{M}) > L(\mathcal{M}^{\star})}} L(\mathcal{M}) - L(\mathcal{M}^{\star}).
	\end{equation*}
	
	The MDE is the minimum difference between the out-of-sample risk of a target hypotheses and the best hypotheses in a model which does not contain a target. In other words, it is the difference between the risk of the best model $\mathcal{M}^{\star}$ and the second best. The meaning of $\epsilon^{\star}$ is depicted in Figure \ref{epsilonstar}.
	
	\begin{figure}[ht]
		\centering
		\includegraphics[width=\linewidth]{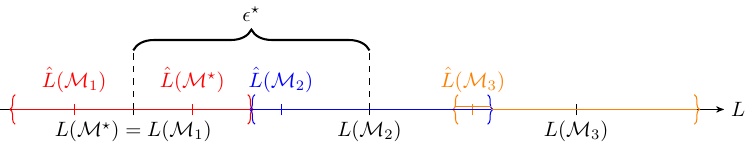}
		\caption{The risks of the equivalence classes (cf. \eqref{equiv_class}) of $\mathbb{C}(\mathcal{H})$ in ascending order. The MDE $\epsilon^{\star}$ is the difference between the risk of the target class $\mathcal{M}^{\star}$, and the second best $\mathcal{M}_{2}$. The colored intervals represent a distance of $\epsilon^{\star}/2$ from the out-of-sample risk of each model, and the colored estimated risks $\hat{L}$ illustrate a case such that the estimated risk is within $\epsilon^{\star}/2$ of the out-of-sample risk for all models. The class $\mathcal{M}_{1}$ has the same risk as $\mathcal{M}^{\star}$, but has a smaller estimated risk, and, by the definition of $\mathcal{M}^{\star}$, greater VC dimension. Note from the representation that, if one can estimate $\hat{L}$ within a margin of error of $\epsilon^{\star}/2$, then $\hat{\mathcal{M}}$ will be a model with the same risk as $\mathcal{M}^{\star}$, in this case $\mathcal{M}_{1}$ (cf. Proposition \ref{proposition_principal}).} \label{epsilonstar}
	\end{figure}
	
	The MDE is defined only if there exists at least one $\mathcal{M} \in \mathbb{C}(\mathcal{H})$ such that $h^{\star} \cap \mathcal{M} = \emptyset$, i.e., there is a subset in $\mathbb{C}(\mathcal{H})$ that does not contain a target hypotheses. If $h^{\star} \cap \mathcal{M} \neq \emptyset$ for all $\mathcal{M} \in \mathbb{C}(\mathcal{H})$, then type III estimation error is zero, and type IV reduces to type II. From this point, we assume that $\epsilon^{\star}$ is well-defined.
	
	The terminology MDE is used because we can show that a fraction of $\epsilon^{\star}$ is the greatest error one can commit when estimating $L(\mathcal{M})$ by $\hat{L}(\mathcal{M})$, for all $\mathcal{M} \in \mathbb{C}(\mathcal{H})$, in order for $L(\mathcal{\hat{M}})$ to be equal to $L(\mathcal{M}^{\star})$. This is the result of the next proposition.
	
	\begin{proposition}
		\label{proposition_principal}
		Assume there exists $\delta > 0$ such that
		\begin{equation}
			\label{cond_prop_principal}
			\mathbb{P}\left(\max\limits_{i \in \mathcal{J}} \text{\textbar}L(\mathcal{M}_{i}) - \hat{L}(\mathcal{M}_{i})\text{\textbar} < \epsilon^{\star}/2\right) \geq 1 - \delta.
		\end{equation}
		Then
		\begin{equation}
			\label{prob_equal}
			\mathbb{P}\left(L(\mathcal{\hat{M}}) = L(\mathcal{M}^{\star})\right) \geq 1-\delta.
		\end{equation}	
	\end{proposition}

	\begin{remark}
		Since there may exist $\mathcal{M} \in \ \nicefrac{\mathbb{C}(\mathcal{H})}{\sim}$ with $L(\mathcal{M}) = L(\mathcal{\mathcal{M}^{\star}})$ and $d_{VC}(\mathcal{M}) > d_{VC}(\mathcal{M}^{\star})$, condition \eqref{cond_prop_principal} guarantees only that the estimated risk of both $\mathcal{M}$ and $\mathcal{M}^{\star}$ is smaller than the estimated risk of any model with risk greater than theirs, but it may happen that $\hat{L}(\mathcal{M}) < \hat{L}(\mathcal{M}^{\star})$ (see Figure \ref{epsilonstar} for an example). In this instance, we have $\hat{\mathcal{M}} = \mathcal{M}$ and $L(\hat{\mathcal{M}}) = L(\mathcal{M}^{\star})$.
	\end{remark}
	
	Recall we are assuming that $\hat{L}$ is of the form \eqref{form_Lhat}. In this case, we may obtain a bound for \eqref{prob_equal} depending on $\epsilon^{\star}$, on $d_{VC}(\mathbb{C}(\mathcal{H}))$, and on bounds for tail probabilities of type I estimation error under each validation and training sample (cf. Proposition \ref{propVC}). These bounds also depend on the number of maximal models of $\mathbb{C}(\mathcal{H})$, which are models in
	\begin{align*}
		\text{Max } \mathbb{C}(\mathcal{H}) = \left\{\mathcal{M} \in \mathbb{C}(\mathcal{H}): \text{ if } \mathcal{M} \subset \mathcal{M}^{\prime} \in \mathbb{C}(\mathcal{H}) \text{ then } \mathcal{M} = \mathcal{M}^{\prime} \right\},
	\end{align*}
	i.e., models not contained in any element of $\mathbb{C}(\mathcal{H})$ besides themselves. We denote
	\begin{equation*}
		\mathfrak{m}(\mathbb{C}(\mathcal{H})) = \text{\textbar}\text{Max } \mathbb{C}(\mathcal{H})\text{\textbar}
	\end{equation*}
	the number of maximal models in $\mathbb{C}(\mathcal{H})$. We have the following rate of convergence of $L(\hat{\mathcal{M}})$ to $L(\mathcal{M}^{\star})$, and condition for $\hat{\mathcal{M}}$ to converge to $\mathcal{M}^{\star}$ with probability one.
	
	\begin{theorem}
		\label{theorem_principal_convergence}
		Assume the loss function is bounded. For each $\epsilon > 0$, let $\{B_{N,\epsilon}: N \geq 1\}$ and $\{\hat{B}_{N,\epsilon}: N \geq 1\}$ be sequences of positive real-valued increasing functions with domain $\mathbb{Z}_{+}$ satisfying
		\begin{equation*}
			\lim\limits_{N \to \infty} B_{N,\epsilon}(k) = \lim\limits_{N \to \infty} \hat{B}_{N,\epsilon}(k) = 0,
		\end{equation*}
		for all $\epsilon > 0$ and $k \in \mathbb{Z}_{+}$ fixed, and such that
		\begin{align*}
			\max_{j} \mathbb{P}\left(\sup\limits_{h \in \mathcal{M}} \text{\textbar}L_{\mathcal{D}_{N}^{(j)}}(h) - L(h)\text{\textbar} > \epsilon \right) \leq B_{N_{t},\epsilon}(d_{VC}(\mathcal{M})) & \text{ and } &\\ \nonumber 
			\max_{j} \mathbb{P}\left(\sup\limits_{h \in \mathcal{M}} \text{\textbar}\hat{L}^{(j)}(h) - L(h)\text{\textbar} > \epsilon \right) \leq \hat{B}_{N_{v},\epsilon}(d_{VC}(\mathcal{M})), & & 
		\end{align*}
		for all $\mathcal{M} \in \mathbb{C}(\mathcal{H})$, recalling that $L_{\mathcal{D}_{N}^{(j)}}$ and $\hat{L}^{(j)}$ represent the empirical risk under the $j$-th training and validation samples, respectively. Let $\mathcal{\hat{M}} \in \mathbb{C}(\mathcal{H})$ be a random model learned by $\mathbb{M}_{\mathbb{C}(\mathcal{H})}$. Then,
		\begin{align}
			\label{bound_pM} \nonumber
			\mathbb{P}\left(L(\hat{\mathcal{M}}) \neq L(\mathcal{M}^{\star})\right) &\leq m \sum_{\mathcal{M} \in \text{Max } \mathbb{C}(\mathcal{H})} \left[B_{N_{t},\epsilon^{\star}/8}(d_{VC}(\mathcal{M})) + \hat{B}_{N_{v},\epsilon^{\star}/4}(d_{VC}(\mathcal{M}))\right]\\
			&\leq m \ \mathfrak{m}(\mathbb{C}(\mathcal{H})) \left[B_{N_{t},\epsilon^{\star}/8}(d_{VC}(\mathbb{C}(\mathcal{H}))) + \hat{B}_{N_{v},\epsilon^{\star}/4}(d_{VC}(\mathbb{C}(\mathcal{H})))\right],
		\end{align}
		in which $m$ is the number of pairs considered to calculate \eqref{form_Lhat}. Furthermore, if
		\begin{align}
			\label{as_conv}
			\max_{\mathcal{M} \in \mathbb{C}(\mathcal{H})} \max_{j} \sup\limits_{h \in \mathcal{M}} \text{\textbar}L_{\mathcal{D}_{N}^{(j)}}(h) - L(h)\text{\textbar} \xrightarrow{\text{a.s.}} 0 & \text{ and } &\\ \nonumber
			\max_{\mathcal{M} \in \mathbb{C}(\mathcal{H})} \max_{j} \sup\limits_{h \in \mathcal{M}} \text{\textbar}\hat{L}^{(j)}(h) - L(h)\text{\textbar} \xrightarrow{\text{a.s.}} 0, & & 
		\end{align}
		then
		\begin{equation*}
			\lim_{N \rightarrow \infty} \mathbb{P}\left(\hat{\mathcal{M}} = \mathcal{M}^{\star}\right) = 1.
		\end{equation*}
	\end{theorem}
	
	A bound for $\mathbb{P}(L(\hat{\mathcal{M}}) \neq L(\mathcal{M}^{\star}))$, and the almost sure convergence of $\hat{\mathcal{M}}$ to $\mathcal{M}^{\star}$ in the case of k-fold cross-validation, follow from Proposition \ref{propVC} by taking $N = N_{t} + N_{v}$ the total of the training and validation samples, with both sample sizes tending to infinity. Analogously, we may obtain a bound when an independent validation sample is considered. This result is stated in the next theorem.
	
	\begin{theorem}
		\label{CVModelconvergence}
		Assume the loss function is bounded. If $\hat{L}$ is given by k-fold cross-validation or by an independent validation sample, then $\hat{\mathcal{M}}$ converges with probability one to $\mathcal{M}^{\star}$.
	\end{theorem}
	
	Since $B_{N,\epsilon}(k)$ and $\hat{B}_{N,\epsilon}(k)$ are decreasing on $\epsilon$ and $N$ for $k$ fixed, it follows from \eqref{bound_pM} that a tighter bound for $\mathbb{P}(L(\hat{\mathcal{M}}) \neq L(\mathcal{M}^{\star}))$ is obtained if the training sample size is greater than the validation sample size ($N_{t} > N_{v}$) since the functions $B_{N_{t},\epsilon^{\star}/8}$ and $\hat{B}_{N_{v},\epsilon^{\star}/4}$ appear on the bound. Observe that this is the case in k-fold cross-validation.
	
	Moreover, from this bound it follows that, with a fixed sample size, we can have a tighter bound for $\mathbb{P}(L(\hat{\mathcal{M}}) \neq L(\mathcal{M}^{\star}))$ by choosing a family of candidate models with small $d_{VC}(\mathbb{C}(\mathcal{H}))$ and few maximal elements, while attempting to increase $\epsilon^{\star}$. Of course, there is a trade-off between $d_{VC}(\mathbb{C}(\mathcal{H}))$ and the number of maximal elements of $\mathbb{C}(\mathcal{H})$, the only known free quantities in bound $\eqref{bound_pM}$, since the sample size is fixed and $\epsilon^{\star}$ is unknown.
	
	\subsection{Deviation bounds for estimation errors on $\hat{\mathcal{M}}$}
	\label{SecConvOnMhat}
	
	Bounds for types I and II estimation errors when learning on a random model with a sample independent of the one employed to compute such random model, may be obtained when there is a bound for them on each $\mathcal{M} \in \mathbb{C}(\mathcal{H})$ under the independent sample. This is the content of Theorem \ref{bound_constant}.
	
	\begin{theorem}
		\label{bound_constant}	
		Fix a bounded loss function. Assume we are learning with an independent sample $\tilde{\mathcal{D}}_{M}$, and that for each $\epsilon > 0$ there exist sequences $\{B^{I}_{M,\epsilon}: M \geq 1\}$ and $\{B^{II}_{M,\epsilon}: M \geq 1\}$ of positive real-valued increasing functions with domain $\mathbb{Z}_{+}$ satisfying
		\begin{equation*}
			\lim\limits_{M \to \infty} B^{I}_{M,\epsilon}(k) = \lim\limits_{M \to \infty} B^{II}_{M,\epsilon}(k) = 0,
		\end{equation*}
		for all $\epsilon > 0$ and $k \in \mathbb{Z}_{+}$ fixed, such that
		\begin{align}
			\label{bound_theoremBC}
			\mathbb{P}\left(\sup\limits_{h \in \mathcal{M}} \text{\textbar}L_{\tilde{\mathcal{D}}_{M}}(h) - L(h) \text{\textbar} > \epsilon \right) \leq B^{I}_{M,\epsilon}(d_{VC}(\mathcal{M})) & \text{ and } &\\ \nonumber 
			\mathbb{P}\left(L(\hat{h}_{\mathcal{M}}^{\tilde{\mathcal{D}}_{M}}) - L(h^{\star}_{\mathcal{M}}) > \epsilon \right) \leq B^{II}_{M,\epsilon}(d_{VC}(\mathcal{M})), & & 
		\end{align}
		for all $\mathcal{M} \in \mathbb{C}(\mathcal{H})$. Let $\mathcal{\hat{M}} \in \mathbb{C}(\mathcal{H})$ be a random model learned by $\mathbb{M}_{\mathbb{C}(\mathcal{H})}$. Then, for any $\epsilon > 0$,
		\begin{align*}
			\textbf{(I)} \ \mathbb{P}&\left(\sup\limits_{h \in \mathcal{\hat{M}}} \text{\textbar}L_{\tilde{\mathcal{D}}_{M}}(h) - L(h) \text{\textbar} > \epsilon \right) \leq \mathbb{E}\Big[B^{I}_{M,\epsilon}(d_{VC}(\mathcal{\hat{M}}))\Big] \leq B^{I}_{M,\epsilon}\left(d_{VC}(\mathbb{C}(\mathcal{H}))\right)
		\end{align*}
		and
		\begin{align*}
			\textbf{(II)} \ \mathbb{P}\left(L(\hat{h}_{\mathcal{\hat{M}}}^{\tilde{\mathcal{D}}_{M}}) - L(h^{\star}_{\mathcal{\hat{M}}}) > \epsilon \right) \leq \mathbb{E}\Big[B^{II}_{M,\epsilon}(d_{VC}(\mathcal{\hat{M}}))\Big] \leq B^{II}_{M,\epsilon}\left(d_{VC}(\mathbb{C}(\mathcal{H}))\right),
		\end{align*}
		in which the expectations are over all samples $\mathcal{D}_{N}$, from which $\hat{\mathcal{M}}$ is calculated. Since $d_{VC}(\mathbb{C}(\mathcal{H})) < \infty$, both probabilities above converge to zero when $M \to \infty$.
	\end{theorem}
	
	Our definition of $\mathcal{\hat{M}}$ ensures that it has the smallest VC dimension under the constraint that it is a global minimum of $\mathbb{C}(\mathcal{H})$. As the quantities inside the expectations of Theorem \ref{bound_constant} are increasing functions of VC dimension, fixed $\epsilon$ and $M$, we tend to have smaller expectations, thus tighter bounds for types I and II estimation errors. Furthermore, it follows from Theorem \ref{bound_constant} that the sample complexity needed to learn on $\hat{\mathcal{M}}$ is at most that of $d_{VC}(\mathbb{C}(\mathcal{H}))$. This implies that this complexity is at most that of $\mathcal{H}$, but may be much smaller if $d_{VC}(\mathbb{C}(\mathcal{H})) \ll d_{VC}(\mathcal{H})$. 
	
	A bound for type III estimation error may be obtained using methods similar to those we employed to prove Theorem \ref{theorem_principal_convergence}. As in that theorem, the bound for type III estimation error depends on $\epsilon^{\star}$, on bounds for type I estimation error under each training and validation sample, and on $\mathbb{C}(\mathcal{H})$, more specifically, on its VC dimension and number of maximal elements. To ease notation, we denote $\epsilon \vee \epsilon^{\star} \coloneqq \max \{\epsilon,\epsilon^{\star}\}$ for any $\epsilon > 0$.
	
	\begin{theorem}
		\label{theorem_tipeIII}
		Assume the premises of Theorem \ref{theorem_principal_convergence} are in force. Let $\mathcal{\hat{M}} \in \mathbb{C}(\mathcal{H})$ be a random model learned by $\mathbb{M}_{\mathbb{C}(\mathcal{H})}$. Then, for any $\epsilon > 0$,
		\begin{align*}
			\textbf{(III)} &\ \mathbb{P}\left(L(h_{\hat{\mathcal{M}}}^{\star}) - L(h^{\star}) > \epsilon\right) \\
			& \leq m \sum_{\mathcal{M} \in \text{Max } \mathbb{C}(\mathcal{H})} \left[ B_{N_{t},(\epsilon \vee \epsilon^{\star})/8}(d_{VC}(\mathcal{M})) + \hat{B}_{N_{v},(\epsilon \vee \epsilon^{\star})/4}(d_{VC}(\mathcal{M}))\right]\\
			&\leq m \ \mathfrak{m}(\mathbb{C}(\mathcal{H})) \left[ B_{N_{t},(\epsilon \vee \epsilon^{\star})/8}(d_{VC}(\mathbb{C}(\mathcal{H}))) + \hat{B}_{N_{v},(\epsilon \vee \epsilon^{\star})/4}(d_{VC}(\mathbb{C}(\mathcal{H})))\right].
		\end{align*}
		In particular,
		\begin{align*}
			\lim_{N \rightarrow \infty} \mathbb{P}\left(L(h_{\hat{\mathcal{M}}}^{\star}) - L(h^{\star}) > \epsilon\right) = 0,
		\end{align*}
		for any $\epsilon > 0$.
	\end{theorem}
	
	\begin{remark}
		\label{remReuse}
		Type III estimation error, and its bound presented in Theorem \ref{theorem_tipeIII}, do not depend on the algorithm $\mathbb{A}$ employed to learn on $\hat{\mathcal{M}}$, hence this theorem is true for both frameworks in Figure \ref{learn_hyp}, holding also when learning by reusing.
	\end{remark}
	
	On the one hand, by definition of $\epsilon^{\star}$, if $\epsilon < \epsilon^{\star}$, then type III estimation error is smaller than $\epsilon$ if, and only if, $L(\hat{\mathcal{M}}) = L(\mathcal{M}^{\star})$, so this error is actually zero, and the result of Theorem \ref{theorem_principal_convergence} is a bound for type III estimation error in this case. On the other hand, if $\epsilon > \epsilon^{\star}$, one way of having type III estimation error smaller than $\epsilon$ is to have the estimated risk of each $\mathcal{M}$ at a distance at most $\epsilon/2$ from its out-of-sample risk and, as can be inferred from the proof of Theorem \ref{theorem_principal_convergence}, this can be accomplished if one has type I estimation error not greater than a fraction of $\epsilon$ under each training and validation sample considered, so a modification of Theorem \ref{theorem_principal_convergence} applies to this case.
	
	Finally, as the tail probability of type IV estimation error may be bounded by the following inequality, involving the tail probabilities of types II and III estimation errors,
	\begin{align}
		\label{triangle} \nonumber
		\textbf{(IV)} \ \mathbb{P}&\left(L(\hat{h}_{\mathcal{\hat{M}}}^{\tilde{\mathcal{D}}_{M}}) - L(h^{\star}) > \epsilon\right)\\& \leq \mathbb{P}\left(L(\hat{h}_{\mathcal{\hat{M}}}^{\tilde{\mathcal{D}}_{M}}) - L(h^{\star}_{\mathcal{\hat{M}}}) > \epsilon/2\right) + \mathbb{P}\left(L(h^{\star}_{\mathcal{\hat{M}}}) - L(h^{\star}) > \epsilon/2\right),
	\end{align}
	a bound for \eqref{triangle} is a direct consequence of Theorems \ref{bound_constant} and \ref{theorem_tipeIII}.
	
	\begin{corollary}
		\label{cor_typeIV}
		Assume the premises of Theorems \ref{theorem_principal_convergence} and \ref{bound_constant} are in force. Let $\mathcal{\hat{M}} \in \mathbb{C}(\mathcal{H})$ be a random model learned by $\mathbb{M}_{\mathbb{C}(\mathcal{H})}$. Then, for any $\epsilon > 0$,
		\begin{align*}
			&\textbf{(IV)} \ \mathbb{P}\left(L(\hat{h}_{\mathcal{\hat{M}}}^{\tilde{\mathcal{D}}_{M}}) - L(h^{\star}) > \epsilon\right)\\
			& \leq \mathbb{E}\Big[B^{II}_{M,\epsilon/2}(d_{VC}(\mathcal{\hat{M}}))\Big] \\
			&+ m \ \sum_{\mathcal{M} \in \text{Max } \mathbb{C}(\mathcal{H})} \left[B_{N_{t},(\epsilon/2 \vee \epsilon^{\star})/8}(d_{VC}(\mathcal{M})) + \hat{B}_{N_{v},(\epsilon/2 \vee \epsilon^{\star})/4}(d_{VC}(\mathcal{M}))\right]\\
			&\leq  B^{II}_{M,\epsilon/2}(d_{VC}(\mathbb{C}(\mathcal{H}))) \\
			&+ m \ \mathfrak{m}(\mathbb{C}(\mathcal{H})) \left[B_{N_{t},(\epsilon/2 \vee \epsilon^{\star})/8}(d_{VC}(\mathbb{C}(\mathcal{H}))) + \hat{B}_{N_{v},(\epsilon/2 \vee \epsilon^{\star})/4}(d_{VC}(\mathbb{C}(\mathcal{H})))\right].
		\end{align*}
		In particular,
		\begin{align*}
			\lim_{\substack{N \rightarrow \infty \\ M \rightarrow \infty}} \mathbb{P}\left(L(\hat{h}_{\mathcal{\hat{M}}}^{\tilde{\mathcal{D}}_{M}}) - L(h^{\star}) > \epsilon\right) = 0,
		\end{align*}
		for any $\epsilon > 0$.
	\end{corollary}
	
	Comparing the bounds of Corollary \ref{cor_typeIV} with those of type II estimation error of learning via ERM in $\mathcal{H}$ for a sample size of $N + M$ (cf. Proposition \ref{propVC}), we see that the former can be tighter. This is the case when the VC dimension of the maximal models is smaller than $d_{VC}(\mathcal{H})$, $\epsilon^{\star}$ is large or $d_{VC}(\mathcal{M}^{\star})$ is small and $d_{VC}(\hat{\mathcal{M}}) \approx d_{VC}(\mathcal{M}^{\star})$ with high probability. This fact evidences that by properly modeling the set of candidate models, it may be possible to increase generalization without increasing the sample size. This fact will be further explored in Sections \ref{SecEnhanceGen} and \ref{SecNum}. This relation between the bounds also holds for unbounded functions (cf. Corollary \ref{cor_typeIV2}) and when learning by reusing (cf. Section \ref{SecReuse}).
	
	\section{Learning via model selection with unbounded loss functions}
	\label{SecUnbounded}
	
	When the loss function is unbounded, we need to consider relative estimation errors and make assumptions about the tail weight of $P$. Heavy tail distributions are classically defined as those with a tail heavier than that of exponential distributions \cite{foss2011}. Nevertheless, in the context of learning, the tail weight of $P$ should take into account the loss function $\ell$. Hence, for $1 < p < \infty$ and a fixed hypotheses space $\mathcal{H}$, we measure the weight of the tails of distribution $P$ by
	\begin{equation*}
		\tau_{p} \coloneqq \sup\limits_{h \in \mathcal{H}} \frac{\left(\int_{\mathcal{Z}} \ell^{p}(z,h) \ dP(z)\right)^{\frac{1}{p}}}{\int_{\mathcal{Z}} \ell(z,h) \ dP(z)} = \sup\limits_{h \in \mathcal{H}} \frac{L^{p}(h)}{L(h)},
	\end{equation*}
	in which $L^{p}(h) \coloneqq \left(\int_{\mathcal{Z}} \ell^{p}(z,h) \ dP(z)\right)^{\frac{1}{p}}$. We omit the dependence of $\tau_{p}$ on $\ell$, $P$ and $\mathcal{H}$ to simplify notation, since they will be clear from context. The weight of the tails of distribution $P$ may be defined based on $\tau_{p}$, as follows. Our presentation is analogous to \cite[Section~5.7]{vapnik1998} and is within the framework of \cite{cortes2019}. In this section, we again focus on learning with an independent sample and in Section \ref{SecReuse} we present analogous results for learning by reusing.
	
	\begin{definition}
		\label{def_tails1}
		We say that distribution $P$ on $\mathcal{H}$ under $\ell$ has:
		\begin{itemize}
			\item Light tails, if there exists a $p > 2$ such that $\tau_{p} < \infty$;
			\item Heavy tails, if there exists a $1 < p \leq 2$ such that $\tau_{p} < \infty$, but $\tau_{p} = \infty$ for all $p > 2$;
			\item Very heavy tails, if $\tau_{p} = \infty$ for all $p > 1$.
		\end{itemize}
	\end{definition}
	
	We assume that $P$ has at most heavy tails, which means there exists a $p > 1$, that can be lesser than 2, with
	\begin{align}
		\label{tauStar}
		\tau_{p} < \tau^{\star} < \infty,
	\end{align}
	that is, $P$ is in a class of distributions for which bound \eqref{tauStar} holds. From now on, fix a $p > 1$ and a $\tau^{\star}$ such that \eqref{tauStar} holds.
	
	Besides the constraint \eqref{tauStar}, we also assume that the loss function is greater or equal to one: $\ell(z,h) \geq 1$ for all $z \in \mathcal{Z}, h \in \mathcal{H}$. This is done to ease the presentation, and without loss of generality, since it is enough to sum add 1 to any unbounded loss function to have this property and, in doing so, not only the minimizers of $L_{\mathcal{D}_{N}}$ and $L$ in each model in $\mathbb{C}(\mathcal{H})$ remain the same, but also $\epsilon^{\star}$ does not change. Hence, by adding one to the loss, the estimated model $\hat{\mathcal{M}}$ and learned hypotheses from it do not change, and the result of the model selection framework is the same. We refer to Remark \ref{remark_geq1} for the technical reason we choose to consider loss functions greater than one.
	
	Finally, we assume that $\ell$ has a finite moment of order $p$, under $P$ and under the empirical measure, for all $h \in \mathcal{H}$. That is, defining\footnote{We elevate \eqref{LNp} to the $1/p$ power to be consistent with the theory presented in Appendix \ref{apVCtheory}.}
	\begin{equation}
		\label{LNp}
		L_{\mathcal{D}_{N}}^{p}(h) \coloneqq \left(\frac{1}{N} \sum_{i=1}^{N} \ell^{p}(Z_{i},h)\right)^ {\frac{1}{p}},
	\end{equation}
	we assume that
	\begin{align}
		\label{finite_moments_text}
		\sup\limits_{h \in \mathcal{H}} L_{\mathcal{D}_{N}}^{p}(h)  < \infty & & \text{ and } & & \sup\limits_{h \in \mathcal{H}} L^{p}(h) < \infty,
	\end{align}
	in which the first inequality should hold with probability one, for all possible samples $\mathcal{D}_{N}$. Since $L^{p}(h)$ is non-decreasing in $p$ for each fixed $h$, \eqref{finite_moments_text} actually implies \eqref{tauStar}, so \eqref{finite_moments_text} is the non-trivial constraint in distribution $P$.
	
	Although this is a deviation from the distribution-free framework, it is a mild constraint on distribution $P$. On the one hand, the condition on $L^{p}$ is usually satisfied for distributions observed in real data (see \cite[Section~5.7]{vapnik1998} for examples with Normal, Uniform, and Laplacian distributions under the quadratic loss function). On the other hand, the condition on $L_{\mathcal{D}_{N}}^{p}$ is more a feature of the loss function, than of the distribution $P$, and can be guaranteed if one excludes from $\mathcal{H}$ some hypotheses with arbitrarily large loss in a way that $h^{\star}$ and $d_{VC}(\mathcal{H})$ remain the same (see Lemma \ref{lemma_norm} and Remark \ref{remark_finite_moments} for more details).
	
	When the loss function is unbounded, besides the constraints in the moments of $\ell$, under $P$ and the empirical measure, we also have to consider variants of the estimation errors. Instead of the estimation errors, we consider the relative estimation errors:
	\begin{align*}
		&\textbf{(I)} \sup\limits_{h \in \hat{\mathcal{M}}} \ \frac{\text{\textbar} L(h) - L_{\tilde{\mathcal{D}}_{M}}(h) \text{\textbar}}{L(h)} & & \textbf{(II)} \ \frac{L(\hat{h}_{\hat{\mathcal{M}}}^{\mathbb{A}}) - L(h^{\star}_{\hat{\mathcal{M}}})}{L(\hat{h}_{\hat{\mathcal{M}}}^{\mathbb{A}})}\\
		&\textbf{(III)} \ \frac{L(h^{\star}_{\hat{\mathcal{M}}}) - L(h^{\star})}{L(h^{\star}_{\hat{\mathcal{M}}})} & & \textbf{(IV)} \ \frac{L(\hat{h}_{\hat{\mathcal{M}}}^{\mathbb{A}}) - L(h^{\star})}{L(\hat{h}_{\hat{\mathcal{M}}}^{\mathbb{A}})}
	\end{align*}
	in which algorithm $\mathbb{A}$ depends on the estimation technique once $\hat{\mathcal{M}}$ is selected.
	
	In this section, we prove analogues of Theorems \ref{theorem_principal_convergence}, \ref{bound_constant} and \ref{theorem_tipeIII}. Before starting the study of the convergence of $\hat{\mathcal{M}}$ to $\mathcal{M}^{\star}$, we state a result analogous to Proposition \ref{propVC} about deviation bounds of relative type I and II estimation errors on $\mathcal{H}$, which are a consequence of Corollaries \ref{convergence_relativeTI} and \ref{cor2TypeII}. These are novel results of this paper which extend those of \cite{cortes2019}. 
	
	\begin{proposition}
		\label{propVC2}
		Assume the loss function is unbounded and $P$ is such that \eqref{finite_moments_text} holds. Fix a hypotheses space $\mathcal{H}$ with $d_{VC}(\mathcal{H}) < \infty$. There exist sequences $\{B^{I}_{N,\epsilon}: N \geq 1\}$ and $\{B^{II}_{N,\epsilon}: N \geq 1\}$ of positive real-valued increasing functions with domain $\mathbb{Z}_{+}$ satisfying
		\begin{equation*}
			\lim\limits_{N \to \infty} B^{I}_{N,\epsilon}(k) = \lim\limits_{N \to \infty} B^{II}_{N,\epsilon}(k) = 0,
		\end{equation*}
		for all $\epsilon > 0$ and $k \in \mathbb{Z}_{+}$ fixed, such that
		\begin{align*}
			\mathbb{P}&\left(\sup\limits_{h \in \mathcal{H}} \frac{\text{\textbar}L_{\mathcal{D}_{N}}(h) - L(h)\text{\textbar}}{L(h)} > \epsilon\right) \leq B^{I}_{N,\epsilon}(d_{VC}(\mathcal{H})) \text{ and }\\
			\mathbb{P}&\left(\frac{L(\hat{h}^{\mathcal{D}_{N}}) - L(h^{\star})}{L(\hat{h}^{\mathcal{D}_{N}})} > \epsilon\right) \leq B^{II}_{N,\epsilon}(d_{VC}(\mathcal{H})).
		\end{align*}
		Furthermore, the following holds:
		\begin{align*}
			\sup\limits_{h \in \mathcal{H}} \frac{\text{\textbar}L_{\mathcal{D}_{N}}(h) - L(h)\text{\textbar}}{L(h)} \xrightarrow[N \to \infty]{a.s.} 0 & & \text{ and } & & \frac{L(\hat{h}^{\mathcal{D}_{N}}) - L(h^{\star})}{L(\hat{h}^{\mathcal{D}_{N}})} \xrightarrow[N \to \infty]{a.s.} 0.
		\end{align*}
	\end{proposition}
	
	We refer to Appendix \ref{ApUnbounded} for explicit formulas for $B^{I}_{N,\epsilon}(k)$ and $B^{II}_{N,\epsilon}(k)$ in this case. In particular, the rate of convergence of $B^{I}_{N,\epsilon}(k)$ and $B^{II}_{N,\epsilon}(k)$ when $N \to \infty$ depends on the value of $p$ for which \eqref{tauStar} holds, even though they converge to zero for all $p > 1$.
	
	The results of this section seek to obtain insights about the asymptotic behavior of learning via model selection in the case of unbounded loss functions, rather than obtain the tightest possible bounds. Hence, in some results, the simplicity of the bounds is preferred over their tightness, and tighter bounds may be readily obtained from the proofs.
	
	\subsection{Convergence to the target model}
	
	We start by showing a result similar to Theorem \ref{theorem_principal_convergence}.
	
	\begin{theorem}
		\label{theorem_principal_convergence_unbounded}
		Assume the loss function is unbounded and $P$ is such that \eqref{finite_moments_text} holds. For each $\epsilon > 0$, let $\{B_{N,\epsilon}: N \geq 1\}$ and $\{\hat{B}_{N,\epsilon}: N \geq 1\}$ be sequences of positive real-valued increasing functions with domain $\mathbb{Z}_{+}$ satisfying
		\begin{equation*}
			\lim\limits_{N \to \infty} B_{N,\epsilon}(k) = \lim\limits_{N \to \infty} \hat{B}_{N,\epsilon}(k) = 0,
		\end{equation*}
		for all $\epsilon > 0$ and $k \in \mathbb{Z}_{+}$ fixed, and such that
		\begin{align*}
			\max_{j} \mathbb{P}\left(\sup\limits_{h \in \mathcal{M}} \frac{\text{\textbar}L(h) - L_{\mathcal{D}_{N}^{(j)}}(h)\text{\textbar}}{L(h)} > \epsilon \right) \leq B_{N_{t},\epsilon}(d_{VC}(\mathcal{M})) & \text{ and } &\\ \nonumber \\ \nonumber
			\max_{j} \mathbb{P}\left(\sup\limits_{h \in \mathcal{M}} \frac{\text{\textbar}L(h) - \hat{L}^{(j)}(h)\text{\textbar}}{L(h)} > \epsilon \right) \leq \hat{B}_{N_{v},\epsilon}(d_{VC}(\mathcal{M})), & & 
		\end{align*}
		for all $\mathcal{M} \in \mathbb{C}(\mathcal{H})$, recalling that $L_{\mathcal{D}_{N}^{(j)}}$ and $\hat{L}^{(j)}$ represent the empirical risk under the $j$-th training and validation samples, respectively. Let $\mathcal{\hat{M}} \in \mathbb{C}(\mathcal{H})$ be a random model learned by $\mathbb{M}_{\mathbb{C}(\mathcal{H})}$. Then,
		\begin{align*}
			\mathbb{P}\left(L(\hat{\mathcal{M}}) \neq L(\mathcal{M}^{\star})\right) & \leq 2m \sum_{\mathcal{M} \in \text{Max } \mathbb{C}(\mathcal{H})} \left[\hat{B}_{N_{v},\frac{\delta(1-\delta)}{2}}(d_{VC}(\mathcal{M})) + B_{N_{t},\frac{\delta(1-\delta)}{4}}(d_{VC}(\mathcal{M}))\right]\\
			&\leq 2 m \ \mathfrak{m}(\mathbb{C}(\mathcal{H})) \left[\hat{B}_{N_{v},\frac{\delta(1-\delta)}{2}}(d_{VC}(\mathbb{C}(\mathcal{H}))) + B_{N_{t},\frac{\delta(1-\delta)}{4}}(d_{VC}(\mathbb{C}(\mathcal{H})))\right],
		\end{align*}
		in which $m$ is the number of pairs considered to calculate \eqref{form_Lhat} and 
		\begin{equation*}
			\delta \coloneqq \frac{\epsilon^{\star}}{2 \max\limits_{i \in \mathcal{J}} L(\mathcal{M}_{i})}.
		\end{equation*}
		Furthermore, if
		\begin{align}
			\label{as_conv2}
			\max_{\mathcal{M} \in \mathbb{C}(\mathcal{H})} \max_{j} \sup\limits_{h \in \mathcal{M}} \frac{\text{\textbar}L(h) - L_{\mathcal{D}_{N}^{(j)}}(h)\text{\textbar}}{L(h)} \xrightarrow[N \to \infty]{\text{a.s.}} 0 & \text{ and } &\\ \nonumber \\ \nonumber
			\max_{\mathcal{M} \in \mathbb{C}(\mathcal{H})} \max_{j} \sup\limits_{h \in \mathcal{M}} \frac{\text{\textbar}L(h) - \hat{L}^{(j)}(h)\text{\textbar}}{L(h)} \xrightarrow[N \to \infty]{\text{a.s.}} 0, & & 
		\end{align}
		then
		\begin{equation*}
			\lim_{N \rightarrow \infty} \mathbb{P}\left(\hat{\mathcal{M}} = \mathcal{M}^{\star}\right) = 1.
		\end{equation*}
	\end{theorem}
	
	In this instance, a bound for $\mathbb{P}(L(\hat{\mathcal{M}}) \neq L(\mathcal{M}^{\star}))$, and the almost sure convergence of $\hat{\mathcal{M}}$ to $\mathcal{M}^{\star}$, in the case of k-fold cross-validation and independent validation sample, follow from Proposition \ref{propVC2} in a manner analogous to Theorem \ref{CVModelconvergence}. We state the almost sure convergence in Theorem \ref{CVModelconvergence2}, whose proof is analogous to that of Theorem \ref{CVModelconvergence}, and follows from Corollary \ref{convergence_relativeTI}. 
	
	\begin{theorem}
		\label{CVModelconvergence2}
		Assume the loss function is unbounded and $P$ is such that \eqref{finite_moments_text} holds. If $\hat{L}$ is given by k-fold cross-validation or by an independent validation sample, then $\hat{\mathcal{M}}$ converges with probability one to $\mathcal{M}^{\star}$.
	\end{theorem}
	
	\subsection{Convergence of estimation errors on $\hat{\mathcal{M}}$}
	
	The results stated here are rather similar to the case of bounded loss functions, with some minor modifications. Hence, we state the analogous results and present a proof only when it is different from the respective result in Section \ref{SecConvOnMhat}.
	
	Bounds for relative types I and II estimation errors, when learning on a random model with a sample independent of the one employed to compute such random model, may be obtained as in Theorem \ref{bound_constant}. In fact, the proof of the following bounds is the same as in that theorem, with the respective changes from estimation errors to relative estimation errors. Hence, we state the results without a proof.
	
	\begin{theorem}
		\label{bound_constant2}	
		Fix an unbounded loss function and assume $P$ is such that \eqref{finite_moments_text} holds. Assume we are learning with an independent sample $\tilde{\mathcal{D}}_{M}$, and that for each $\epsilon > 0$ there exist sequences $\{B^{I}_{M,\epsilon}: M \geq 1\}$ and $\{B^{II}_{M,\epsilon}: M \geq 1\}$ of positive real-valued increasing functions with domain $\mathbb{Z}_{+}$ satisfying
		\begin{equation*}
			\lim\limits_{M \to \infty} B^{I}_{M,\epsilon}(k) = \lim\limits_{M \to \infty} B^{II}_{M,\epsilon}(k) = 0,
		\end{equation*}
		for all $\epsilon > 0$ and $k \in \mathbb{Z}_{+}$ fixed, such that
		\begin{align*}
			\mathbb{P}\left(\sup\limits_{h \in \mathcal{M}} \frac{\text{\textbar}L_{\tilde{\mathcal{D}}_{M}}(h) - L(h)\text{\textbar}}{L(h)}  > \epsilon \right) \leq B^{I}_{M,\epsilon}(d_{VC}(\mathcal{M})) & \text{ and } &\\
			\mathbb{P}\left(\frac{L(\hat{h}_{\mathcal{M}}^{\tilde{\mathcal{D}}_{M}}) - L(h^{\star}_{\mathcal{M}})}{L(\hat{h}_{\mathcal{M}}^{\tilde{\mathcal{D}}_{M}})} > \epsilon \right) \leq B^{II}_{M,\epsilon}(d_{VC}(\mathcal{M})), & & 
		\end{align*}
		for all $\mathcal{M} \in \mathbb{C}(\mathcal{H})$. Let $\mathcal{\hat{M}} \in \mathbb{C}(\mathcal{H})$ be a random model learned by $\mathbb{M}_{\mathbb{C}(\mathcal{H})}$. Then, for any $\epsilon > 0$,
		\begin{align*}
			\textbf{(I)} \ \mathbb{P}&\left(\sup\limits_{h \in \mathcal{\hat{M}}} \frac{\text{\textbar}L_{\tilde{\mathcal{D}}_{M}}(h) - L(h)\text{\textbar}}{L(h)}  > \epsilon \right) \leq \mathbb{E}\Big[B^{I}_{M,\epsilon}(d_{VC}(\mathcal{\hat{M}}))\Big] \leq B^{I}_{M,\epsilon}\left(d_{VC}(\mathbb{C}(\mathcal{H}))\right)
		\end{align*}
		and
		\begin{align*}
			\textbf{(II)} \ \mathbb{P}\left(\frac{L(\hat{h}_{\mathcal{\hat{M}}}^{\tilde{\mathcal{D}}_{M}}) - L(h^{\star}_{\mathcal{\hat{M}}})}{L(\hat{h}_{\mathcal{\hat{M}}}^{\tilde{\mathcal{D}}_{M}})} > \epsilon \right) \leq \mathbb{E}\Big[B^{II}_{M,\epsilon}(d_{VC}(\mathcal{\hat{M}}))\Big] \leq B^{II}_{M,\epsilon}\left(d_{VC}(\mathbb{C}(\mathcal{H}))\right),
		\end{align*}
		in which the expectations are over all samples $\mathcal{D}_{N}$, from which $\hat{\mathcal{M}}$ is calculated. Since $d_{VC}(\mathbb{C}(\mathcal{H})) < \infty$, both probabilities above converge to zero when $M \to \infty$.
	\end{theorem}
	
	The convergence to zero of relative type III estimation error may be obtained, as in Theorem \ref{theorem_tipeIII}, by the methods used to prove Theorem \ref{theorem_principal_convergence_unbounded}. We state and prove this result, since its proof is slightly different from that of Theorem \ref{theorem_tipeIII}.
	
	\begin{theorem}
		\label{theorem_tipeIII2}
		Assume the premises of Theorem \ref{theorem_principal_convergence_unbounded} are in force. Let $\mathcal{\hat{M}} \in \mathbb{C}(\mathcal{H})$ be a random model learned by $\mathbb{M}_{\mathbb{C}(\mathcal{H})}$. Then, for any $\epsilon > 0$,
		\begin{align*}
			\textbf{(III)} &\ \mathbb{P}\left(\frac{L(h_{\hat{\mathcal{M}}}^{\star}) - L(h^{\star})}{L(h_{\hat{\mathcal{M}}}^{\star})} > \frac{\epsilon}{L(\mathcal{M}^{\star})}\right) \\
			& \leq 2m \sum_{\mathcal{M} \in \text{Max } \mathbb{C}(\mathcal{H})} \left[\hat{B}_{N_{v},\frac{\delta^\prime(1-\delta^\prime)}{2}}(d_{VC}(\mathcal{M})) + B_{N_{t},\frac{\delta^\prime(1-\delta^\prime)}{4}}(d_{VC}(\mathcal{M}))\right]\\
			& \leq 2m \ \mathfrak{m}(\mathbb{C}(\mathcal{H})) \left[\hat{B}_{N_{v},\frac{\delta^\prime(1-\delta^\prime)}{2}}(d_{VC}(\mathbb{C}(\mathcal{H}))) + B_{N_{t},\frac{\delta^\prime(1-\delta^\prime)}{4}}(d_{VC}(\mathbb{C}(\mathcal{H})))\right],
		\end{align*}
		in which
		\begin{equation*}
			\delta^\prime \coloneqq \frac{\epsilon \vee \epsilon^{\star}}{2 \max\limits_{i \in \mathcal{J}} L(\mathcal{M}_{i})}.
		\end{equation*}
		In particular,
		\begin{align*}
			\lim_{N \rightarrow \infty} \mathbb{P}\left(\frac{L(h_{\hat{\mathcal{M}}}^{\star}) - L(h^{\star})}{L(h_{\hat{\mathcal{M}}}^{\star})} > \epsilon\right) = 0,
		\end{align*}
		for any $\epsilon > 0$.
	\end{theorem}
	
	Finally, a bound on the rate of convergence of type IV estimation error to zero is a direct consequence of Theorems \ref{bound_constant2} and \ref{theorem_tipeIII2}, and the following inequality
	\begin{align*}
		&\textbf{(IV)} \ \mathbb{P}\left(\frac{L(\hat{h}_{\mathcal{\hat{M}}}^{\tilde{\mathcal{D}}_{M}}) - L(h^{\star})}{L(\hat{h}_{\mathcal{\hat{M}}}^{\tilde{\mathcal{D}}_{M}})} > \frac{\epsilon}{L(\mathcal{M}^{\star})}\right)\\
		&\leq \mathbb{P}\left(\frac{L(\hat{h}_{\mathcal{\hat{\mathcal{M}}}}^{\tilde{\mathcal{D}}_{M}}) - L(h^{\star}_{\hat{\mathcal{M}}})}{L(\hat{h}_{\mathcal{\hat{M}}}^{\tilde{\mathcal{D}}_{M}})} > \frac{\epsilon}{2L(\mathcal{M}^{\star})}\right) + \mathbb{P}\left(\frac{L(h^{\star}_{\hat{\mathcal{M}}}) - L(h^{\star})}{L(h^{\star}_{\mathcal{\hat{\mathcal{M}}}})} > \frac{\epsilon}{2L(\mathcal{M}^{\star})}\right),
	\end{align*}
	which is true since $L(\hat{h}_{\mathcal{\hat{M}}}^{\tilde{\mathcal{D}}_{M}}) \geq L(h^{\star}_{\mathcal{\hat{M}}})$.
	
	\begin{corollary}
		\label{cor_typeIV2}
		Assume the premises of Theorems \ref{theorem_principal_convergence_unbounded} and \ref{bound_constant2} are in force. Let $\mathcal{\hat{M}} \in \mathbb{C}(\mathcal{H})$ be a random model learned by $\mathbb{M}_{\mathbb{C}(\mathcal{H})}$. Then, for any $\epsilon > 0$,
		\begin{align*}
			&\mathbb{P}\left(\frac{L(\hat{h}_{\mathcal{\hat{M}}}^{\tilde{\mathcal{D}}_{M}}) - L(h^{\star})}{L(\hat{h}_{\mathcal{\hat{M}}}^{\tilde{\mathcal{D}}_{M}})} > \frac{\epsilon}{L(\mathcal{M}^{\star})}\right)\\
			& \leq \mathbb{E}\Big[B^{II}_{M,\frac{\epsilon}{2L(\mathcal{M}^{\star})}}(d_{VC}(\mathcal{\hat{M}}))\Big] \\
			&+ 2m \ \sum_{\mathcal{M} \in \text{Max } \mathbb{C}(\mathcal{H})} \left[\hat{B}_{N_{v},\frac{\delta^\prime(1-\delta^\prime)}{2}}(d_{VC}(\mathcal{M})) + B_{N_{t},\frac{\delta^\prime(1-\delta^\prime)}{4}}(d_{VC}(\mathcal{M}))\right]\\
			&\leq  B^{II}_{M,\frac{\epsilon}{2L(\mathcal{M}^{\star})}}(d_{VC}(\mathbb{C}(\mathcal{H}))) \\
			&+ 2m \ \mathfrak{m}(\mathbb{C}(\mathcal{H})) \left[\hat{B}_{N_{v},\frac{\delta^\prime(1-\delta^\prime)}{2}}(d_{VC}(\mathbb{C}(\mathcal{H}))) + B_{N_{t},\frac{\delta^\prime(1-\delta^\prime)}{4}}(d_{VC}(\mathbb{C}(\mathcal{H})))\right]
		\end{align*}
		with
		\begin{equation*}
			\delta^\prime \coloneqq \frac{\epsilon/2 \vee \epsilon^{\star}}{2 \max\limits_{i \in \mathcal{J}} L(\mathcal{M}_{i})}.
		\end{equation*}
		In particular,
		\begin{align*}
			\lim_{\substack{N \rightarrow \infty \\ M \rightarrow \infty}} \mathbb{P}\left(\frac{L(\hat{h}_{\mathcal{\hat{M}}}^{\tilde{\mathcal{D}}_{M}}) - L(h^{\star})}{L(\hat{h}_{\mathcal{\hat{M}}}^{\tilde{\mathcal{D}}_{M}})}  > \epsilon\right) = 0,
		\end{align*}
		for any $\epsilon > 0$.
	\end{corollary}
	
	\section{Learning by reusing}
	\label{SecReuse}
	
	In this section, we present bounds for types I, II, and consequently IV, estimation errors for learning by reusing for bounded and unbounded loss functions. When learning by reusing, one is employing the same sample points to estimate $\hat{\mathcal{M}}$ and to learn a hypotheses $\hat{h}^{\mathcal{D}_{N}}_{\hat{\mathcal{M}}} \in \hat{\mathcal{M}}$ from it, so there is a dependence between types I and II estimation errors and the events $\{\hat{\mathcal{M}} = \mathcal{M}\}, \mathcal{M} \in \mathbb{C}(\mathcal{H})$. For instance, the bounds for types I and II estimation errors in Theorem \ref{bound_constant} depend on an equality (cf. \eqref{cond_independence} in the proof of Theorem \ref{bound_constant}) that does not hold when learning by reusing since
	\begin{align*}
		\mathbb{P}\left(\sup\limits_{h \in \hat{\mathcal{M}}} \big|L_{\mathcal{D}_{N}}(h) - L(h) \big| > \epsilon \Big|\mathcal{\hat{M}} = \mathcal{M}\right) \neq \mathbb{P}\left(\sup\limits_{h \in \mathcal{M}} \big|L_{\mathcal{D}_{N}}(h) - L(h) \big| > \epsilon\right).
	\end{align*}
	Conditioned on $\{\hat{\mathcal{M}} = \mathcal{M}\}$, the distribution of each sample point $(X_{i},Y_{i}), i = 1,\dots,N$, changes, and moreover these points are no longer independent, as they must satisfy $\hat{\mathcal{M}} = \mathcal{M}$. Therefore, the argument of the proof of Theorem \ref{bound_constant} does not hold in this instance.
	
	Nevertheless, since $\hat{\mathcal{M}}$ converges with probability one to $\mathcal{M}^{\star}$ by Theorem \ref{theorem_principal_convergence}, we may obtain a bound for types I and II estimation errors when learning by reusing, expressed in terms of the corresponding bounds on $\mathcal{M}^{\star}$ and the rate of convergence of $\hat{\mathcal{M}}$ to $\mathcal{M}^{\star}$.
	
	\begin{theorem}
		\label{bound_constant_reusing}	
		Fix a bounded loss function. Assume we are learning by reusing and that, for each $\epsilon > 0$, there exist sequences $\{B^{I}_{N,\epsilon}: N \geq 1\}$ and $\{B^{II}_{N,\epsilon}: N \geq 1\}$ of positive real-valued increasing functions with domain $\mathbb{Z}_{+}$ satisfying
		\begin{equation*}
			\lim\limits_{N \to \infty} B^{I}_{N,\epsilon}(k) = \lim\limits_{N \to \infty} B^{II}_{N,\epsilon}(k) = 0,
		\end{equation*}
		for all $\epsilon > 0$ and $k \in \mathbb{Z}_{+}$ fixed, such that
		\begin{align*}
			\label{bound_theoremBC2}
			\mathbb{P}\left(\sup\limits_{h \in \mathcal{M}} \big|L_{\mathcal{D}_{N}}(h) - L(h) \big| > \epsilon \right) \leq B^{I}_{N,\epsilon}(d_{VC}(\mathcal{M})) & \text{ and } &\\ \nonumber
			\mathbb{P}\left(L(\hat{h}_{\mathcal{M}}^{\mathcal{D}_{N}}) - L(h^{\star}_{\mathcal{M}}) > \epsilon \right) \leq B^{II}_{N,\epsilon}(d_{VC}(\mathcal{M})), & &
		\end{align*}
		for all $\mathcal{M} \in \mathbb{C}(\mathcal{H})$. Let $\mathcal{\hat{M}} \in \mathbb{C}(\mathcal{H})$ be a random model learned by $\mathbb{M}_{\mathbb{C}(\mathcal{H})}$. Then, for any $\epsilon > 0$,
		\begin{align*}
			\textbf{(I)} \ \mathbb{P}\left(\sup\limits_{h \in \mathcal{\hat{M}}} \big|L_{\mathcal{D}_{N}}(h) - L(h) \big| > \epsilon \right) \leq B^{I}_{N,\epsilon}(d_{VC}(\mathcal{M}^{\star})) + \mathbb{P}\left(\hat{\mathcal{M}} \neq \mathcal{M}^{\star}\right)
		\end{align*}
		and
		\begin{align*}
			\textbf{(II)} \ \mathbb{P}\left(L(\hat{h}_{\mathcal{\hat{M}}}^{\mathcal{D}_{N}}) - L(h^{\star}_{\mathcal{\hat{M}}}) > \epsilon \right) \leq B^{II}_{N,\epsilon}(d_{VC}(\mathcal{M}^{\star})) + \mathbb{P}\left(\hat{\mathcal{M}} \neq \mathcal{M}^{\star}\right).
		\end{align*}
		If conditions \eqref{as_conv} of Theorem \ref{theorem_principal_convergence} are satisfied, both probabilities above converge to zero when $N \to \infty$.
	\end{theorem}
	
	\begin{remark}
		Theorem \ref{bound_constant_reusing} also holds when learning with an independent sample by exchanging $\mathcal{D}_{N}$ with $\tilde{\mathcal{D}}_{M}$.
	\end{remark}
	
	From inequality \eqref{triangle} and the bound for type III estimation error established in Theorem \ref{theorem_tipeIII}, which also holds when learning by reusing (cf. Remark \ref{remReuse}), it follows that the tail probability of type IV estimation error converges to zero as $N$ tends to infinity, a result analogous to Corollary \ref{cor_typeIV}.
	
	For unbounded loss functions, when learning by reusing, a result analogous to Theorem \ref{bound_constant_reusing}, together with Theorem \ref{theorem_principal_convergence_unbounded} and a result analogous to Corollary \ref{cor_typeIV2}, will imply the convergence of the estimation errors to zero. We state this result without proof.
	
	\begin{theorem}
		\label{bound_constant_reusing2}	
		Fix an unbounded loss function and assume $P$ is such that \eqref{finite_moments_text} holds. Assume we are learning by reusing and that, for each $\epsilon > 0$, there exist sequences $\{B^{I}_{N,\epsilon}: N \geq 1\}$ and $\{B^{II}_{N,\epsilon}: N \geq 1\}$ of positive real-valued increasing functions with domain $\mathbb{Z}_{+}$ satisfying
		\begin{equation*}
			\lim\limits_{N \to \infty} B^{I}_{N,\epsilon}(k) = \lim\limits_{N \to \infty} B^{II}_{N,\epsilon}(k) = 0,
		\end{equation*}
		for all $\epsilon > 0$ and $k \in \mathbb{Z}_{+}$ fixed, such that
		\begin{align*}
			\mathbb{P}\left(\sup\limits_{h \in \mathcal{M}} \left|\frac{L_{\mathcal{D}_{N}}(h) - L(h)}{L(h)} \right| > \epsilon \right) \leq B^{I}_{N,\epsilon}(d_{VC}(\mathcal{M})) & \text{ and } &\\ \nonumber
			\mathbb{P}\left(\frac{L(\hat{h}_{\mathcal{M}}^{\mathcal{D}_{N}}) - L(h^{\star}_{\mathcal{M}})}{L(\hat{h}_{\mathcal{M}}^{\mathcal{D}_{N}})} > \epsilon \right) \leq B^{II}_{N,\epsilon}(d_{VC}(\mathcal{M})), & & 
		\end{align*}
		for all $\mathcal{M} \in \mathbb{C}(\mathcal{H})$. Let $\mathcal{\hat{M}} \in \mathbb{C}(\mathcal{H})$ be a random model learned by $\mathbb{M}_{\mathbb{C}(\mathcal{H})}$. Then, for any $\epsilon > 0$,
		\begin{align*}
			\textbf{(I)} \ \mathbb{P}\left(\sup\limits_{h \in \mathcal{\hat{M}}} \left|\frac{L_{\mathcal{D}_{N}}(h) - L(h)}{L(h)} \right| > \epsilon \right) \leq B^{I}_{N,\epsilon}(d_{VC}(\mathcal{M}^{\star})) + \mathbb{P}\left(\hat{\mathcal{M}} \neq \mathcal{M}^{\star}\right)
		\end{align*}
		and
		\begin{align*}
			\textbf{(II)} \ \mathbb{P}\left(\frac{L(\hat{h}_{\mathcal{\hat{M}}}^{\mathcal{D}_{N}}) - L(h^{\star}_{\mathcal{\hat{M}}})}{L(\hat{h}_{\mathcal{\hat{M}}}^{\mathcal{D}_{N}})} > \epsilon \right) \leq B^{II}_{N,\epsilon}(d_{VC}(\mathcal{M}^{\star})) + \mathbb{P}\left(\hat{\mathcal{M}} \neq \mathcal{M}^{\star}\right).
		\end{align*}
		If conditions \eqref{as_conv2} of Theorem \ref{theorem_principal_convergence_unbounded} are satisfied, both probabilities above converge to zero when $N \to \infty$.
	\end{theorem}
	
	\begin{remark}
		Bounds for the probabilities in Theorems \ref{bound_constant_reusing} and \ref{bound_constant_reusing2} of the form $B^{\cdot}_{N,\epsilon}(k) + \mathbb{P}(d_{VC}(\hat{\mathcal{M}}) > k), 1 \leq k < d_{VC}(\mathbb{C}(\mathcal{H})),$ could be obtained by conditioning the respective probability on the events $\{d_{VC}(\hat{\mathcal{M}}) \leq k\}$ and the complement $\{d_{VC}(\hat{\mathcal{M}}) > k\}$ instead of $\{\hat{\mathcal{M}} = \mathcal{M}^\star\}$ and $\{\hat{\mathcal{M}} \neq \mathcal{M}^\star\}$ in the proof of these theorems (cf. Section \ref{SecProofReuse}). In principle, this quantity could be minimized in $k$ to obtain tighter bounds.
	\end{remark}
	
	\begin{remark}
		Another approach for learning via model selection is considering a combination of the hypotheses used in the computation of cross-validation errors as the learned hypotheses. This is also a case that cannot be easily analyzed in the general, distribution-free framework, as the same sample is reused both to select the model and to learn on it. Its study in specific cases is an interesting topic for future research.
	\end{remark}
	
	\section{Increasing generalization by learning via model selection}
	\label{SecEnhanceGen}
	
	In this section, we analyze concrete examples to better understand how generalization may be increased by learning via model selection. In particular, we illustrate the role of $\epsilon^{\star}$, $d_{VC}(\mathcal{M}^{\star})$ and $d_{VC}(\mathbb{C}(\mathcal{H}))$ on the generalization and how it can be increased by inserting prior information about $h^{\star}$ into the family of candidate models.
	
	In Sections \ref{SecLS} and \ref{SecExamplesLS}, we present a class of candidate models in which the partial order by inclusion reflects the complexity of the models. We call this class Learning Spaces and focus our analyses on a subclass that has a lattice structure. In Section \ref{SecPLLS}, we analyze a worst-case scenario of learning a classifier with finite domain, focusing on the effect of $\epsilon^{\star}$ and $d_{VC}(\mathbb{C}(\mathcal{H}))$ on the generalization. In Section \ref{SecRegression}, we simulate a typical case in linear regression in high dimension, focusing on the effect of $d_{VC}(\mathcal{M}^{\star})$ and prior information about $h^{\star}$ on generalization. In Section \ref{SecPIEnhance}, we further discuss how generalization may be increased by inserting domain knowledge into the family of candidate models.
	
	\subsection{Learning Spaces}
	\label{SecLS}
	
	Let $\mathbb{L}(\mathcal{H}) \coloneqq \{\mathcal{M}_{i}: i \in \mathcal{J} \subset \mathbb{Z}_{+}\}$ be a finite subset of the power set of $\mathcal{H}$, i.e., $\mathbb{L}(\mathcal{H}) \subset \mathcal{P}(\mathcal{H})$ and $|\mathcal{J}| < \infty$. We say that the partially ordered set (poset) $(\mathbb{L}(\mathcal{H}),\subset)$ is a Learning Space if
	\begin{itemize}
		\item[] (i) $\bigcup\limits_{i \in \mathcal{J}} \mathcal{M}_{i} = \mathcal{H}$
		\item[] (ii) $\mathcal{M}_{1}, \mathcal{M}_{2} \in \mathbb{L}(\mathcal{H})$ and $\mathcal{M}_{1} \subset \mathcal{M}_{2}$ implies $d_{VC}(\mathcal{M}_{1}) < d_{VC}(\mathcal{M}_{2})$.
	\end{itemize}
	
	Learning Spaces are families of candidate models that cover $\mathcal{H}$ and such that any element $\mathcal{M} \in \mathbb{L}(\mathcal{H})$ is complexity maximal in the sense that there does not exist $\mathcal{M}' \in \mathbb{L}(\mathcal{H}), \mathcal{M}' \neq \mathcal{M},$ such that $d_{VC}(\mathcal{M}') = d_{VC}(\mathcal{M})$ and $\mathcal{M}' \subset \mathcal{M}$. This condition guarantees that if $\mathcal{M}_{1} \subset \mathcal{M}_{2}$, then the complexity of $\mathcal{M}_{2}$ is greater than that of $\mathcal{M}_{1}$. This implies that the poset $(\mathbb{L}(\mathcal{H}),\subset)$ reflects the complexity of its models. We note that one could choose $\{\mathcal{M}_{1},\dots,\mathcal{M}_{n}\}$ without thinking of it as a decomposition of a hypotheses space $\mathcal{H}$. Nevertheless, if condition (ii) is satisfied, then it would be a Learning Space of $\mathcal{H} = \cup_{i} \mathcal{M}_{i}$, so taking $\mathcal{H}$ as this union, the only non-trivial condition is (ii).
	
	Learning Spaces are not unique, i.e., there are multiple subsets of $\mathcal{P}(\mathcal{H})$ which are Learning Spaces, and the main class of Learning Spaces are the Lattice Learning Spaces, which have a complete lattice structure. A complete lattice $(\mathbb{L}(\mathcal{H}),\subset,\wedge,\vee,\mathcal{O},\mathcal{I})$ is the poset $(\mathbb{L}(\mathcal{H}),\subset)$ with binary relations meet ($\wedge$) and join ($\vee$) representing the greatest lower bound and least upper bound of two elements. A lattice is complete if every subset of $\mathbb{L}(\mathcal{H})$ has a meet and a join in $\mathbb{L}(\mathcal{H})$. In particular, a complete lattice has a least element ($\mathcal{O}$) and greatest element ($\mathcal{I}$) that are the meet and join of $\mathbb{L}(\mathcal{H})$, so $Max \ \mathbb{L}(\mathcal{H}) = \{\mathcal{I}\}$ and $\mathfrak{m}(\mathbb{L}(\mathcal{H})) = 1$.
	
	\begin{remark}
		Although we consider the VC-dimension, other complexity measures of hypotheses spaces could be used to define the Learning Spaces, such as the fat-shattering dimension \cite{bartlett1994fat} and the Rademacher and Gaussian complexities \cite{bartlett2002rademacher}. However, these are complexity measures that depend on the data-generating distribution, and hence could lead to collections of models that are Learning Spaces for some distributions, but not for others. Nevertheless, we note that the value of the VC dimension is not important to the algebraic aspect of the Learning Space definition, but only the fact that it increases when we consider nested models. Hence, any other complexity measure such that this increase is also observed for the chosen poset of models, under the data generating distribution, would generate a valid Learning Space.
	\end{remark}
	
	The first step in building a Learning Space is fixing an algebraic parametric representation of the hypotheses in $\mathcal{H}$. The algebraic structure of $(\mathbb{L}(\mathcal{H}),\subset)$ may be defined from the learning model and algebraic parametric representation fixed, in the following manner. Let $(\mathcal{F},\leq)$ be a poset, in which $\mathcal{F}$ is an arbitrary set with finite cardinality. Moreover, let $\mathcal{R}: \mathcal{F} \mapsto Im(\mathcal{R}) \subset \mathcal{P}(\mathcal{H})$ be a lattice isomorphism from set $(\mathcal{F},\leq)$ to $(Im(\mathcal{R}),\subset)$, a subset of the power set of $\mathcal{H}$ partially ordered by inclusion. This means that $\mathcal{R}$ is bijective, and if $a,b \in \mathcal{F}, a \leq b$, then $\mathcal{R}(a) \subset \mathcal{R}(b)$, so $\mathcal{R}$ preserves the partial order $\leq$ on $\mathcal{F}$ as the partial order on $Im(\mathcal{R})$ given by inclusion. Then, if 
	\begin{itemize}
		\item[] (i) $\bigcup\limits_{a \in \mathcal{F}} \mathcal{R}(a) = \mathcal{H}$ and
		\item[] (ii) $a,b \in \mathcal{F}, a \leq b,$ implies $d_{VC}(\mathcal{R}(a)) < d_{VC}(\mathcal{R}(b))$,
	\end{itemize}
	we may define $\mathbb{L}(\mathcal{H}) \coloneqq Im(\mathcal{R})$ as a Learning Space of $\mathcal{H}$. We call isomorphisms which satisfy these conditions Learning Space generators. Since the generator $\mathcal{R}$ is an isomorphism, it preserves properties of $(\mathcal{F},\leq)$, hence, for instance, by applying $\mathcal{R}$ to a complete lattice we obtain a Lattice Learning Space.
	
	A Learning Space is completely defined by a triple $(\mathcal{F},\leq,\mathcal{R})$, in which the elements of $\mathcal{F}$ may be interpreted as sets of parameters which describe a subset of hypotheses, i.e., the hypotheses in $\mathcal{R}(a), a \in \mathcal{F},$ are represented by the parameters $a$, so that, in particular, $\mathcal{F}$ generates a parametric representation of the functions in $\mathcal{H}$. For this reason, we call $(\mathcal{F},\leq)$ a parametric poset of $\mathcal{H}$. Therefore, in general, to build a Learning Space of $\mathcal{H}$ we apply a generator to a parametric poset of its hypotheses. 
	
	\subsection{Examples of Learning Spaces}
	\label{SecExamplesLS}
	
	We present some examples of Learning Spaces completely defined by a triple $(\mathcal{F},\leq,\mathcal{R})$.
	
	\begin{example}[Variable Selection]
		\label{feature_lattice} \normalfont
		Assume that $\mathcal{H}$ is a space of functions with domain $\mathcal{X} \subset \mathbb{R}^{d}, d > 1$. Let $\mathcal{F} = \mathcal{P}(\{1,\dots,d\})$ be the power set of $\{1,\dots,d\}$ partially ordered by inclusion, so that $(\mathcal{F},\subset,\cap,\cup,\emptyset,\{1,\dots,d\})$ is a complete Boolean lattice. Consider the Learning Space generator $\mathcal{R}: \mathcal{F} \mapsto Im(\mathcal{R}) \subset \mathcal{P}(\mathcal{H})$ given by
		\begin{equation*}
			\mathcal{R}(a) = \Big\{h \in \mathcal{H}: h(x) = h(z), \text{ if } x \equiv_{a} z\Big\},
		\end{equation*} 
		in which $a = \{a_{1},\dots,a_{j}\} \in \mathcal{F}$ and $x = (x_{1},\dots,x_{d}) \equiv_{a} z = (z_{1},\dots,z_{d})$ if, and only if, $x_{a_{i}} = z_{a_{i}}$ for $i = 1,\dots,j$, so $\mathcal{R}(a)$ contains the hypotheses which depend solely on variables in $a$. The lattice isomorphism $\mathcal{R}$ satisfies condition (i) and often satisfies (ii), as in many applications the VC dimension is an increasing function of the number of variables, so $Im(\mathcal{R})$ is often a Learning Space. This is the usual collection of candidate models for variable selection problems.
		
		\hfill$\square$
	\end{example}
	
	\begin{example}[Partition Lattice]
		\label{partition_lattice} \normalfont
		Let $\mathcal{X}$ be an arbitrary set with $\text{\textbar}\mathcal{X}\text{\textbar} < \infty$ and let $\mathcal{H} = \{h: \mathcal{X} \mapsto \{0,1\}\}$ be the set of all functions from $\mathcal{X}$ to $\{0,1\}$. Assuming $Z = (X,Y)$ and $\mathcal{Z} = \mathcal{X} \times \{0,1\}$, under the simple loss function $\ell((x,y),h) = \mathds{1}\{h(x) \neq y\}$, it follows that $L(h) = \mathbb{P}(h(X) \neq Y)$ is the classification error.
		
		Denote by $\Pi = \{\pi: \pi \text{ is a partition of } \mathcal{X}\}$ the set of all partitions of $\mathcal{X}$. A partition $\pi = \{p_{1},\dots,p_{k}\}$ is a collection of subsets of $\mathcal{X}$, called blocks or parts, such that
		\begin{align*}
			\bigcup_{i=1}^{k} p_{i} = \mathcal{X} & & p_{i} \cap p_{i^\prime} = \emptyset \text{ if } i \neq i^\prime.
		\end{align*}
		Each partition $\pi \in \Pi$ induces an equivalence relation on $\mathcal{X}$, with equivalence between points in the same block of partition $\pi$:
		\begin{equation*}
			x \equiv_{\pi} y \iff \exists p \in \pi \text{ such that } \{x,y\} \subset p.
		\end{equation*}
		Consider in $\Pi$ the partial order $\leq$ defined as
		\begin{equation*}
			\pi_{1} \leq \pi_{2} \text{ if, and only if, } x \equiv_{\pi_{2}} z \text{ implies } x \equiv_{\pi_{1}} z
		\end{equation*}
		for $\pi_{1}, \pi_{2} \in \Pi$, which turns it into a complete lattice $(\Pi,\leq,\wedge,\vee,\{\mathcal{X}\},\mathcal{X})$. This partial order is equivalent to $\pi_{1} \leq \pi_{2}$ if, and only if, for every $p_{2} \in \pi_{2}$ there exists a $p_{1} \in \pi_{1}$ such that $p_{2} \subset p_{1}$. See Figure \ref{partitionL} for an example of a partition lattice. By applying the generator $\mathcal{R}: \mathcal{F} \mapsto Im(\mathcal{R}) \subset \mathcal{P}(\mathcal{H})$ given by
		\begin{equation*}
			\mathcal{R}(\pi) = \mathcal{M}_{\pi} \coloneqq \Big\{h \in \mathcal{H}: h(x) = h(z) \text{ if } x \equiv_{\pi} z \Big\},
		\end{equation*}
		for $\pi \in \Pi$, we obtain the partition lattice Learning Space $\mathbb{L}(\mathcal{H}) \coloneqq Im(\mathcal{R})$. The set $\mathcal{R}(\pi)$ is formed by all hypotheses which classify the points inside a block of $\pi$ in the same category; those are the hypotheses which respect $\pi$. This lattice is indeed a Learning Space since $\mathcal{H} \in \mathbb{L}(\mathcal{H})$ and $d_{VC}(\mathcal{M}_{\pi}) = |\pi|$. In this case, the parameters of the functions $h \in \mathcal{H}$ are the elements in their domain $\mathcal{X}$, in contrast, for example, to the variables they depend on, as in Example \ref{feature_lattice}.
		
		For any distribution $P$, $d_{VC}(\mathcal{M}^{\star}) \leq 2$ since $h^{\star} \in \mathcal{M}_{\pi^{\star}}$ in which $\pi^{\star} = \{\{x: h^{\star}(x) = 0\},\{x: h^{\star}(x) = 1\}\}$ is the partition generated by $h^{\star}$.
		
		\hfill$\square$
		
		\begin{figure}[ht]
			\centering
			\includegraphics[width=\textwidth]{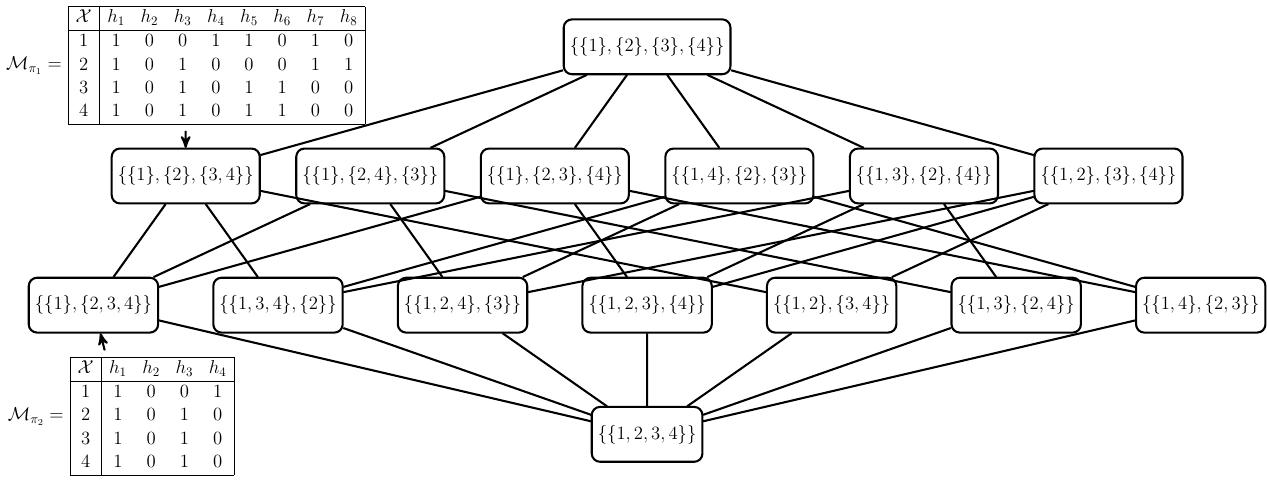}
			\caption{The set $\Pi$ of all partitions of $\mathcal{X} = \{1,2,3,4\}$. The tables present the hypotheses in selected models $\mathcal{M}_{\pi_{1}}, \mathcal{M}_{\pi_{2}}$.}
			\label{partitionL}
		\end{figure}
	\end{example}
	
	\begin{example}[Linear Regression]
		\label{parametric_lattice} \normalfont
		Let $\mathcal{H}$ be given by the linear functions in $\mathbb{R}^{d}, d \geq 1$:
		\begin{equation*}
			\mathcal{H} = \Bigg\{h_{a}(x) = a_{0} + \sum_{i=1}^{d} a_{i}x_{i}: a_{i} \in \mathbb{R}\Bigg\},
		\end{equation*}
		in which $x = (x_{1},\dots,x_{d}) \in \mathbb{R}^{d}$ and $h_{a}$ is the function indexed by its parameters $a = (a_{0},\dots,a_{d}) \in \mathbb{R}^{d+1}$. Assume $Z = (X,Y)$ and $\mathcal{Z} = \mathbb{R}^{d} \times \mathbb{R}$, and consider the square loss function $\ell((x,y),h) = (h(x) - y)^{2}$, so $L(h) = \mathbb{E}[(h(X) - Y)^{2}]$ is the mean squared error.
		
		Denoting $\mathcal{A} = \{1,\dots,d\}$, we consider two distinct Learning Space generators: from the Boolean lattice $(\mathcal{P}(\mathcal{A}),\subset,\cap,\cup,\emptyset,\mathcal{A})$ and from the partition lattice $(\Pi_{\mathcal{A}},\leq,\wedge,\vee,\{\mathcal{A}\},\mathcal{A})$ of $\mathcal{A}$, in which $\mathcal{P}(\mathcal{A})$ is the power set of $\mathcal{A}$ and $\Pi_{\mathcal{A}}$ is the set of all partitions of $\mathcal{A}$. The partition lattice is represented in Figure \ref{partitionL} for $d = 4$.
		
		Define $\mathcal{R}_{1}: \mathcal{P}(\mathcal{A}) \mapsto \mathcal{P}(\mathcal{H})$ as
		\begin{equation*}
			\mathcal{R}_{1}(A) = \Big\{h_{a} \in \mathcal{H}: a_{j} = 0  \text{ if } j \notin A \cup \{0\}\Big\},
		\end{equation*}
		for $A \in \mathcal{P}(\mathcal{A})$ as a variable selection generator, and define $\mathcal{R}_{2}: \Pi_{\mathcal{A}} \mapsto \mathcal{P}(\mathcal{H})$ as
		\begin{equation*}
			\mathcal{R}_{2}(\pi) = \Big\{h_{a} \in \mathcal{H}: a_{j} = a_{k} \text{ if } j \equiv_{\pi} k\Big\},
		\end{equation*}
		for $\pi \in \Pi_{\mathcal{A}}$ as a generator that constrains parameters to be equal within equivalence classes of $\mathcal{A}$. Both $\mathcal{R}_{1}, \mathcal{R}_{2}$ clearly satisfy (i) and (ii). Therefore, these lattice isomorphisms generate two distinct Lattice Learning Spaces for the same hypotheses space $\mathcal{H}$. Isomorphic Learning Spaces also apply to the hypotheses space of linear classifiers.	
		
		\hfill$\square$
	\end{example}
	
	\section{Numerical experiments}
	\label{SecNum}

	This section presents two numerical studies designed to illustrate the theoretical results established in Sections \ref{boundedL} to \ref{SecReuse}. The first study (Section \ref{SecPLLS}) considers learning a classifier on a finite domain using the partition lattice Learning Space, and focuses on the role of the maximum discrimination error $\epsilon^\star$ and the VC dimension $d_{VC}(\mathbb{C}(\mathcal{H}))$ in determining when learning via model selection outperforms ERM.
	
	The second study (Section \ref{SecRegression}) considers high-dimensional linear regression with two distinct Learning Spaces, the variable selection Learning Space and the partition lattice Learning Space, and compares learning via model selection against ERM, LASSO and ridge regression across four scenarios that span a wide range of alignment between domain knowledge and the true target function, from perfect to complete misalignment. All simulations are implemented in \textbf{R} \cite{R} and the code will be made available. The simulations of Section \ref{SecRegression} were performed using nodes equipped with Intel Xeon Platinum 8274 processors (3.2 GHz, 48 cores per node). 
	
	The four regression scenarios in Section \ref{SecRegression} are explicitly designed to instantiate the theoretical decomposition of the generalization error into types II, III, and IV estimation errors (cf. Section \ref{SecErrors} and Figure \ref{fig_error}). When the Learning Space is well-adapted to the target, i.e., the target model is simple, and the search algorithm successfully locates a low-complexity \textit{good} model, the theory predicts, via Corollaries \ref{cor_typeIV} and \ref{cor_typeIV2}, substantially lower Type IV error than ERM, since the bound depends on $d_{VC}(\hat{\mathcal{M}})$ rather than $d_{VC}(\mathcal{H})$. The simulations evidence this quantitatively. When the search fails, Type III error is large and performance degrades, irrespective of the theoretical adaptation of the Learning Space, illustrating the binding role of the optimization algorithm as a practical constraint.
	
	Together, the two studies address the practical dimension of the framework. Section \ref{SecPLLS} clarifies under what conditions on $\epsilon^\star$ and sample size model selection is most beneficial relative to ERM. Section \ref{SecRegression} illustrates that the theoretical gains are realizable in high-dimensional settings ($d \in \{100, 500, 1000\}$ variables), provided the prior information encoded in the Learning Space is at least partially aligned with the true target and a proper search of the Learning Space can be performed.

	\subsection{Partition Lattice Learning Space for classification}
	\label{SecPLLS}
	
	In this section, we compare learning via ERM on the whole hypotheses space with learning via model selection in the partition lattice Learning Space of Example \ref{partition_lattice}. We consider the \textit{worst-case} distribution for a fixed value of $\epsilon^{\star}$, that is, when the joint distribution of $(X,Y)$ has maximum conditional entropy. Clearly, this happens when $\mathbb{P}(X = x) = 1/|\mathcal{X}|$ for all $x \in \mathcal{X} = \{1,\dots,K\}$ and
	\begin{align}
		\label{cond_max_entropy}
		|\mathbb{P}(Y = 1|X = x) - \mathbb{P}(Y = 0|X = x)| = |\mathcal{X}|\epsilon^{\star}, \ \text{ for all } x \in \mathcal{X}.
	\end{align}
	There are $2^{|\mathcal{X}|}$ distributions that satisfy \eqref{cond_max_entropy} and we choose the one such that $\mathbb{P}(Y = 1|X = x) > \mathbb{P}(Y = 0|X = x)$ if $x$ is odd and $\mathbb{P}(Y = 1|X = x) < \mathbb{P}(Y = 0|X = x)$ if $x$ is even. For any value of $\epsilon^{\star}$, $h^{\star}(x)$ equals one for $x$ odd and zero for $x$ even, and $\mathcal{M}^{\star} = \mathcal{M}_{\pi}$ in which $\pi$ is the partition of even and odd numbers with $d_{VC}(\mathcal{M}^{\star}) = 2$.
	
	For each $\epsilon^{\star} \in \{0.0025,0.005,0.0125,0.025,0.05,0.075,0.1\}$ we simulate $1,000$ samples of total size $n = M + N \in \{16,32,64,128,256\}$ considering $\mathcal{X} = \{1,\dots,8\}$ and the worst-case distribution described above. The distributions considered are in Table \ref{jointDist} in Appendix \ref{AppResults}. For each sample, we learn by ERM with the whole sample, learn via model selection with an independent sample by dividing the sample into 50\% - 50\% for training $(N = n/2)$ and independent $(M = n/2)$ sample, and learn via model selection by reusing. Model risks are estimated by k-fold cross-validation with $k = 4$, and when the solution is not unique, i.e., $\hat{\mathcal{M}}$ is an equivalence class with more than one model, learning is performed by minimizing the respective empirical risk in the union of the models in $\hat{\mathcal{M}}$.
	
	Figure \ref{fig_PLLS1} presents for each scenario the number of simulated samples in which the risk of the estimated hypotheses via model selection was greater, lesser or equal to the risk of the ERM hypotheses. We see that for all $\epsilon^{\star}$, the number of samples in which learning via model selection is better than ERM decreases with $n$, illustrating that learning via model selection is more beneficial for small sample sizes. For all sample sizes, the number of samples in which learning via model selection is better decreases with $\epsilon^{\star}$, and for large values of $\epsilon^{\star}$ ($\geq 0.05$) learning via model selection is as good as via ERM in the majority of cases for mild to large sample sizes. Finally, learning via model selection by reusing is, in general, slightly better than with an independent sample relative to learning via ERM.
	
	The results in Figure \ref{fig_PLLS1} illustrate the effect of $\epsilon^{\star}$ in the generalization of learning via model selection. On the one hand, if $\epsilon^{\star}$ is too large, the learning problem is actually quite \textit{easy} and model selection is not necessary. For example, the conditional distribution for $\epsilon^{\star} = 0.1$ is such that $\mathbb{P}(Y = 1|X = 1) = 0.9$ (cf. Table \ref{jointDist}) so we do not need many samples with $X = 1$ to figure out that $h^{\star}(1) = 1$. Therefore, with or without model selection, we can properly learn $h^{\star}$.
	
	On the other hand, when $\epsilon^{\star}$ is small, the learning problem becomes \textit{harder}. For example, the conditional distribution for $\epsilon^{\star} = 0.0025$ is such that $\mathbb{P}(Y = 1|X = 1) = 0.51$ so more samples with $X = 1$ are needed to establish that $h^{\star}(1) = 1$. However, if we could somehow figure out that, say, $h^{\star}(1) = h^{\star}(3)$ then we would pool the samples with $X = 1$ and $X = 3$ together to figure the value of $h^{\star}$ at these points. Therefore, by knowing a partition of the domain with points in the same block having the same value of $h^{\star}$, we need fewer samples in the independent sample to properly learn. In the limit case, when the partition selected is that of the even and odd numbers, learning with an independent sample will be performed in a (\textit{good}) hypotheses space with VC dimension 2 and few samples will suffice.
	
	\begin{figure}[ht]
		\centering
		\includegraphics[width=\linewidth]{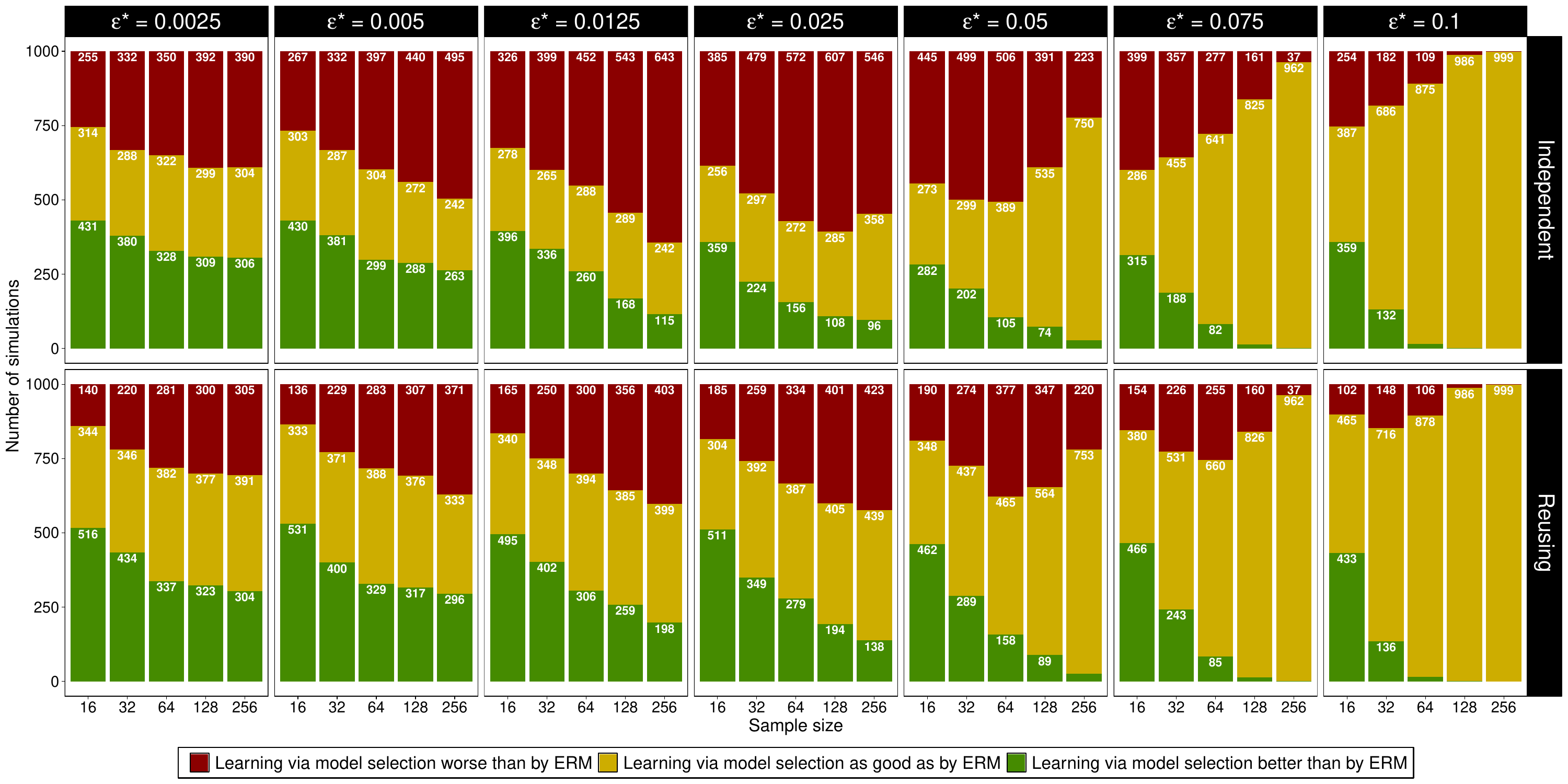}
		\caption{Number of simulations in which learning via model selection in the whole partition lattice Learning Space is better, worse, and as good as learning via ERM with respect to the risk of the estimated hypotheses.} \label{fig_PLLS1}
	\end{figure}
	
	\begin{figure}[ht]
		\centering
		\includegraphics[width=\linewidth]{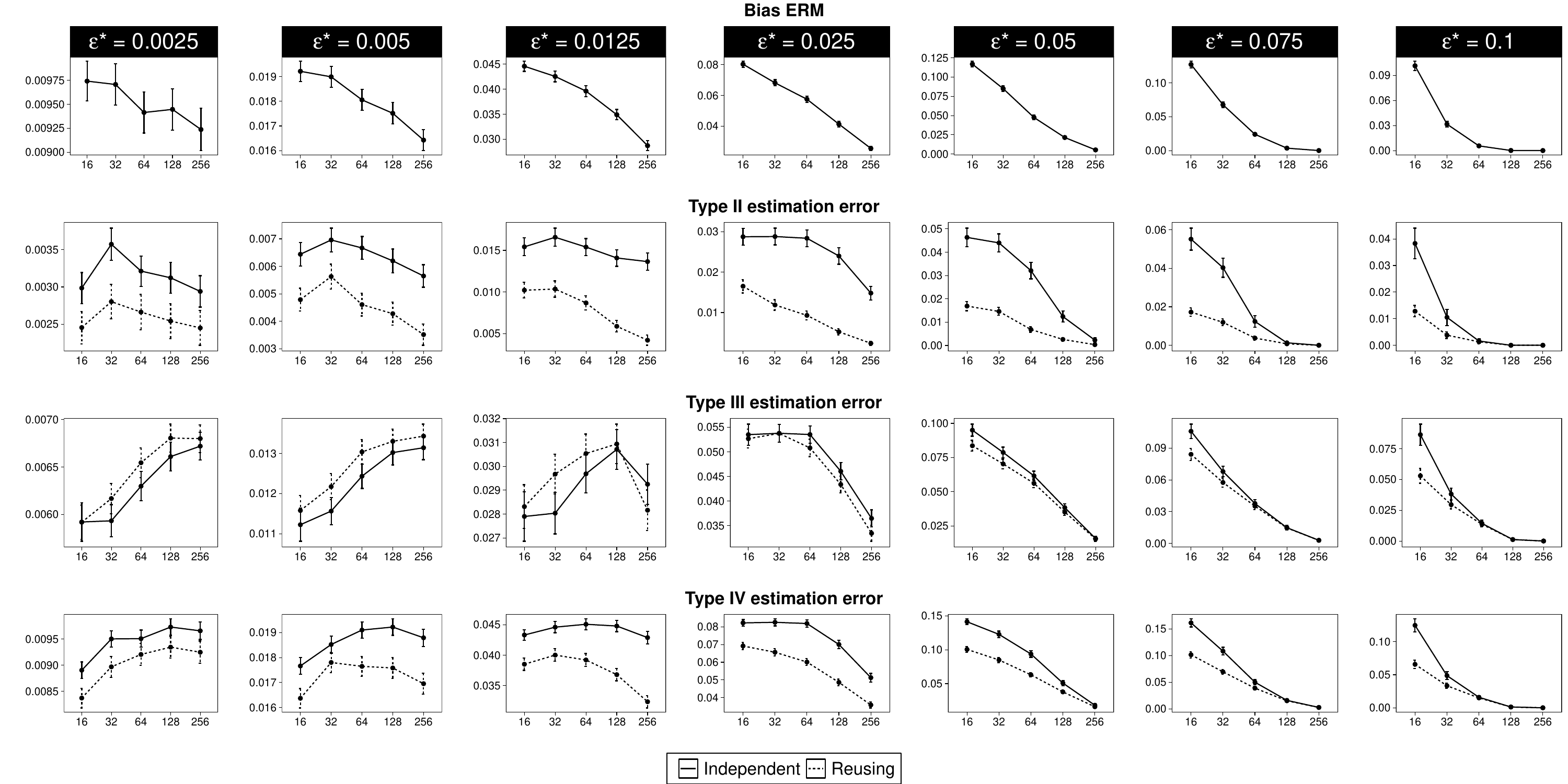}
		\caption{The average with its 95\% confidence interval of the ERM hypotheses bias and types II, III, and IV estimation errors over the 1,000 simulated samples for each case when learning via model selection in the whole partition lattice Learning Space. These results are also presented in Table \ref{res_PLLS} in Appendix \ref{AppResults}.} \label{fig_PLLS2}
	\end{figure}
	
	Figure \ref{fig_PLLS2} presents the average bias of the ERM hypotheses and types II, III, and IV estimation errors over the 1,000 simulated samples for each case. For $\epsilon^{\star} \geq 0.05$, the average estimation errors decrease with the sample size, while the average bias of the ERM hypotheses decreases with $n$ for all $\epsilon^{\star} \geq 0.005$. For $\epsilon^{\star} \leq 0.025$, the average of the estimation errors first increases with $n$, then remains stable, until attaining a sample size in which it rapidly decreases, especially for learning with an independent sample. In particular, the average of the bias of the ERM hypotheses is, in general, greater than the average of type IV estimation error for small sample sizes.
	
	The increase in the estimation errors from sample size $16$ to $32$ can be explained by the fact that with $16$ samples the model of the constant functions, related to the partition $\pi = \{\{\mathcal{X}\}\}$, which has VC dimension 1, is often selected: between 22-33\% for small values of $\epsilon^{\star}$ (see Figure \ref{fig_PLLS4} in Appendix \ref{AppResults}). When this model is selected, then the learned hypotheses has a risk equal to $L(h^{\star}) + 4\epsilon^{\star}$, that is, the risk of the constant hypotheses. When more samples are available, models with VC dimension 2 are selected more often and the risk can be as large as $L(h^{\star}) + 8\epsilon^{\star}$, so it is reasonable that the average of estimation errors slightly increases.
	
	The abrupt decrease of the error when $n$ increases, especially for type II estimation error, is due to a phase transition which is a feature of, for example, the bounds in Theorems \ref{theorem_tipeIII} and \ref{theorem_tipeIII2} and Corollaries \ref{cor_typeIV} and \ref{cor_typeIV2}, because of the term $\epsilon \vee \epsilon^{\star}$. When there are not enough samples to properly estimate the risks with a precision of $\epsilon^{\star}/2$ with high probability, it is unlikely that a model \textit{close}, in the type III estimation error sense, to $\mathcal{M}^{\star}$ is selected. But once there are enough samples for type III estimation error to be low, learning improves and types II and IV estimation errors decrease significantly.
	
	We end this section by analyzing the case of learning via model selection, considering only the models with VC dimension 2 in the partition lattice Learning Space. This family of candidate models has VC dimension two and $2^{7} - 1$ maximal elements, since all of its elements are maximal, but it has the same $\epsilon^{\star}$ as the whole partition lattice Learning Space. The comparison with learning via ERM for this case is in Figure \ref{fig_PLLS3}. The performance of model selection relative to learning via ERM improves substantially for this family of candidate models when compared to the whole partition lattice. For instance, learning via ERM is better than learning via model selection by reusing in no more than 8.4\% of the simulated samples for all sample sizes, and learning via model selection is strictly better in as many as two thirds of the samples for small sample sizes. However, as $\epsilon^{\star}$ and the sample size increase, learning via ERM becomes as good as learning via model selection, since the learning problem becomes easier to solve. More details about this simulation can be found in Appendix \ref{AppResults}.
	
	\begin{figure}[ht]
		\centering
		\includegraphics[width=\linewidth]{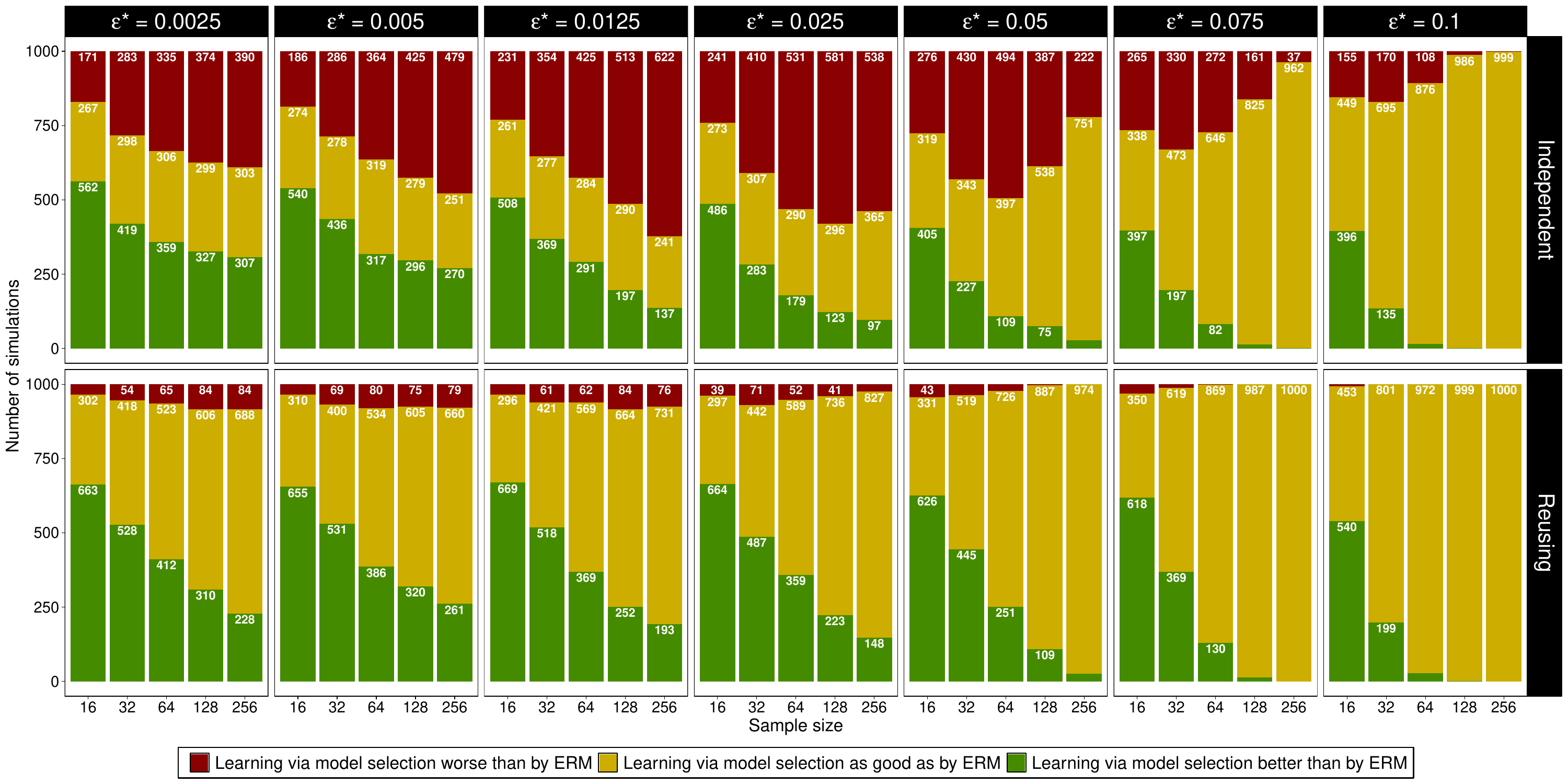}
		\caption{Number of simulations in which learning via model selection in the models of the partition lattice Learning Space with VC dimension 2 is better, worse, and as good as learning via ERM with respect to the risk of the estimated hypotheses.} \label{fig_PLLS3}
	\end{figure}
	
	\subsection{Learning Spaces for linear regression}
	\label{SecRegression}
	
	In this section, we compare learning via ERM on the whole hypotheses space (i.e., ordinary least squares), LASSO \cite{tibshirani1996regression} and ridge \cite{hoerl1970ridge} regression with learning via model selection in linear regression leveraging domain knowledge. We consider two kinds of prior information about $h^{\star}: \mathbb{R}^{d} \to \mathbb{R}$: that it is sparse and that the effect of some of the $d$ input variables is the same. Denoting $h^{\star}(x) = a_{0}^{\star} + \sum_{i=1}^{d} a_{i}^{\star} x_{i}$, sparsity means that $a_{i}^{\star} = 0$ for $i$ in a subset of $\{1,\dots,d\}$ and two input variables $x_{i}, x_{j}$ having the same effect on the output means that $a_{i}^{\star} = a_{j}^{\star}$.
	
	We simulate four cases, combining two levels of sparsity, namely, sparse when only $\ell \coloneqq d/10$ of the $d$ input variables are active and dense when all are active, with two levels of homogeneity of effect, namely, unequal when the active variables have pairwise distinct coefficients and grouped when the active variables are split into few groups that share a common coefficient. Table \ref{tab_scenarios} presents the four target hypotheses considered $h^{\star}_{1},\dots,h^{\star}_{4}$, each with intercept $a_{0}^{\star} = 0$.
	
	\begin{table}[ht]
		\centering
		\caption{The four target hypotheses considered, combining two levels of sparsity with two levels of homogeneity of effect.}
		\label{tab_scenarios}
		\begin{tabular}{c|cc}
			\hline
			Sparsity & \multicolumn{2}{c}{Homogeneity of effect} \\
			& Unequal & Grouped \\
			\hline
			Sparse & $h^{\star}_{1}(x) = \displaystyle\sum_{i=1}^{\ell} \frac{i}{\ell} x_{i}$ & $h^{\star}_{2}(x) = \displaystyle\frac{1}{\ell}\sum_{i=1}^{\ell} x_{i}$ \\[1.5em]
			\hline
			Dense & $h^{\star}_{4}(x) = \displaystyle\sum_{i=1}^{d} \frac{i}{d} x_{i}$ & $h^{\star}_{3}(x) = \displaystyle\frac{1}{2\ell}\sum_{i \text{ odd}} x_{i} + \frac{3}{2\ell}\sum_{i \text{ even}} x_{i}$ \\[1.5em]
			\hline
		\end{tabular}
	\end{table}
	
	We consider a \textit{typical} scenario for each target hypotheses. Let $X$ be a random vector uniformly distributed in $[-1,1]^{d}$ and define, for $i = 1,\dots,4$, the random variable $Y^{(i)} = h^{\star}_{i}(X) + \varepsilon$, in which $\varepsilon$ is independent of $X$ and has mean zero and variance $\sigma^{2} > 0$. Denote the distribution of $(X,Y^{(i)})$ as $P_{i}$, so $h^{\star}_{i}$ is the target hypotheses under $P_{i}$. Considering the square loss function, the risk of each target hypotheses is $L_{i}(h_{i}^{\star}) = \sigma^{2}$, in which $L_{i}$ is the expected loss under $P_{i}$.
	
	First, we analyze the variable selection Learning Space (VSLS) generated by $\mathcal{R}_{1}$ from Example \ref{parametric_lattice}. Under $P_{1}$ and $P_{2}$, the target hypotheses are sparse depending only on the first $\ell$ variables, so $\mathcal{M}_{1}^{\star,(1)} = \mathcal{M}_{1}^{\star,(2)} = \mathcal{R}_{1}(\{1,\dots,\ell\})$ is formed by the linear functions that depend only on the first $\ell$ input variables. Under $P_{3}$ and $P_{4}$, the target hypotheses are dense, so $\mathcal{M}_{1}^{\star,(3)} = \mathcal{M}_{1}^{\star,(4)} = \mathcal{H} = \mathcal{R}_{1}(\{1,\dots,d\})$.
	
	Now consider the partition lattice Learning Space (PLLS) generated by $\mathcal{R}_{2}$ from Example \ref{parametric_lattice}. Under $P_{1}$, the zero coefficients are grouped, but, since the active coefficients are pairwise distinct, they cannot be grouped, hence $\mathcal{M}^{\star,(1)}_{2} = \mathcal{R}_{2}(\{\{1\},\dots,\{\ell\},\{\ell+1,\dots,d\}\})$. Under $P_{2}$, the active variables share a common coefficient, hence can be grouped, and $\mathcal{M}^{\star,(2)}_{2} = \mathcal{R}_{2}(\{\{1,\dots,\ell\},\{\ell+1,\dots,d\}\})$. Under $P_{3}$, the odd and even indexed variables each share a common coefficient, so $\mathcal{M}^{\star,(3)}_{2} = \mathcal{R}_{2}(\{\{i: i \text{ odd}\},\{i: i \text{ even}\}\})$. Finally, under $P_{4}$, all coefficients are pairwise distinct, so $\mathcal{M}^{\star,(4)}_{2} = \mathcal{H} = \mathcal{R}_{2}(\{\{1\},\dots,\{d\}\})$.
	
	We observe that the variable selection Learning Space is especially adapted to scenarios 1 (S1) and 2 (S2) in which the target is sparse, while the partition lattice Learning Space is particularly adapted to the scenarios 2 and 3 (S3) in which the coefficients can be grouped into two groups, and partially adapted to S1 in which the variables can be grouped into $\ell + 1$ groups. In these cases, the target model has low complexity, which, theoretically, can lead to better performance compared with ERM on the whole hypotheses space as long as these sets can be properly learned. Scenario 4 (S4) is dense and the variables cannot be grouped, hence neither Learning Space provides any theoretical gain over ERM since the target model is the whole hypotheses space.
	
	From a prior knowledge perspective, if one knows that the problem is sparse (S1 and S2) then the variable selection Learning Space should be chosen; if, additionally, it is known that many of the active variables have equal effect (S2), then the partition lattice Learning Space should be considered. Now, if one only knows that there are few groups of variables with the same effect (S2 and S3), then the partition lattice Learning Space should be considered. If the prior information is inaccurate, and the true target function is actually dense and the effects cannot be grouped (S4), then both Learning Spaces should perform poorly compared with ERM.
	
	\subsubsection{Simulation details}
	
	For each $d \in \{100,500,1000\}$, with $\ell = d/10$, and each scenario, we consider three values of $\sigma$, called \textit{easy}, \textit{moderate} and \textit{hard}: $\sigma = \sqrt{0.1/3m^{2}}$, $\sigma = \sqrt{1/3m^{2}}$ and $\sigma = \sqrt{10/3m^{2}}$, respectively for increasing difficulty, in which $m \in \{\ell,d\}$ is the number of active variables. As $\sigma$ increases, the signal-to-noise ratio decreases and the problem becomes harder. The value $1/3m^{2}$ was selected as a scale comparable to the value of $\epsilon^{\star}$ for the VSLS in all scenarios.	
	
	For each $(d,\sigma)$ combination we simulate $96$ replications, and for each replication we generate two independent samples of size $N = M \in \big\{\lceil 4(d+5)/3\rceil,\, \lceil 5(d+5)/3\rceil,\, \lceil 6(d+5)/3\rceil\big\}$ from $P_{i}$, considering a Gaussian distribution for $\varepsilon$. For each pair of samples, we learn by ERM, LASSO and ridge regression with the pooled sample of size N + M, and by model selection with an independent sample in both the variable selection and partition lattice Learning Spaces, in which the candidate model is selected by minimizing the k-fold cross-validation risk estimate on one sample and the final hypotheses is learned on the other. Model risk is estimated by k-fold cross-validation with $k=4$.
	
	LASSO and ridge regression estimate the parameters by minimizing the empirical squared loss plus, respectively, an $\ell_{1}$ penalty $\lambda\sum_{i=1}^{d}|a_{i}|$ and an $\ell_{2}$ penalty $\lambda\sum_{i=1}^{d}a_{i}^{2}$ on the slope coefficients, with the former additionally inducing sparsity by shrinking some coefficients exactly to zero. The regularization parameter $\lambda$ is selected by $k$-fold cross-validation under the one-standard-error rule, that is, as the largest $\lambda$ whose cross-validated risk is within one standard error of the minimum, favoring the most regularized, and hence simplest, model among those statistically indistinguishable from the one with the smallest estimated risk. They are fitted using the R package \texttt{glmnet} \cite{glmnet}.
	
	\subsubsection{Optimization algorithms for Learning Spaces}
	
	Given the exponential size of the lattices for $d \in \{100,500,1000\}$, we cannot perform an exhaustive search and we need specialized algorithms for minimizing functions in lattices. For the Boolean VSLS we consider a stepwise approach which, starting from the model without variables, one variable is added/removed at each step. The selected addition or removal is that which leads to the least cross-validation error. When there are ties, we select one of the minimums randomly with equal probability among the ones with least number of variables. The algorithm runs for a pre-specified computation time and returns the model with the least cross-validation error visited during the search. The computation time was set to 120 seconds. This differs from usual stepwise algorithms that stop when no addition/removal of variables decreases the cross-validation error.
	
	For the PLLS we consider the stochastic lattice descent algorithm (SLDA) proposed in \cite{marcondes2023discrete,marcondes2024lattice,marcondes2024algorithm,marcondes2025unrestricted}, which is a generalization of the U-curve algorithms \cite{u-curve1,u-curve3,ucurveParallel,reis2018}, proposed for minimizing U-shaped functions in Boolean lattices. In the implemented SLDA,	the search is initialized at the partition with a single block containing all variables. At each step, a set of candidate neighboring partitions is generated by either splitting one existing block into two, or merging two existing blocks into one. 
	
	Rather than sampling these candidates uniformly, splits are proposed by first regressing the residuals (i.e., point-wise errors) of the current partition model trained with the whole $N$ size training sample onto the raw variables within each block. The probability of each block being selected for splitting is proportional to the $R^2$ of this within-block regression. Observe that large values of $R^2$ are evidence that the residuals are strongly correlated with the variables, which is an indication that considering the same coefficient for all variables is not providing goodness of fit.
	
	When a block is selected, it is then split into two groups: one containing all variables in which the residual regression coefficients are positive and the other all that have non-positive coefficients. The sign of the residual regression coefficient is an indication that the true value of the model coefficient for the variable should be greater or smaller than the common estimated coefficient. This heuristic should lead to splittings that tend to group together variables with true coefficients that are close to each other.
	
	On the other hand, merges of blocks are proposed by sampling pairs of blocks with probability inversely proportional to the squared difference between their current fitted coefficients of the model considering the whole training sample, since blocks with similar coefficients are the best candidates for merging without loss of fit.
	
	At each step, a pre-specified number of split and merge neighbors are sampled with the established probabilities, their cross-validation error is computed, and the algorithm moves to the neighbor with the least cross-validation error, breaking ties first in favor of those with fewest blocks and then uniformly at random. As with the Boolean VSLS, the algorithm runs for a pre-specified computation time and returns the partition with the least cross-validation error visited during the search, rather than stopping at the first local optimum.
	
	The number of sampled neighbors is $128$, $64$ and  $32$ for $d = 100, 500$ and $1000$, respectively, with half being sampled among the splits, and half among the merges. These choices are made to allow the algorithm to explore a bigger region of the bigger lattices in the pre-specified running time, which was also set at $120$ seconds.
	
	\subsubsection{Results}
	
	\begin{figure}[ht]
		\centering
		\includegraphics[width=\linewidth]{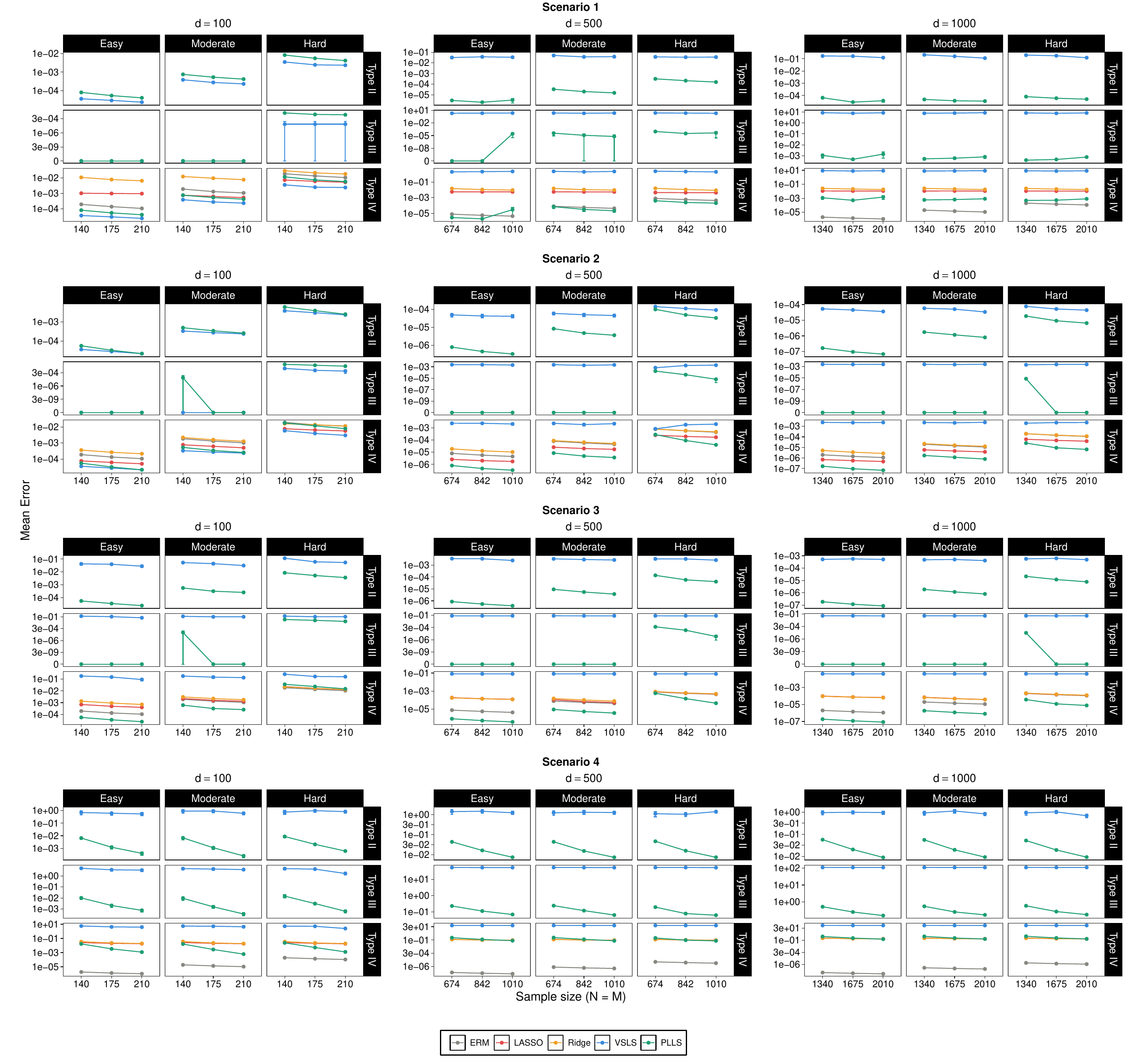}
		\caption{Average across the 96 repetitions with 95\% confidence intervals of the type II, III, and IV estimation errors of learning in each Learning Space for all scenarios, across $d \in \{100,500,1000\}$ and noise levels (easy, moderate, hard), as a function of sample size. The average with 95\% confidence intervals of the difference between the risk of the fitted function and the target is presented for ERM, LASSO and ridge regression as the type IV estimation error in each case.}
		\label{fig_estimation_errors}
	\end{figure}
	
	\begin{figure}[ht]
		\centering
		\includegraphics[width=\linewidth]{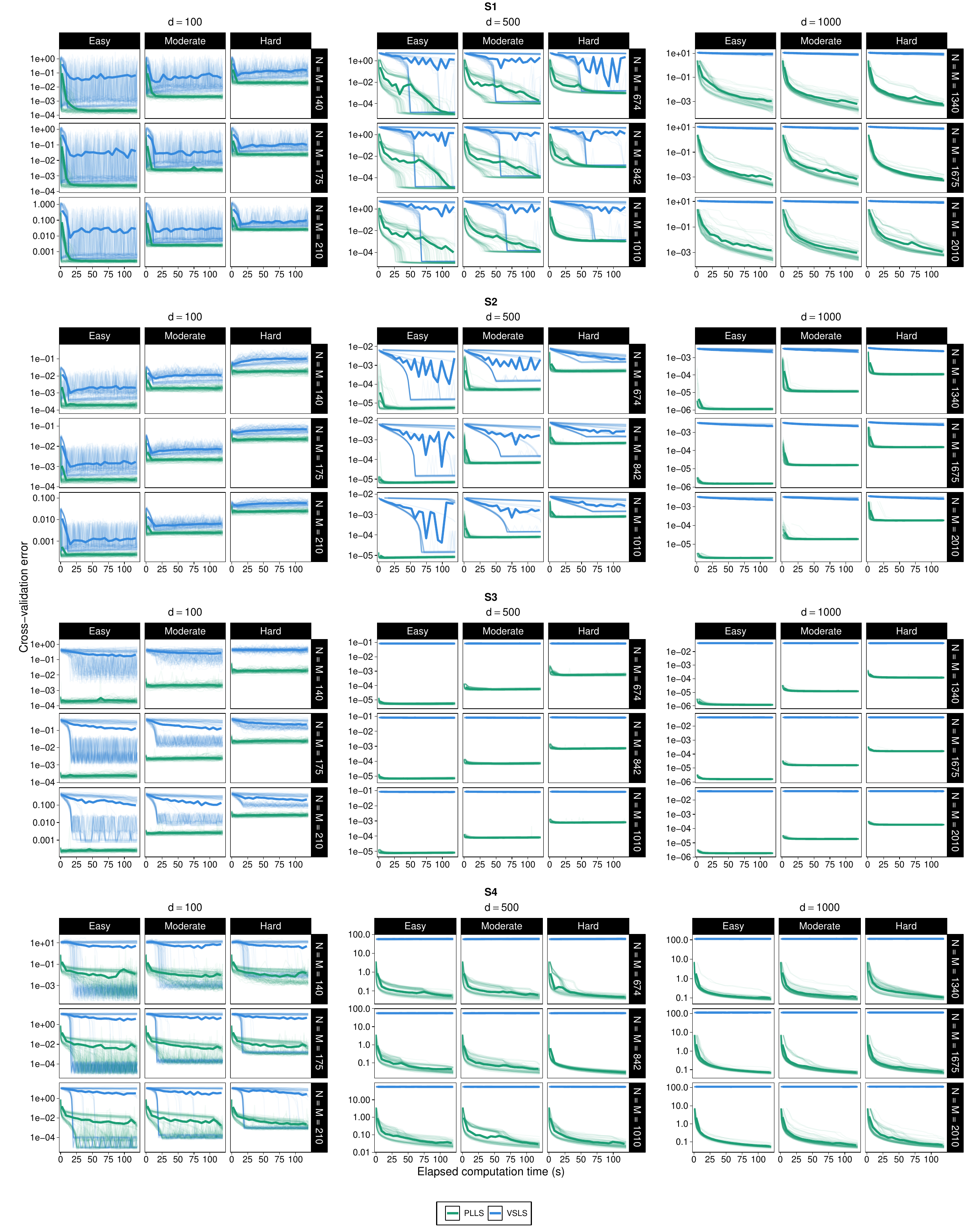}
		\caption{Current cross-validation error against elapsed computation time for the VSLS and PLLS, for each scenario, $d \in \{100,500,1000\}$, noise level, and sample size. Each line is a repetition, with the thick line representing the average over the repetitions.}
		\label{fig_trace_all}
	\end{figure}
	
	Figure \ref{fig_estimation_errors} presents, for each scenario, number of variables, error level and sample size, the average with 95\% confidence intervals of the types II, III, and IV estimation errors of learning in each Learning Space. The average with 95\% confidence intervals of the difference between the risk of the fitted function and the target is presented for ERM, LASSO and ridge regression as the type IV estimation error in each case. Table \ref{tab.typeIV} in Appendix \ref{AppResults} presents the detailed results for each case. Figure \ref{fig_trace_all} presents the cross-validation error at each step of the search on the VSLS and PLLS for all repetitions in each case. 
	
	\textbf{Scenario 1}. The target is sparse with unequal active coefficients, so the VSLS is theoretically well-adapted, while the PLLS is only partially adapted. The results evidence this clearly for $d = 100$: VSLS achieves the lowest Type IV estimation error across all noise levels and sample sizes, with values on the order of $10^{-5}$ to $10^{-3}$, substantially below ERM ($10^{-4}$ to $10^{-2}$) and far below LASSO and ridge, while PLLS sits above VSLS in all cases. As $d$ grows, however, the picture changes sharply for VSLS: with only $\sim 30$-$45$ search steps available within 120 seconds, the stepwise algorithm cannot adequately explore the $2^d$-element Boolean lattice, and its Type IV estimation error collapses to values orders of magnitude worse than any competitor. On the other hand, PLLS, running the SLDA with 110-140 steps, is better, or as competitive as ERM, for all cases at $d = 500$. For $d = 1000$, the PLLS has a Type IV estimation error orders of magnitude lower than LASSO and ridge regression, but is not as low as ERM. We observe that learning via model selection has a better performance when the Type III estimation error is low, or even zero, meaning a \textit{good} model was selected; when Type III error is large, especially for the VSLS in higher dimensions, the performance cannot be as high. This scenario outlines an important practical message: even when the VSLS is theoretically optimal for a sparse target, its stepwise search algorithm does not scale, and the PLLS with the SLDA provides a computationally viable alternative that leverages the partial grouping structure of the target function to achieve remarkable generalization at high dimension.
	
	\textbf{Scenario 2}. This scenario combines sparsity with grouped active coefficients, making it suited to both Learning Spaces simultaneously. At $d = 100$, VSLS again achieves the best Type IV error (order $10^{-5}$ to $10^{-3}$), closely followed by PLLS, while LASSO performs competitively. The PLLS's particular strength emerges as $d$ increases: at $d = 500$ and $d = 1000$, PLLS achieves the lowest Type IV errors of any method, reaching values as low as $\sim 7 \times 10^{-8}$ at $d = 1000$, $n = 2010$, easy, roughly 15 times better than LASSO and 16 times better than ERM in the same setting. This gain is a direct consequence of the PLLS target model's minimal complexity: with only two groups of variables (active and inactive), the SLDA needs very few steps to locate the optimum, explaining why PLLS's step counts at $d = 500$-$1000$ (120-175) remain comparable to or slightly above those in S1, while the performance is far higher. This is also evidenced in the Type III estimation error, that is zero in many cases, evidencing that a model that contains the target was selected. On the other hand, VSLS again degrades for $d \geq 500$ for the same search-cost reasons as in S1. The results in S2 thus provide the clearest empirical evidence of the central theoretical message: structuring the collection of candidate models to match domain knowledge yields results that beat ERM and regularization (LASSO and ridge).
	
	\textbf{Scenario 3}. The target is dense with grouped coefficients, where all $d$ variables are active but split into two groups with distinct shared coefficients. Here VSLS has no theoretical advantage over ERM, while PLLS is ideally adapted. The results confirm this asymmetry: the PLLS achieved the lowest Type IV estimation error with the exception of the hard noise level for $d = 100$. In particular, the true advantage of PLLS materializes at $d = 500$ and $d = 1000$, where it dominates all methods across all noise levels and sample sizes, by multiple orders of magnitude in many cases. On the other hand, VSLS, as expected theoretically, performs very poorly for $d \geq 500$, with Type IV errors near 0.08-                                                                                                                                                         0.04 that barely decrease with $n$, reflecting the fact that the search cannot locate the full-variable model within the 120-second budget. These results underscore that the PLLS's advantage is not merely theoretical: for dense problems with groupable coefficients, the SLDA efficiently exploits the two-group structure to achieve generalization that scales gracefully with $d$.
	
	\textbf{Scenario 4}. This is the adversarial case for both Learning Spaces in which the target is dense with all $d$ variables active and pairwise distinct coefficients. The results confirm this unambiguously: ERM dominates across every combination of $d$, $n$, and noise level, achieving Type IV estimation error orders of magnitude smaller than any other method.
	
	\textbf{Evolution of optimization algorithms.} The search algorithm of the VSLS has two distinct behaviors. The first, observed in the cases it struggles most ($d = 500$ for S3 and S4, and $d = 1000$), is characterized by the cross-validation error not decreasing over time and a low variability among the repetitions. The second is characterized by high variability among repetitions; a sudden decrease in the error; and high variability over time evidencing visits to subsets with higher cross-validation error after a (local) minimum is found. 
	
	On the other hand, the SLDA in general converges, especially in lower dimension and S1, S2 and S3, without large variability over time: after a \textit{good} minimum is found, the cross-validation error does not increase substantially afterwards. Furthermore, no abrupt decrease on the error is seen, with the exception of the first steps in some cases. Nevertheless, in some cases, such as S1 and S4 in $d = 500$ and  S4 in $d = 1000$, there is a variability among the repetitions, with some converging faster  than the others outlining that the stochasticity of the algorithm may lead to distinct convergence behaviors.
	
	\subsubsection{Comment}
	
	The results above highlight that the performance of learning via model selection depends jointly on two distinct factors: the alignment between the Learning Space and the target hypotheses, and the ability of the search algorithm to locate the target model within the allotted computation budget. When both conditions are met, as in S2 under PLLS and S1 under VSLS at $d = 100$, the Type III estimation error is zero across replications, meaning the selected model contains the target, and the Type IV estimation error is substantially below that of ERM and regularization methods. When the Learning Space is well-adapted but the search fails to locate the target model, as with VSLS at $d \geq 500$, the Type III error grows and directly degrades performance, irrespective of the theoretical advantage conferred by the Learning Space structure. This makes the quality of the optimization algorithm a binding practical constraint, and suggests that bounding the gap between the model returned by the SLDA and the global minimum as a function of computation time is an important open problem.
	
	The results also clarify the effect of partially correct prior knowledge. Scenario 1 under the PLLS is precisely such a case: the target is sparse with pairwise distinct active coefficients, so the PLLS target model has $\ell + 1$ blocks rather than the minimal two blocks of S2, meaning the imposed structure is imperfect. Nevertheless, PLLS still outperforms ERM and regularization at $d = 500$ and $d = 1000$, illustrating that any reduction in the complexity of the target model, even partial, translates into improved generalization, with the magnitude of the gain related to how much the structure reduces the complexity of the target model relative to the full hypotheses space.
	
	Finally, the four scenarios are designed to span a wide range of alignment between domain knowledge and the true target: from perfect alignment, to partial alignment, to complete misalignment (S4 for both). The Type II, III, and IV estimation errors reported in Figure \ref{fig_estimation_errors} directly instantiate the theoretical decomposition of the generalization error: Type II measures the approximation gap of the selected model, Type III measures the model selection error, and Type IV measures the total gap between the fitted hypotheses and the target. The consistent pattern across all scenarios, that lower target model complexity leads to lower Type IV error when the search successfully minimizes Type III error, provides direct empirical corroboration of the paper's central theoretical message. In particular, it illustrates when learning via model selection with cross-validation error estimation based on prior information can beat ERM and established regularization methods.

	\begin{remark}
		We note that the stepwise algorithm for the VSLS could be adapted by also sampling candidate models at each step to increase the range of the lattice that is visited with the limited computation budget. However, a good heuristic would be necessary to assign probabilities to each neighboring model.
	\end{remark}
	
	\FloatBarrier
			
	\section{Inserting domain knowledge into Learning Spaces increases generalization}
	\label{SecPIEnhance}
	
	The example in Section \ref{SecRegression} is a special case in which generalization may be increased by inserting prior information into the Learning Space. If it is known that the target hypotheses is sparse, then the variable selection Learning Space should be chosen to increase generalization, while if it is known that some input variables have the same effect, the partition lattice Learning Space should be chosen.
	
	In ERM, in order to increase generalization, much stronger domain knowledge must be available to consider $\mathcal{H}$ as a simpler hypotheses space without adding a large bias. For instance, knowing that the target is sparse is not enough, and it is necessary to know exactly on what variables it depends on. Moreover, knowing that some variables have the same effect is not enough, and it is necessary to know which variables have the same effect. Conversely, the Learning Spaces can leverage weaker domain knowledge to increase generalization, as illustrated in Section \ref{SecRegression}.
	
	Domain knowledge can be especially leveraged to decrease the complexity of $\mathcal{M}^{\star}$, which should enable greater generalization. Indeed, if it is too complex, then there would only be the possibility of learning a target hypotheses if a highly complex model was selected, and then it was learned from it. This scenario would require a large computational budget to search the Learning Space and a large independent sample to properly learn the target hypotheses once a highly complex model that contains it is selected; actually, a large sample size might also be necessary to select such a complex model. On the other hand, if $\mathcal{M}^{\star}$ has low complexity, then it might be more likely to select it, or a low-complexity model that contains it, and fewer samples and computational resources are necessary to both select a suitable model and learn a hypotheses from it. The results in Figure \ref{fig_estimation_errors} illustrate this fact.
	
	Figure \ref{candidates} presents some examples of how the quality of prior information can be associated with the complexity of $\mathcal{M}^{\star}$. In Figure \ref{candidates} (a), the complexity of $\mathcal{M}^{\star}$ is low, so generalization might be higher. The partition lattice Learning Space of linear functions in Section \ref{SecRegression} when the target is $h_{3}^{\star}$ is an example of this scenario. In Figure \ref{candidates} (b), the complexity of $\mathcal{M}^{\star}$ is not as low, but generalization might be higher than in the scenario in Figure \ref{candidates} (c), in which $\mathcal{M}^{\star}$ is highly complex. The partition lattice Learning Space of linear functions in Section \ref{SecRegression} when the target is $h_{1}^{\star}$ is an example of Figure \ref{candidates} (b) and the variable selection Learning Space when the target is $h^{\star}_{3}$ is an example of Figure \ref{candidates} (c).
	
	\begin{figure*}[ht]
		\centering
		\includegraphics[width=\linewidth]{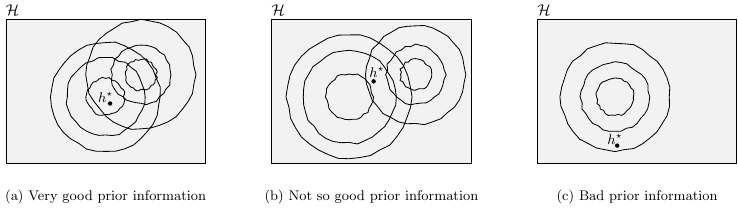}
		\caption{Three examples illustrating how the quality of prior information is associated with the complexity of $\mathcal{M}^{\star}$. The circles represent some models in $\mathbb{L}(\mathcal{H})$ and their area is proportional to the model complexity.} \label{candidates}
	\end{figure*}
	
	The reasoning illustrated in Figure \ref{candidates} also argues \textit{against} considering nested models as the family of candidates from a statistical perspective. Nested models represent a chain $\mathcal{M}_{1} \subset \dots \subset \mathcal{M}_{r}$ in the lattice, and the \textit{amount} of domain knowledge has to be \textit{substantial} to choose the \textit{right} chain. This is illustrated in Figure \ref{candidates} (a) in which $h^{\star}$ is in a simple model of one chain, but a complex model of the other chain. In this case, domain knowledge would need to be spot on so that the \textit{right} chain could be selected. In Scenario 3 of Section \ref{SecRegression}, this would mean that only knowing that some variables have the same effect would not be sufficient and more prior information would have to lead one to select a chain that passes through the partition of the even and odd numbers. Moreover, we argue that lattice Learning Spaces also have advantages from a computational perspective, especially when searching for approximately optimal models within the allotted computation budget. 
	
	\section{Final Remarks}
	\label{FinalRemarks}
	
	Model selection techniques have been historically sought to increase generalization. The main concern of these methods has been to control the complexity of the hypotheses space to avoid overfitting. This is explicitly done in model selection by complexity penalization, and in variable selection, since the complexity of the model is related to the number of variables. However, selecting a low-complexity model is only one step to achieving high generalization, since it is also necessary to have low-risk hypotheses in such a model; otherwise high generalization is unfeasible. In this paper, we investigated how modeling the collection of candidate models based on domain knowledge and prior information may increase generalization.
	
	We presented learning via model selection with cross-validation risk estimation as a systematic data-driven framework consisting of selecting the simplest global minimum of a family of candidate models, and then learning a hypotheses on it with an independent sample or by reusing, seeking to approximate a target hypotheses of $\mathcal{H}$. We studied the distribution-free asymptotics of such a framework by showing the convergence of the estimated model to the target one, and of the estimation errors to zero, for both bounded and unbounded loss functions. The case of bounded loss functions was treated with the usual tools of VC theory, while the case of unbounded loss functions required some new technical results, which are an extension of those in \cite{cortes2019}. 
	
	We introduced the maximum discrimination error $\epsilon^{\star}$, formalized the concept of target model, and illustrated the possibility of better learning with a fixed sample size by properly modeling the family of candidate models with two simulation studies. We argued that by modeling the collection of candidate models based on domain knowledge, it is possible to increase generalization. In particular, we introduced the Learning Spaces and discussed how generalization might be increased by designing them in such a way that the target hypotheses is contained in simple models. The theoretical results also support this assertion, as follows.
	
	First, since $\hat{\mathcal{M}}$ converges to $\mathcal{M}^{\star}$ with probability one, 
	\begin{equation*}
		\mathbb{E}(G(\mathcal{\hat{M}})) \xrightarrow{N \rightarrow \infty} G(\mathcal{M}^{\star}),
	\end{equation*}
	in which $G: \mathbb{C}(\mathcal{H}) \mapsto  \mathbb{R}$ is any real-valued function. The convergence of $\mathbb{E}(G(\mathcal{\hat{M}}))$ ensures that the expectations of functions of $\hat{\mathcal{M}}$ on the right-hand side of inequalities in Theorems \ref{bound_constant} and \ref{bound_constant2} and Corollaries \ref{cor_typeIV} and \ref{cor_typeIV2} tend to the same functions evaluated at $\mathcal{M}^{\star}$ when $N$ tends to infinity. Hence, if one was able to isolate $h^{\star}$ within a model $\mathcal{M}^{\star}$ with small VC dimension, the bounds for types I, II and IV estimation errors will tend to be tighter for a fixed sample size. Furthermore, tighter bounds are obtained in cases in which $d_{VC}(\hat{\mathcal{M}}) \approx d_{VC}(\mathcal{M}^{\star})$, or $\hat{\mathcal{M}} = \mathcal{M}^{\star}$, with high probability (see also Theorems \ref{bound_constant_reusing} and \ref{bound_constant_reusing2}).
	
	Second, if the MDE of $\mathbb{C}(\mathcal{H})$ under $P$ is large, then we need less precision when estimating $L(\mathcal{M})$ for $L(\hat{\mathcal{M}})$ to be equal to $L(\mathcal{M}^{\star})$, and for types III and IV estimation errors to be smaller than an $\epsilon \ll \epsilon^{\star}$ with high probability, so fewer samples are needed to learn a model as good as $\mathcal{M}^{\star}$ and to have smaller types III and IV estimation errors. Moreover, the sample complexity to learn this model is that of the most complex model in $\mathbb{C}(\mathcal{H})$, hence is at most the complexity of a model with VC dimension $d_{VC}(\mathbb{C}(\mathcal{H}))$, which may be smaller than that of $\mathcal{H}$. However, there may be a trade-off between $d_{VC}(\mathbb{C}(\mathcal{H}))$ and the number $\mathfrak{m}(\mathbb{C}(\mathcal{H}))$ of maximal elements in $\mathbb{C}(\mathcal{H})$, as can be seen on the established bounds.
	
	The simulation study of Section \ref{SecNum} provides direct empirical support for these theoretical assertions. In the regression experiments, learning via model selection in the partition lattice Learning Space substantially outperformed ERM, LASSO, and ridge regression in scenarios where the target model had low complexity and the search algorithm successfully located it (low Type III estimation error), illustrating that a well-designed Learning Space can yield generalization gains orders of magnitude beyond those of standard methods. Conversely, when the search failed to locate the target model, or when the Learning Space was misaligned with the true target (Scenario 4), performance degraded as predicted by the bounds, with ERM dominating. Together, the theoretical bounds and simulation results highlight that the practical benefit of learning via model selection depends jointly on the alignment between the Learning Space and the true target, i.e., prior knowledge, and on the quality of the optimization algorithm used to search it.
	
	Since the bounds were developed for a general hypotheses space with finite VC dimension in a distribution-free framework, they are not the tightest possible in specific cases, hence an interesting topic for future research would be to apply the methods used here to obtain tighter bounds for restricted classes of hypotheses spaces, candidate models and/or data generating distributions. The results of this paper may be extended when distribution-dependent bounds for types I and II estimation errors are available in the framework of Propositions \ref{propVC} and \ref{propVC2}. In particular, a theoretical study of the estimation errors in the examples in Section \ref{SecEnhanceGen} should be straightforward, so they ought to be the first distribution-dependent settings to be investigated in future studies. Bounding the estimation errors in the framework of \cite{mendelson2004importance} is also a promising line of research in this context.
	
	The established bounds are not useful to determine the number of folds in cross-validation, since we have not made use of the dependence between the samples in different folds and just applied a union bound to consider each pair of samples separately. The bounds could be improved and insights about the optimal number of folds could be obtained if this dependence was considered, but it would be necessary to restrict the study to a specific class of models, since such a result should not be possible in a general framework. Insights and refined results in the learning by reusing framework could also be obtained by restricting the study to a specific class of models and data generating distributions, in which the effect of considering the same sample to select the model and learn on it could be better understood.
	
	Although outside the scope of this paper, the computational cost of computing $\hat{\mathcal{M}}$ by solving optimization problem \eqref{Ghat} should also be considered when choosing family $\mathbb{C}(\mathcal{H})$. As evidenced in the simulation study of Section \ref{SecRegression}, this family of candidate models should have some structure that allows an efficient computation of \eqref{Ghat}, or the computation of a suboptimal solution with satisfactory practical results. 
	
	The framework of this paper could be studied considering other risk estimators besides cross-validation. For example, one could consider modeling the family of candidate models based on domain knowledge together with penalization methods, by taking the risk $\hat{L}$ as a penalization of the empirical error. This approach, in both a distribution-free and dependent scenario, could lead to tighter bounds and better practical methods for specific problems. Also, an interesting topic for future research would be to prove the results of this paper relaxing the assumption that the sample is independent and identically distributed.
	
	From an applied perspective, there is great promise in investigating how specific domain knowledge may be leveraged to design a Learning Space to solve specific learning problems. The examples in Section \ref{SecEnhanceGen} illustrate how this can be done, but it is necessary to further study not only how to translate prior information into a suitable Learning Space, but also the effect of this modeling on the generalization in other applied problems. This can be investigated from a theoretical perspective, by developing deviation-bounds in specific scenarios, and empirically by comparing learning via model selection based on domain knowledge with other methods.
	
	\section{Proof of results}
	\label{SecProof}
	
	\subsection{Definition of validation errors}
	\label{SecVal}
	
	\subsubsection{Validation sample}
	
	Fix a sequence $\{V_{N}: N \geq 1\}$ such that $\lim\limits_{N \rightarrow \infty} V_{N} = \lim\limits_{N \to \infty} N - V_{N} = \infty$, and let 
	\begin{align*}
		\mathcal{D}_{N}^{(\text{train})} = \{Z_{l}: 1 \leq l \leq N - V_{N}\} & & \mathcal{D}_{N}^{(\text{val})} = \{Z_{l}: N - V_{N} < l \leq N\}
	\end{align*}
	be a split of $\mathcal{D}_{N}$ into a training and validation sample. The two samples $\mathcal{D}_{N}^{(\text{train})}$ and $\mathcal{D}_{N}^{(\text{val})}$ are independent. The estimator under the validation sample is given by
	\begin{equation*}
		\label{VALLhat}
		\hat{L}_{\text{val}}(\mathcal{M}) \coloneqq L_{\mathcal{D}_{N}^{(\text{val})}}(\hat{h}_{\mathcal{M}}^{(\text{train})}) =  \frac{1}{V_{N}} \sum_{N- V_{N} < l \leq N}  \ell\big(Z_{l},\hat{h}_{\mathcal{M}}^{(\text{train})}\big),
	\end{equation*}
	in which
	\begin{equation*}
		\hat{h}_{\mathcal{M}}^{(\text{train})} = \argminA_{h \in \mathcal{M}} L_{\mathcal{D}_{N}^{(\text{train})}}(h)
	\end{equation*}
	minimizes the empirical risk in $\mathcal{M}$ under $\mathcal{D}_{N}^{(\text{train})}$.
	
	\subsubsection{K-fold cross-validation}
	
	Fix $k \in \mathbb{Z}_{+}$ and assume $N \coloneqq kn$, for a $n \in \mathbb{Z}_{+}$. Then, let 
	\begin{align*}
		\mathcal{D}_{N}^{(j)} \coloneqq \{Z_{l}: (j-1)n < l \leq jn\}, & & j = 1,\dots,k,
	\end{align*}
	be a partition of $\mathcal{D}_{N}$: 
	\begin{align*}
		\mathcal{D}_{N} = \bigcup_{j=1}^{k} \mathcal{D}_{N}^{(j)} & & \text{ and } & & \mathcal{D}_{N}^{(j)} \cap \mathcal{D}_{N}^{(j^{\prime})} = \emptyset \text{ if } j \neq j^{\prime}.
	\end{align*}
	We define
	\begin{equation*}
		\hat{h}^{(j)}_{\mathcal{M}} \coloneqq \argminA_{h \in \mathcal{M}} \ L_{\mathcal{D}_{N}\setminus\mathcal{D}_{N}^{(j)}}(h) = \argminA_{h \in \mathcal{M}} \ \frac{1}{(k-1)n} \sum_{\substack{l \leq (j-1)n \\ \cup \ l > jn}} \ell(Z_{l},h)
	\end{equation*}
	as the hypotheses which minimizes the empirical risk in $\mathcal{M}$ under the sample $\mathcal{D}_{N}\setminus\mathcal{D}_{N}^{(j)}$, that is the sample composed by all folds, but the $j$-th, and
	\begin{equation*}
		\hat{L}_{\text{cv(k)}}^{(j)}(\mathcal{M}) \coloneqq L_{\mathcal{D}_{N}^{(j)}}(\hat{h}^{(j)}_{\mathcal{M}}) =  \frac{1}{n} \sum_{(j-1)n < l \leq jn} \ell(Z_{l},\hat{h}^{(j)}_{\mathcal{M}}),
	\end{equation*}
	as the validation risk of the $j$-th fold.
	
	The k-fold cross-validation estimator of $L(\mathcal{M})$ is then given by
	\begin{equation*}
		\label{CVLhat}
		\hat{L}_{\text{cv(k)}}(\mathcal{M}) \coloneqq \frac{1}{k} \ \sum_{j=1}^{k} \hat{L}_{\text{cv(k)}}^{(j)}(\mathcal{M}),
	\end{equation*}
	that is the average validation risk over the folds. 
	
	\subsection{Results of Section \ref{boundedL}}
	
	We start with a lemma.
	
	\begin{lemma}
		\label{lemma0}
		\begin{equation*}
			\label{inclusion_star}
			\left\{\max\limits_{i \in \mathcal{J}} \text{\textbar}L(\mathcal{M}_{i}) - \hat{L}(\mathcal{M}_{i})\text{\textbar} < \epsilon^{\star}/2\right\} \subset \left\{L(\hat{\mathcal{M}}) = L(\mathcal{M}^{\star})\right\},
		\end{equation*}
	\end{lemma}
	\begin{proof}
		If
		\begin{equation*}
			\max\limits_{i \in \mathcal{J}} \text{\textbar}L(\mathcal{M}_{i}) - \hat{L}(\mathcal{M}_{i})\text{\textbar} < \epsilon^{\star}/2
		\end{equation*}
		then, for any $i \in \mathcal{J}$ such that $L(\mathcal{M}_{i}) > L(\mathcal{M}^{\star})$, we have
		\begin{align}
			\label{ineq11}
			\hat{L}(\mathcal{M}_{i}) - \hat{L}(\mathcal{M}^{\star}) > L(\mathcal{M}_{i}) - L(\mathcal{M}^{\star}) - \epsilon^{\star} \geq 0,
		\end{align}
		in which the last inequality follows from the definition of $\epsilon^{\star}$. From \eqref{ineq11} it follows that the global minimum of $\nicefrac{\mathbb{C}(\mathcal{H})}{\hat{\sim}}$ with the least VC dimension, that is $\hat{\mathcal{M}}$, is such that $L(\hat{\mathcal{M}}) = L(\mathcal{M}^{\star})$. Indeed, from \eqref{ineq11} it follows that $\hat{L}(\mathcal{M}) > \hat{L}(\mathcal{M}^{\star})$ for all $\mathcal{M} \in \mathbb{C}(\mathcal{H})$ such that $L(\mathcal{M}) > L(\mathcal{M}^{\star})$. Hence, since $\hat{L}(\hat{\mathcal{M}}) \leq \hat{L}(\mathcal{M}^{\star})$, we must have $L(\hat{\mathcal{M}}) = L(\mathcal{M}^{\star})$ implying the desired inclusion of events.
	\end{proof}
	
	\begin{proof}[\textbf{Proof of Proposition \ref{proposition_principal}}]
		The result is a direct consequence of Lemma \ref{lemma0}.
	\end{proof}
	
	We state and prove a lemma that will aid the proof of Theorem \ref{theorem_principal_convergence}
	
	\begin{lemma}
		\label{lemma1}
		Assume the premises of Theorem \ref{theorem_principal_convergence} are in force. Then, for any $\epsilon > 0$ it holds
		\begin{align*}
			\mathbb{P}&\left(\max\limits_{i \in \mathcal{J}} \text{\textbar}L(\mathcal{M}_{i}) - \hat{L}(\mathcal{M}_{i})\text{\textbar} \geq \epsilon/2\right) \leq \\
			&\leq m \sum_{\mathcal{M} \in \text{ Max } \mathbb{C}(\mathcal{H})} \left[B_{N_{t},\epsilon/8}(d_{VC}(\mathcal{M})) + \hat{B}_{N_{v},\epsilon/4}(d_{VC}(\mathcal{M}))\right]\\
			&\leq m \ \mathfrak{m}(\mathbb{C}(\mathcal{H})) \left[B_{N_{t},\epsilon/8}(d_{VC}(\mathbb{C}(\mathcal{H}))) + \hat{B}_{N_{v},\epsilon/4}(d_{VC}(\mathbb{C}(\mathcal{H})))\right].
		\end{align*}
	\end{lemma}
	\begin{proof}
		Fix $\epsilon > 0$. Denoting $\hat{h}_{i}^{(j)} \coloneqq \hat{h}_{\mathcal{M}_{i}}^{(j)}$,
		\begin{align}
			\label{dp1} \nonumber
			\mathbb{P}&\left(\max\limits_{i \in \mathcal{J}} \text{\textbar}L(\mathcal{M}_{i}) - \hat{L}(\mathcal{M}_{i})\text{\textbar} \geq \epsilon/2\right) \leq \mathbb{P}\left(\max\limits_{i \in \mathcal{J}} \sum_{j=1}^{m} \frac{1}{m} \text{\textbar}L(\mathcal{M}_{i}) - \hat{L}^{(j)}(\hat{h}^{(j)}_{i})\text{\textbar} > \epsilon/2\right)\\ \nonumber
			&\leq \mathbb{P}\left(\max_{j} \max\limits_{i \in \mathcal{J}} \text{\textbar}L(\mathcal{M}_{i}) - \hat{L}^{(j)}(\hat{h}^{(j)}_{i})\text{\textbar} > \epsilon/2\right)\\ \nonumber
			&\leq \mathbb{P}\left(\bigcup_{j=1}^{m} \left\{\max\limits_{i \in \mathcal{J}} \text{\textbar}L(\mathcal{M}_{i}) - \hat{L}^{(j)}(\hat{h}^{(j)}_{i})\text{\textbar} > \epsilon/2\right\}\right)\\ \nonumber
			&\leq \sum_{j=1}^{m} \mathbb{P}\left(\max\limits_{i \in \mathcal{J}} \text{\textbar}L(\mathcal{M}_{i}) - \hat{L}^{(j)}(\hat{h}^{(j)}_{i})\text{\textbar} > \epsilon/2\right)\\ \nonumber
			&= \sum_{j=1}^{m} \mathbb{P}\left(\max\limits_{i \in \mathcal{J}} \text{\textbar}L(\mathcal{M}_{i}) - L(\hat{h}^{(j)}_{i}) + L(\hat{h}^{(j)}_{i}) - \hat{L}^{(j)}(\hat{h}^{(j)}_{i})\text{\textbar} > \epsilon/2\right)\\ \nonumber
			&\leq \sum_{j=1}^{m} \mathbb{P}\left(\max\limits_{i \in \mathcal{J}} L(\hat{h}^{(j)}_{i}) - L(\mathcal{M}_{i}) + \max\limits_{i \in \mathcal{J}} \text{\textbar}L(\hat{h}^{(j)}_{i}) - \hat{L}^{(j)}(\hat{h}^{(j)}_{i})\text{\textbar} > \epsilon/2\right)\\ \nonumber
			&\leq \sum_{j=1}^{m} \mathbb{P}\left(\max\limits_{i \in \mathcal{J}} L(\hat{h}^{(j)}_{i}) - L(\mathcal{M}_{i}) > \epsilon/4\right) + \mathbb{P}\left(\max\limits_{i \in \mathcal{J}} \text{\textbar}L(\hat{h}^{(j)}_{i}) - \hat{L}^{(j)}(\hat{h}^{(j)}_{i})\text{\textbar} > \epsilon/4\right)\\
			&\leq \sum_{j=1}^{m} \mathbb{P}\left(\max\limits_{i \in \mathcal{J}} L(\hat{h}^{(j)}_{i}) - L(\mathcal{M}_{i}) > \epsilon/4\right) + \mathbb{P}\left(\max\limits_{i \in \mathcal{J}} \sup_{h \in \mathcal{M}_{i}} \text{\textbar}\hat{L}^{(j)}(h) - L(h)\text{\textbar} > \epsilon/4\right)
		\end{align}
		in which in the first inequality we applied the definition of $\hat{L}(\mathcal{M})$. For each $j$, the first probability in \eqref{dp1} is equal to
		\begin{align*}
			\mathbb{P}&\left(\max\limits_{i \in \mathcal{J}} L(\hat{h}^{(j)}_{i}) - L_{\mathcal{D}_{N}^{(j)}}(\hat{h}^{(j)}_{i}) + L_{\mathcal{D}_{N}^{(j)}}(\hat{h}^{(j)}_{i}) - L(\mathcal{M}_{i}) > \epsilon/4\right)\\
			&\leq \mathbb{P}\left(\max\limits_{i \in \mathcal{J}} L(\hat{h}^{(j)}_{i}) - L_{\mathcal{D}_{N}^{(j)}}(\hat{h}^{(j)}_{i}) + L_{\mathcal{D}_{N}^{(j)}}(h_{i}) - L(\mathcal{M}_{i}) > \epsilon/4\right)\\
			&\leq \mathbb{P}\left(\left\{\max\limits_{i \in \mathcal{J}} \text{\textbar}L(\hat{h}^{(j)}_{i}) - L_{\mathcal{D}_{N}^{(j)}}(\hat{h}^{(j)}_{i})\text{\textbar} > \epsilon/8\right\} \bigcup \left\{\max\limits_{i \in \mathcal{J}} \text{\textbar} L_{\mathcal{D}_{N}^{(j)}}(h_{i}) - L(\mathcal{M}_{i}) \text{\textbar} > \epsilon/8\right\}\right)\\
			&\leq \mathbb{P}\left(\max\limits_{i \in \mathcal{J}} \sup_{h \in \mathcal{M}_{i}} \text{\textbar}L_{\mathcal{D}_{N}^{(j)}}(h) - L(h)\text{\textbar} > \epsilon/8\right),
		\end{align*}
		in which the first inequality follows from the fact that $L_{\mathcal{D}_{N}^{(j)}}(\hat{h}^{(j)}_{i}) \leq L_{\mathcal{D}_{N}^{(j)}}(h_{i})$, and the last follows since $L(\mathcal{M}_{i}) = L(h_{i})$. We conclude that
		\begin{align*}
			&\mathbb{P}\left(\max\limits_{i \in \mathcal{J}} \text{\textbar}L(\mathcal{M}_{i}) - \hat{L}(\mathcal{M}_{i})\text{\textbar} \geq \epsilon/2\right) \\
			&\leq \sum_{j=1}^{m} \mathbb{P}\left(\max\limits_{i \in \mathcal{J}} \sup_{h \in \mathcal{M}_{i}} \text{\textbar}L_{\mathcal{D}_{N}^{(j)}}(h) - L(h)\text{\textbar} > \epsilon/8\right) + \mathbb{P}\left(\max\limits_{i \in \mathcal{J}} \sup_{h \in \mathcal{M}_{i}} \text{\textbar}\hat{L}^{(j)}(h) - L(h)\text{\textbar} > \epsilon/4\right).
		\end{align*}	
		If $\mathcal{M}_{1} \subset \mathcal{M}_{2}$ then, for any $\epsilon > 0$ and $j = 1, \dots, m$, we have the following inclusion of events
		\begin{align*}
			&\left\{\sup_{h \in \mathcal{M}_{1}} \text{\textbar}\hat{L}^{(j)}(h) - L(h)\text{\textbar} > \epsilon\right\} \subset \left\{\sup_{h \in \mathcal{M}_{2}} \text{\textbar}\hat{L}^{(j)}(h) - L(h)\text{\textbar} > \epsilon\right\}\\
			&\left\{\sup_{h \in \mathcal{M}_{1}} \text{\textbar}L_{\mathcal{D}_{N}^{(j)}}(h) - L(h)\text{\textbar} > \epsilon\right\} \subset \left\{\sup_{h \in \mathcal{M}_{2}} \text{\textbar}L_{\mathcal{D}_{N}^{(j)}}(h) - L(h)\text{\textbar} > \epsilon\right\},
		\end{align*}
		hence it is true that
		\begin{align*}
			&\left\{\max\limits_{i \in \mathcal{J}} \sup_{h \in \mathcal{M}_{i}} \text{\textbar}\hat{L}^{(j)}(h) - L(h)\text{\textbar} > \epsilon/4\right\} \subset \left\{\max_{\mathcal{M} \in \text{ Max } \mathbb{C}(\mathcal{H})} \sup_{h \in \mathcal{M}} \text{\textbar}\hat{L}^{(j)}(h) - L(h)\text{\textbar} > \epsilon/4\right\}\\
			&\left\{\max\limits_{i \in \mathcal{J}} \sup_{h \in \mathcal{M}_{i}} \text{\textbar}L_{\mathcal{D}_{N}^{(j)}}(h) - L(h)\text{\textbar} > \epsilon/8\right\} \subset \left\{\max_{\mathcal{M} \in \text{ Max } \mathbb{C}(\mathcal{H})} \sup_{h \in \mathcal{M}} \text{\textbar}L_{\mathcal{D}_{N}^{(j)}}(h) - L(h)\text{\textbar} > \epsilon/8\right\},
		\end{align*}
		which yields
		\begin{align}
			\label{conlusion_cond} \nonumber
			\mathbb{P}&\left(\max\limits_{i \in \mathcal{J}} \text{\textbar}L(\mathcal{M}_{i}) - \hat{L}(\mathcal{M}_{i})\text{\textbar} \geq \epsilon/2\right) \\  \nonumber
			&\leq \sum_{j=1}^{m} \sum_{\mathcal{M} \in \text{ Max } \mathbb{C}(\mathcal{H})} \mathbb{P}\left(\sup_{h \in \mathcal{M}} \text{\textbar}L_{\mathcal{D}_{N}^{(j)}}(h) - L(h)\text{\textbar} > \epsilon/8\right) + \mathbb{P}\left(\sup_{h \in \mathcal{M}} \text{\textbar}\hat{L}^{(j)}(h) - L(h)\text{\textbar} > \epsilon/4\right)\\ \nonumber
			&\leq m \sum_{\mathcal{M} \in \text{ Max } \mathbb{C}(\mathcal{H})} \left[B_{N_{t},\epsilon/8}(d_{VC}(\mathcal{M})) + \hat{B}_{N_{v},\epsilon/4}(d_{VC}(\mathcal{M}))\right]\\
			&\leq m \ \mathfrak{m}(\mathbb{C}(\mathcal{H})) \left[B_{N_{t},\epsilon/8}(d_{VC}(\mathbb{C}(\mathcal{H}))) + \hat{B}_{N_{v},\epsilon/4}(d_{VC}(\mathbb{C}(\mathcal{H})))\right],
		\end{align}
		in which the last inequality follows from the fact that both $\hat{B}_{N_{v},\epsilon/4}$ and $B_{N_{t},\epsilon/8}$ are increasing functions, and $d_{VC}(\mathbb{C}(\mathcal{H})) = \max_{\mathcal{M} \in \mathbb{C}(\mathcal{H})} d_{VC}(\mathcal{M})$.
	\end{proof}
	
	\begin{proof}[\textbf{Proof of Theorem \ref{theorem_principal_convergence}}]
		It follows from Lemma \ref{lemma1} that
		\begin{align*}
			\mathbb{P}&\left(\max\limits_{i \in \mathcal{J}} \text{\textbar}L(\mathcal{M}_{i}) - \hat{L}(\mathcal{M}_{i})\text{\textbar} \geq \epsilon^{\star}/2\right) \\  \nonumber
			&\leq m \sum_{\mathcal{M} \in \text{ Max } \mathbb{C}(\mathcal{H})} \left[B_{N_{t},\epsilon^{\star}/8}(d_{VC}(\mathcal{M})) + \hat{B}_{N_{v},\epsilon^{\star}/4}(d_{VC}(\mathcal{M}))\right]\\
			&\leq m \ \mathfrak{m}(\mathbb{C}(\mathcal{H})) \left[B_{N_{t},\epsilon^{\star}/8}(d_{VC}(\mathbb{C}(\mathcal{H}))) + \hat{B}_{N_{v},\epsilon^{\star}/4}(d_{VC}(\mathbb{C}(\mathcal{H})))\right],
		\end{align*}
		so the result follows from Proposition \ref{proposition_principal} since
		\begin{equation*}
			\{L(\hat{\mathcal{M}}) \neq L(\mathcal{M}^{\star})\} \subset \left\{\max\limits_{i \in \mathcal{J}} \text{\textbar}L(\mathcal{M}_{i}) - \hat{L}(\mathcal{M}_{i})\text{\textbar} \geq \epsilon^{\star}/2\right\}.
		\end{equation*}	
		
		If the almost sure convergences \eqref{as_conv} hold, then
		\begin{equation}
			\label{as_proof}
			\hat{L}(\mathcal{M}) \xrightarrow[N \to \infty]{\text{a.s.}} L(\mathcal{M})
		\end{equation}
		for all $\mathcal{M} \in \mathbb{C}(\mathcal{H})$, since, if $L(h) = \hat{L}^{(j)}(h) = L_{\mathcal{D}_{N}}^{(j)}(h)$ for all $j = 1,\dots,m$ and $h \in \mathcal{H}$, then $\hat{L}(\mathcal{M}) = L(\mathcal{M})$ for all $\mathcal{M} \in \mathbb{C}(\mathcal{H})$. Observe that
		\begin{align}
			\label{incl}
			\left\{\max_{\mathcal{M} \in \mathbb{C}(\mathcal{H})} \text{\textbar}L(\mathcal{M}) - \hat{L}(\mathcal{M})\text{\textbar} = 0\right\} \subset \left\{\hat{\mathcal{M}} = \mathcal{M}^{\star}\right\},
		\end{align}
		since, if the estimated risk $\hat{L}$ is equal to the out-of-sample risk $L$, then the definitions of $\hat{\mathcal{M}}$ and $\mathcal{M}^{\star}$ coincide. As the probability of the event on the left hand-side of \eqref{incl} converges to one if \eqref{as_proof} is true, we conclude that, if \eqref{as_conv} holds, then $\hat{\mathcal{M}}$ converges to $\mathcal{M}^{\star}$ with probability one.
	\end{proof}
	
	\begin{proof}[\textbf{Proof of Theorem \ref{CVModelconvergence}}]
		We need to show that \eqref{as_conv} holds in these instances. For any $\epsilon > 0$, by Corollary \ref{cor3TypeI},	
		\begin{align*}
			&\mathbb{P}\left(\max_{\mathcal{M} \in \mathbb{C}(\mathcal{H})} \max_{j} \sup\limits_{h \in \mathcal{M}} \text{\textbar}L_{\mathcal{D}_{N}^{(j)}}(h) - L(h)\text{\textbar} > \epsilon\right) \leq \sum_{j=1}^{m} \mathbb{P}\left(\sup\limits_{h \in \mathcal{H}} \text{\textbar}L_{\mathcal{D}_{N}^{(j)}}(h) - L(h)\text{\textbar} > \epsilon\right)\\
			&\leq m \ 8 \exp\left\{d_{VC}(\mathcal{H}) \left(1 + \ln \frac{N_{t}}{d_{VC}(\mathcal{H})} - N_{t}\frac{\epsilon^{2}}{32C^{2}}\right)\right\}.
		\end{align*}
		By the inequality above, and Borel-Cantelli Lemma \cite[Theorem~4.3]{billingsley2008}, the first convergence in \eqref{as_conv} holds. The second convergence holds since the inequality above is also true, but with $L_{\mathcal{D}_{N}^{(j)}}$ and $N_{t}$ interchanged by $\hat{L}^{(j)}$ and $N_{v}$, the empirical risk and size of the $j$-th validation sample.
	\end{proof}
	
	\begin{proof}[\textbf{Proof of Theorem \ref{bound_constant}}]
		We first note that
		\begin{align}
			\label{Sum1} \nonumber
			\mathbb{P}&\left(\sup\limits_{h \in \mathcal{\hat{M}}} \text{\textbar}L_{\tilde{\mathcal{D}}_{M}}(h) - L(h) \text{\textbar} > \epsilon \right)  = \mathbb{E} \Bigg(\mathbb{P}\left(\sup\limits_{h \in \mathcal{\hat{M}}} \text{\textbar}L_{\tilde{\mathcal{D}}_{M}}(h) - L(h) \text{\textbar} > \epsilon \text{\textbar}\mathcal{\hat{M}}\right)\Bigg)\\ \nonumber
			& = \sum_{i \in \mathcal{J}} \mathbb{P}\left(\sup\limits_{h \in \mathcal{\hat{M}}} \text{\textbar}L_{\tilde{\mathcal{D}}_{M}}(h) - L(h) \text{\textbar} > \epsilon \text{\textbar}\mathcal{\hat{M}} = \mathcal{M}_{i}\right) \mathbb{P}(\mathcal{\hat{M}} = \mathcal{M}_{i})\\
			& = \sum_{i \in \mathcal{J}} \mathbb{P}\left(\sup\limits_{h \in \mathcal{M}_{i}} \text{\textbar}L_{\tilde{\mathcal{D}}_{M}}(h) - L(h) \text{\textbar} > \epsilon \text{\textbar}\mathcal{\hat{M}} = \mathcal{M}_{i}\right) \mathbb{P}(\mathcal{\hat{M}} = \mathcal{M}_{i}).
		\end{align}
		Fix $\mathcal{M} \in \mathbb{C}(\mathcal{H})$ with $\mathbb{P}(\mathcal{\hat{M}} = \mathcal{M}) > 0$. We claim that
		\begin{align}
			\label{cond_independence}
			\mathbb{P}\left(\sup\limits_{h \in \mathcal{M}} \text{\textbar}L_{\tilde{\mathcal{D}}_{M}}(h) - L(h) \text{\textbar} > \epsilon \text{\textbar}\mathcal{\hat{M}} = \mathcal{M}\right) = \mathbb{P}\left(\sup\limits_{h \in \mathcal{M}} \text{\textbar}L_{\tilde{\mathcal{D}}_{M}}(h) - L(h) \text{\textbar} > \epsilon\right).
		\end{align}
		Indeed, since $\tilde{\mathcal{D}}_{M}$ is independent of $\mathcal{D}_{N}$, the event 
		\begin{equation*}
			\left\{\sup_{h \in \mathcal{M}} \text{\textbar}L_{\tilde{\mathcal{D}}_{M}}(h) - L(h) \text{\textbar} > \epsilon\right\}
		\end{equation*}
		is independent of $\{\mathcal{\hat{M}} = \mathcal{M}\}$, as the former depends solely on $\tilde{\mathcal{D}}_{M}$, and the latter solely on $\mathcal{D}_{N}$. Hence, by applying bound $\eqref{bound_theoremBC}$ to each positive probability in the sum \eqref{Sum1}, we obtain that
		\begin{align*}
			\mathbb{P}\left(\sup\limits_{h \in \mathcal{\hat{M}}} \text{\textbar}L_{\tilde{\mathcal{D}}_{M}}(h) - L(h) \text{\textbar} > \epsilon \right) & \leq \sum_{i \in \mathcal{J}} B_{M,\epsilon}^{I}(d_{VC}(\mathcal{M}_{i}))  \mathbb{P}(\mathcal{\hat{M}} = \mathcal{M}_{i})\\
			& = \mathbb{E} \left(B_{M,\epsilon}^{I}(d_{VC}(\mathcal{\hat{M}}))\right) \leq B_{N,\epsilon}^{I}(d_{VC}(\mathbb{C}(\mathcal{H}))),
		\end{align*}
		as desired, in which the last inequality follows from the fact that $B_{M,\epsilon}^{I}$ is an increasing function and $d_{VC}(\mathbb{C}(\mathcal{H})) = \max_{\mathcal{M} \in \mathbb{C}(\mathcal{H})} d_{VC}(\mathcal{M})$.
		
		The bound for type II estimation error may be obtained similarly, since
		\begin{align*}
			\mathbb{P}&\left(L(\hat{h}_{\mathcal{\hat{M}}}^{\tilde{\mathcal{D}}_{M}}) - L(h^{\star}_{\mathcal{\hat{M}}}) > \epsilon \right) = \mathbb{E} \Bigg(\mathbb{P}\left(L(\hat{h}_{\mathcal{\hat{M}}}^{\tilde{\mathcal{D}}_{M}}) - L(h^{\star}_{\mathcal{\hat{M}}}) > \epsilon \text{\textbar} \mathcal{\hat{M}} \right)\Bigg)\\
			& = \sum_{i \in \mathcal{J}} \mathbb{P}\left(L(\hat{h}_{\mathcal{\hat{M}}}^{\tilde{\mathcal{D}}_{M}}) - L(h^{\star}_{\mathcal{\hat{M}}}) > \epsilon \text{\textbar} \mathcal{\hat{M}} = \mathcal{M}_{i} \right) \mathbb{P}(\mathcal{\hat{M}} = \mathcal{M}_{i})\\
			& = \sum_{i \in \mathcal{J}} \mathbb{P}\left(L(\hat{h}_{\mathcal{M}_{i}}^{\tilde{\mathcal{D}}_{M}}) - L(h^{\star}_{\mathcal{M}_{i}}) > \epsilon \text{\textbar} \mathcal{\hat{M}} = \mathcal{M}_{i} \right) \mathbb{P}(\mathcal{\hat{M}} = \mathcal{M}_{i})\\
			& = \sum_{i \in \mathcal{J}} \mathbb{P}\left(L(\hat{h}_{\mathcal{M}_{i}}^{\tilde{\mathcal{D}}_{M}}) - L(h^{\star}_{\mathcal{M}_{i}}) > \epsilon \right) \mathbb{P}(\mathcal{\hat{M}} = \mathcal{M}_{i}),
		\end{align*}
		and $B^{II}_{M,\epsilon}(d_{VC}(\mathcal{M}_{i}))$ is a bound for the probabilities inside the sum by \eqref{bound_theoremBC}. The assertion that types I and II estimation errors are asymptotically zero when $d_{VC}(\mathbb{C}(\mathcal{H})) < \infty$ is immediate from the established bounds.
	\end{proof}
	
	\begin{proof}[\textbf{Proof of Theorem \ref{theorem_tipeIII}}]
		We first show that
		\begin{align}
			\label{lemma_inside}
			\mathbb{P}\left(L(h_{\hat{\mathcal{M}}}^{\star}) - L(h^{\star}) > \epsilon\right) \leq \mathbb{P}\left(\max\limits_{i \in \mathcal{J}} \text{\textbar}\hat{L}(\mathcal{M}_{i}) - L(\mathcal{M}_{i})\text{\textbar} > (\epsilon \vee \epsilon^{\star})/2 \right).
		\end{align}
		If $\epsilon \leq \epsilon^{\star}$ then, by Lemma \ref{lemma0}, we have that
		\begin{align}
			\label{incl1}
			\left\{\max\limits_{i \in \mathcal{J}} \text{\textbar}\hat{L}(\mathcal{M}_{i}) - L(\mathcal{M}_{i})\text{\textbar} < (\epsilon \vee \epsilon^{\star})/2\right\} \subset \left\{L(\hat{\mathcal{M}}) = L(\mathcal{M}^{\star})\right\} \subset \left\{L(h_{\hat{\mathcal{M}}}^{\star}) - L(h^{\star}) < \epsilon\right\},
		\end{align}
		since $L(h^{\star}_{\hat{\mathcal{M}}}) = L(\hat{\mathcal{M}})$ and $L(h^{\star}_{\mathcal{M}^{\star}}) = L(\mathcal{M}^{\star})$, so \eqref{lemma_inside} follows in this case.
		
		Now, if $\epsilon > \epsilon^{\star}$ and $\max\limits_{i \in \mathcal{J}} \text{\textbar}\hat{L}(\mathcal{M}_{i}) - L(\mathcal{M}_{i})\text{\textbar} < \epsilon/2$, then
		\begin{align*}
			L(\hat{\mathcal{M}}) - L(\mathcal{M}^{\star}) &= [L(\hat{\mathcal{M}}) - \hat{L}(\mathcal{M}^{\star})] - [L(\mathcal{M}^{\star}) - \hat{L}(\mathcal{M}^{\star})]\\
			&\leq [L(\hat{\mathcal{M}}) - \hat{L}(\hat{\mathcal{M}})] - [L(\mathcal{M}^{\star}) - \hat{L}(\mathcal{M}^{\star})]\\
			&\leq \epsilon/2 + \epsilon/2 = \epsilon,
		\end{align*}
		in which the first inequality follows from the fact that the minimum of $\hat{L}$ is attained at $\hat{\mathcal{M}}$, and the last inequality follows from $\max\limits_{i \in \mathcal{J}} \text{\textbar}\hat{L}(\mathcal{M}_{i}) - L(\mathcal{M}_{i})\text{\textbar} < \epsilon/2$. Since $L(\hat{\mathcal{M}}) - L(\mathcal{M}^{\star}) = L(h_{\hat{\mathcal{M}}}^{\star}) - L(h^{\star})$, we also have the inclusion of events
		\begin{align}
			\label{incl2}
			\left\{\max\limits_{i \in \mathcal{J}} \text{\textbar}\hat{L}(\mathcal{M}_{i}) - L(\mathcal{M}_{i})\text{\textbar} < (\epsilon \vee \epsilon^{\star})/2\right\} \subset \left\{L(h_{\hat{\mathcal{M}}}^{\star}) - L(h^{\star}) < \epsilon\right\},
		\end{align}
		when $\epsilon > \epsilon^{\star}$. From \eqref{incl1} and \eqref{incl2} follows \eqref{lemma_inside}, as desired.
		
		It follows from Lemma \ref{lemma1} that
		\begin{align}
			\label{conclusion2} \nonumber
			\mathbb{P}&\left(\max\limits_{i \in \mathcal{J}} \text{\textbar}L(\mathcal{M}_{i}) - \hat{L}(\mathcal{M}_{i})\text{\textbar} \geq (\epsilon \vee \epsilon^{\star})/2\right) \leq\\ \nonumber
			&\leq m \sum_{\mathcal{M} \in \text{Max } \mathbb{C}(\mathcal{H})} \left[B_{N_{t},(\epsilon \vee \epsilon^{\star})/8}(d_{VC}(\mathcal{M})) + \hat{B}_{N_{v},(\epsilon \vee \epsilon^{\star})/4}(d_{VC}(\mathcal{M}))\right]\\
			& m \ \mathfrak{m}(\mathbb{C}(\mathcal{H})) \left[B_{N_{t},(\epsilon \vee \epsilon^{\star})/8}(d_{VC}(\mathbb{C}(\mathcal{H}))) + \hat{B}_{N_{v},(\epsilon \vee \epsilon^{\star})/4}(d_{VC}(\mathbb{C}(\mathcal{H})))\right].
		\end{align}
		The result follows combining \eqref{lemma_inside} and \eqref{conclusion2}.	
	\end{proof}
	
	\subsection{Results of Section \ref{SecUnbounded}}
	
	\begin{proof}[\textbf{Proof of Theorem \ref{theorem_principal_convergence_unbounded}}]
		We claim that
		\begin{equation}
			\label{implication1}
			1 - \delta < \frac{\hat{L}(\mathcal{M}_{i})}{L(\mathcal{M}_{i})} < 1 + \delta, \ \forall i \in \mathcal{J} \implies \max\limits_{i \in \mathcal{J}} \ \text{\textbar}\hat{L}(\mathcal{M}_{i}) - L(\mathcal{M}_{i})\text{\textbar} < \frac{\epsilon^{\star}}{2}.
		\end{equation}
		Indeed, the left-hand side of \eqref{implication1} implies
		\begin{equation*}
			\begin{cases}
				L(\mathcal{M}_{i}) - \hat{L}(\mathcal{M}_{i}) < \frac{\epsilon^{\star} L(\mathcal{M}_{i})}{2 \max\limits_{i \in \mathcal{J}} L(\mathcal{M}_{i})} < \frac{\epsilon^{\star}}{2}\\
				\hat{L}(\mathcal{M}_{i}) - L(\mathcal{M}_{i}) < \frac{\epsilon^{\star} L(\mathcal{M}_{i})}{2 \max\limits_{i \in \mathcal{J}} L(\mathcal{M}_{i})} < \frac{\epsilon^{\star}}{2}\\
			\end{cases} \ \forall i \in \mathcal{J},
		\end{equation*}
		as desired. In particular, it follows from Lemma \ref{lemma0} that
		\begin{equation}
			\label{writeP}
			\mathbb{P}\left(L(\hat{\mathcal{M}}) \neq L(\mathcal{M}^{\star})\right) \leq \mathbb{P}\left(\min\limits_{i \in \mathcal{J}} \frac{\hat{L}(\mathcal{M}_{i})}{L(\mathcal{M}_{i})} \leq 1 - \delta\right) + \mathbb{P}\left(\max\limits_{i \in \mathcal{J}} \frac{\hat{L}(\mathcal{M}_{i})}{L(\mathcal{M}_{i})} \geq 1 + \delta\right)
		\end{equation}
		hence it is enough to bound both probabilities on the right-hand side of the expression above.
		
		The first probability in \eqref{writeP} may be written as
		\begin{align}	
			\label{writeP2}
			\mathbb{P}\left(\max\limits_{i \in \mathcal{J}} \ \frac{L(\mathcal{M}_{i}) - \hat{L}(\mathcal{M}_{i})}{L(\mathcal{M}_{i})} \geq \delta\right) \leq \sum_{j=1}^{m} \mathbb{P}\left(\max\limits_{i \in \mathcal{J}} \ \frac{L(\mathcal{M}_{i}) - \hat{L}^{(j)}(\hat{h}^{(j)}_{i})}{L(\mathcal{M}_{i})} \geq \delta\right),
		\end{align}
		in which the inequality follows from a union bound. Since $x \mapsto  \frac{x - \alpha}{x}$ is increasing, and $L(\mathcal{M}_{i}) \leq L(\hat{h}^{(j)}_{i})$ for every $j = 1,\dots,m$, each probability in \eqref{writeP2} is bounded by
		\begin{equation}
			\label{res1}
			\mathbb{P}\left(\max\limits_{i \in \mathcal{J}} \ \frac{L(\hat{h}^{(j)}_{i}) - \hat{L}^{(j)}(\hat{h}^{(j)}_{i})}{L(\hat{h}^{(j)}_{i})} \geq \delta\right) \leq \mathbb{P}\left(\max\limits_{i \in \mathcal{J}} \sup\limits_{h \in \mathcal{M}_{i}} \ \frac{\text{\textbar}L(h) - \hat{L}^{(j)}(h)\text{\textbar}}{L(h)} \geq \delta\right).
		\end{equation}
		
		We turn to the second probability in \eqref{writeP} which can be written as
		\begin{equation}
			\label{writeP3}
			\mathbb{P}\left(\max\limits_{i \in \mathcal{J}} \ \frac{\hat{L}(\mathcal{M}_{i}) - L(\mathcal{M}_{i})}{L(\mathcal{M}_{i})} \geq \delta\right) \leq \sum_{j=1}^{m} \mathbb{P}\left(\max\limits_{i \in \mathcal{J}} \ \frac{\hat{L}^{(j)}(\hat{h}^{(j)}_{i}) - L(\mathcal{M}_{i})}{L(\mathcal{M}_{i})} \geq \delta\right),
		\end{equation}
		in which again the inequality follows from a union bound. In order to bound each probability in \eqref{writeP3} we intersect its event with
		\begin{equation*}
			\max\limits_{i \in \mathcal{J}} \frac{L(\hat{h}_{i}^{(j)})}{L(\mathcal{M}_{i})} \leq \frac{1}{1 - \delta} \iff \max\limits_{i \in \mathcal{J}} \frac{L(\hat{h}_{i}^{(j)}) - L(\mathcal{M}_{i})}{L(\hat{h}_{i}^{(j)})} \leq \delta,
		\end{equation*}
		and its complement, to obtain
		\begin{align}
			\label{res2} \nonumber
			&\mathbb{P}\left(\max\limits_{i \in \mathcal{J}} \ \frac{\hat{L}^{(j)}(\hat{h}^{(j)}_{i}) - L(\mathcal{M}_{i})}{L(\mathcal{M}_{i})} \geq \delta\right) \leq \mathbb{P}\left(\max\limits_{i \in \mathcal{J}} \frac{L(\hat{h}_{i}^{(j)}) - L(\mathcal{M}_{i})}{L(\hat{h}_{i}^{(j)})} \geq \delta\right)\\ \nonumber
			&+ \mathbb{P}\left(\max\limits_{i \in \mathcal{J}} \ \left(\frac{L(\hat{h}_{i}^{(j)})}{L(\mathcal{M}_{i})}\right) \frac{\hat{L}^{(j)}(\hat{h}^{(j)}_{i}) - L(\mathcal{M}_{i})}{L(\hat{h}_{i}^{(j)})} \geq \delta,\max\limits_{i \in \mathcal{J}} \frac{L(\hat{h}_{i}^{(j)})}{L(\mathcal{M}_{i})} \leq \frac{1}{1 - \delta}\right)\\ \nonumber
			&\leq \mathbb{P}\left(\max\limits_{i \in \mathcal{J}} \frac{L(\hat{h}_{i}^{(j)}) - L(\mathcal{M}_{i})}{L(\hat{h}_{i}^{(j)})} \geq \delta\right) + \mathbb{P}\left(\max\limits_{i \in \mathcal{J}} \ \frac{\hat{L}^{(j)}(\hat{h}^{(j)}_{i}) - L(\mathcal{M}_{i})}{L(\hat{h}_{i}^{(j)})} \geq \delta(1-\delta)\right)\\
			&\leq \mathbb{P}\left(\max\limits_{i \in \mathcal{J}} \sup\limits_{h \in \mathcal{M}_{i}} \frac{\text{\textbar}L_{\mathcal{D}_{N}}^{(j)}(h) - L(h)\text{\textbar}}{L(h)} \geq \frac{\delta}{2}\right) + \mathbb{P}\left(\max\limits_{i \in \mathcal{J}} \ \frac{\hat{L}^{(j)}(\hat{h}^{(j)}_{i}) - L(\mathcal{M}_{i})}{L(\hat{h}_{i}^{(j)})} \geq \delta(1-\delta)\right)
		\end{align}
		in which the last inequality follows from Lemma \ref{lemmaTypeItoII}.
		
		It remains to bound the second probability in \eqref{res2}. We have that it is equal to
		\begin{align}
			\label{res3} \nonumber
			&\mathbb{P}\left(\max\limits_{i \in \mathcal{J}} \ \frac{\hat{L}^{(j)}(\hat{h}^{(j)}_{i}) - L(\hat{h}^{(j)}_{i}) + L(\hat{h}^{(j)}_{i}) - L(\mathcal{M}_{i})}{L(\hat{h}_{i}^{(j)})} \geq \delta(1-\delta)\right)\\ \nonumber
			&\leq \mathbb{P}\left(\max\limits_{i \in \mathcal{J}} \ \frac{\hat{L}^{(j)}(\hat{h}^{(j)}_{i}) - L(\hat{h}^{(j)}_{i})}{L(\hat{h}_{i}^{(j)})} \geq \frac{\delta(1-\delta)}{2}\right) + \mathbb{P}\left(\max\limits_{i \in \mathcal{J}} \ \frac{L(\hat{h}^{(j)}_{i}) - L(\mathcal{M}_{i})}{L(\hat{h}_{i}^{(j)})} \geq \frac{\delta(1-\delta)}{2}\right)\\
			&\leq \mathbb{P}\left(\max\limits_{i \in \mathcal{J}} \sup\limits_{h \in \mathcal{M}_{i}} \ \frac{\text{\textbar}\hat{L}^{(j)}(h) - L(h)\text{\textbar}}{L(h)} \geq \frac{\delta(1-\delta)}{2}\right) + \mathbb{P}\left(\max\limits_{i \in \mathcal{J}} \sup\limits_{h \in \mathcal{M}_{i}} \ \frac{\text{\textbar}L(h) - L_{\mathcal{D}_{N}}(h)\text{\textbar}}{L(h)} \geq \frac{\delta(1-\delta)}{4}\right),
		\end{align}
		in which the last inequality follows again from Lemma \ref{lemmaTypeItoII}. From (\ref{writeP}-\ref{res3}), it follows that
		\begin{align*}
			\mathbb{P}\left(L(\hat{\mathcal{M}}) \neq L(\mathcal{M}^{\star})\right)\leq &2 \sum_{j=1}^{m} \Bigg[\mathbb{P}\left(\max\limits_{i \in \mathcal{J}} \sup\limits_{h \in \mathcal{M}_{i}} \ \frac{\text{\textbar}\hat{L}^{(j)}(h) - L(h)\text{\textbar}}{L(h)} \geq \frac{\delta(1-\delta)}{2}\right) +\\
			&\mathbb{P}\left(\max\limits_{i \in \mathcal{J}} \sup\limits_{h \in \mathcal{M}_{i}} \ \frac{\text{\textbar}L(h) - L_{\mathcal{D}_{N}}(h)\text{\textbar}}{L(h)} \geq \frac{\delta(1-\delta)}{4}\right)\Bigg]\\
			&\leq 2 \sum_{j=1}^{m} \sum_{\mathcal{M} \in \text{ Max } \mathbb{C}(\mathcal{H})} \Bigg[\mathbb{P}\left(\sup\limits_{h \in \mathcal{M}} \ \frac{\text{\textbar}\hat{L}^{(j)}(h) - L(h)\text{\textbar}}{L(h)} \geq \frac{\delta(1-\delta)}{2}\right) +\\
			&\mathbb{P}\left(\sup\limits_{h \in \mathcal{M}} \ \frac{\text{\textbar}L(h) - L_{\mathcal{D}_{N}}(h)\text{\textbar}}{L(h)} \geq \frac{\delta(1-\delta)}{4}\right)\Bigg],
		\end{align*}
		in which the inequality holds by the same arguments as in \eqref{conlusion_cond}, hence
		\begin{equation*}
			\mathbb{P}\left(L(\hat{\mathcal{M}}) \neq L(\mathcal{M}^{\star})\right) \leq 2m \ \mathfrak{m}(\mathbb{C}(\mathcal{H})) \left[\hat{B}_{N_{v},\frac{\delta(1-\delta)}{2}}(d_{VC}(\mathbb{C}(\mathcal{H}))) + B_{N_{t},\frac{\delta(1-\delta)}{4}}(d_{VC}(\mathbb{C}(\mathcal{H})))\right].
		\end{equation*}
		
		If the almost sure convergences \eqref{as_conv2} hold, then $L(h) = L_{\mathcal{D}_{N}}^{(j)}(h) = \hat{L}^{(j)}(h)$ for all $j$ and $h \in \mathcal{H}$, and the definitions of $\mathcal{M}^{\star}$ and $\hat{\mathcal{M}}$ coincide.
	\end{proof}
	
	\begin{proof}[\textbf{Proof of Theorem \ref{theorem_tipeIII2}}]
		We show that
		\begin{align}
			\label{implication2}
			\mathbb{P}\left(L(h_{\hat{\mathcal{M}}}^{\star}) - L(h^{\star}) > \epsilon\right) \geq \mathbb{P}\left(\frac{L(h_{\hat{\mathcal{M}}}^{\star}) - L(h^{\star})}{L(h_{\hat{\mathcal{M}}}^{\star})} > \frac{\epsilon}{L(\mathcal{M}^{\star})}\right),
		\end{align}
		so from \eqref{lemma_inside} and \eqref{implication1} will follow that
		\begin{equation*}
			\mathbb{P}\left(\frac{L(h_{\hat{\mathcal{M}}}^{\star}) - L(h^{\star})}{L(h_{\hat{\mathcal{M}}}^{\star})} > \frac{\epsilon}{L(\mathcal{M}^{\star})}\right) \leq \mathbb{P}\left(\min\limits_{i \in \mathcal{J}} \frac{\hat{L}(\mathcal{M}_{i})}{L(\mathcal{M}_{i})} \leq 1 - \delta^\prime\right) + \mathbb{P}\left(\max\limits_{i \in \mathcal{J}} \frac{\hat{L}(\mathcal{M}_{i})}{L(\mathcal{M}_{i})} \geq 1 + \delta^\prime\right),
		\end{equation*}
		and the result is then direct from the proof of Theorem \ref{theorem_principal_convergence_unbounded}. But \eqref{implication2} is clearly true since
		\begin{align*}
			\frac{L(h_{\hat{\mathcal{M}}}^{\star}) - L(h^{\star})}{L(h_{\hat{\mathcal{M}}}^{\star})} > \frac{\epsilon}{L(\mathcal{M}^{\star})} \implies L(h_{\hat{\mathcal{M}}}^{\star}) - L(h^{\star}) > \epsilon \frac{L(h_{\hat{\mathcal{M}}}^{\star})}{L(\mathcal{M}^{\star})} \geq \epsilon.
		\end{align*}		
	\end{proof}
	
	\subsection{Results of Section \ref{SecReuse}}
	\label{SecProofReuse}
	
	\begin{proof}[Proof of Theorem \ref{bound_constant_reusing}]
		The bound for type I estimation error follows from the inequality
		\begin{align*}
			&\mathbb{P}\left(\sup\limits_{h \in \hat{\mathcal{M}}} \left|L_{\mathcal{D}_{N}}(h) - L(h)\right| > \epsilon\right) \\
			& = \mathbb{P}\left(\sup\limits_{h \in \hat{\mathcal{M}}} \left|L_{\mathcal{D}_{N}}(h) - L(h)\right| > \epsilon,\hat{\mathcal{M}} = \mathcal{M}^{\star}\right) + \mathbb{P}\left(\sup\limits_{h \in \hat{\mathcal{M}}} \left|L_{\mathcal{D}_{N}}(h) - L(h)\right| > \epsilon,\hat{\mathcal{M}} \neq \mathcal{M}^{\star}\right)\\
			&\leq \mathbb{P}\left(\sup\limits_{h \in \mathcal{M}^{\star}} \left|L_{\mathcal{D}_{N}}(h) - L(h)\right| > \epsilon\right) + \mathbb{P}\left(\hat{\mathcal{M}} \neq \mathcal{M}^{\star}\right),
		\end{align*}
		by noting that $B_{N,\epsilon}^{I}(d_{VC}(\mathcal{M}^{\star}))$ is a bound for the first probability. With a similar argument, we have the bound for type II estimation error.
	\end{proof}
	
	\FloatBarrier
	
	\appendix
		\section{Vapnik-Chervonenkis theory}
		\label{apVCtheory}
		
		In this appendix, we present the main ideas and results of classical Vapnik-Chervonenkis (VC) theory, the stone upon which the results in this paper are built. The presentation of the theory is a simplified merge of \cite{vapnik1998}, \cite{vapnik2000}, \cite{devroye1996} and \cite{cortes2019}, where the simplicity of the arguments is preferred over the refinement of the bounds. Hence, we present results which support those in this paper and outline the main ideas of VC theory, even though are not the tightest available bounds. We omit the proofs, and note that refined versions of the results presented here may be found at one or more of the references.
		
		This appendix is a review of VC theory, except for novel results presented in Section \ref{ApUnbounded} for the case of unbounded loss functions, where we obtain new bounds for relative type I estimation error by extending the results in \cite{cortes2019}. We start defining the shatter coefficient and VC dimension of a hypotheses space under loss function $\ell$.
		
		\begin{definition}[Shatter coefficient]
			\label{shatter} 
			Let $\mathcal{G} = \{I: \mathcal{Z} \mapsto  \{0,1\}\}$ be a set of binary functions with domain $\mathcal{Z}$. The $N$-shatter coefficient of $\mathcal{G}$ is defined as
			\begin{equation*}
				S(\mathcal{G},N) = \max\limits_{(z_{1},\dots,z_{N}) \in \mathcal{Z}^{N}} \text{\textbar}\big\{\big(I(z_{1}),\dots,I(z_{N})\big): I \in \mathcal{G}\big\}\text{\textbar},
			\end{equation*}
			for $N \in \mathbb{Z}_{+}$, in which $\text{\textbar}\cdot\text{\textbar}$ is the cardinality of a set.
		\end{definition}
		
		\begin{definition}[Vapnik-Chervonenkis dimension]
			\label{VCdimension} 
			Fixed a hypotheses space $\mathcal{H}$ and a loss function $\ell$, set
			\begin{align*}
				C = \sup\limits_{\substack{z \in \mathcal{Z} \\ h \in \mathcal{H}}} \ell(z,h),
			\end{align*}
			in which $C$ can be infinity. Consider, for each $h \in \mathcal{H}$ and $\beta \in (0,C)$, the binary function $I(z;h,\beta) = \mathds{1}\{\ell(z,h) \geq \beta\}$, for $z \in \mathcal{Z}$, and denote
			\begin{align*}
				\mathcal{G}_{\mathcal{H},\ell} = \Big\{I(\cdot;h,\beta): h \in \mathcal{H}, \beta \in (0,C)\Big\}.
			\end{align*}
			We define the shatter coefficient of $\mathcal{H}$ under loss function $\ell$ as
			\begin{equation*}
				S(\mathcal{H},\ell,N) \coloneqq S(\mathcal{G}_{\mathcal{H},\ell},N).
			\end{equation*}
			The Vapnik-Chervonenkis (VC) dimension of $\mathcal{H}$ under loss function $\ell$ is the greatest integer $k \geq 1$ such that $S(\mathcal{H},\ell,k) = 2^{k}$, and is denoted by $d_{VC}(\mathcal{H},\ell)$. If $S(\mathcal{H},\ell,k) = 2^{k}$, for all integer $k \geq 1$, we denote $d_{VC}(\mathcal{H},\ell) = \infty$.
		\end{definition} 
		
		\begin{remark}
			If there is no confusion about which loss function we are referring to, or when it is not of importance to our argument, we omit $\ell$ and denote the shatter coefficient and VC dimension simply by $S(\mathcal{H},N)$ and $d_{VC}(\mathcal{H})$. We note that if the hypotheses in $\mathcal{H}$ are binary valued functions and $\ell$ is the simple loss function $\ell((x,y),h) = \mathds{1}\{h(x) \neq y\}$, then $\mathcal{H} = \mathcal{G}_{\mathcal{H},\ell}$, and its $N$-th shatter coefficient is actually the maximum number of dichotomies that can be generated by the functions in $\mathcal{H}$ with $N$ points.
		\end{remark}
		
		\subsection{Generalized Glivenko-Cantelli Problems}
		
		The main results of VC theory are based on a generalization of the Glivenko-Cantelli Theorem, which can be stated as follows. Recall that $\mathcal{D}_{N} = \{Z_{1},\dots,Z_{N}\}$ is a sequence of independent random vectors with a same distribution $P(z) \coloneqq \mathbb{P}(Z \leq z)$, for $z \in \mathcal{Z} \subset \mathbb{R}^{d}$, defined in a probability space $(\Omega,\mathcal{S},\mathbb{P})$. 
		
		In order to ease notation, we assume, without loss of generality, that $\Omega = \mathbb{R}^{d}$, $\mathcal{S}$ is the Borel $\sigma$-algebra of $\mathbb{R}^{d}$, the random vector $Z$ is the identity $Z(\omega) = \omega$, for $\omega \in \Omega$, and $\mathbb{P}$ is the unique probability measure such that $\mathbb{P}(\{\omega:\omega \leq z\}) = P(z)$, for all $z \in \mathbb{R}^{d}$. Define
		\begin{align*}
			P_{\mathcal{D}_{N}}(z) \coloneqq \frac{1}{N} \sum_{i=1}^{N} \mathds{1}\{Z_{i} \leq z\}, & & z \in \mathcal{Z}
		\end{align*}
		as the empirical distribution of $Z$ under sample $\mathcal{D}_{N}$.
		
		The assertion of the theorem below is that of \cite[Theorem~12.4]{devroye1996}. Its bottom line is that the empirical distribution of random variables converges uniformly to $P$ with probability one.
		
		\begin{theorem}[Glivenko-Cantelli Theorem]
			\label{glivenko_cantelli}
			Assume $d = 1$ and $\mathcal{Z} = \mathbb{R}$. Then, for a fixed $\epsilon > 0$ and $N$ great enough,
			\begin{equation}
				\label{GCbound}
				\mathbb{P}\left(\sup\limits_{z \in \mathbb{R}} \text{\textbar}P(z) - P_{\mathcal{D}_{N}}(z)\text{\textbar} > \epsilon\right) \leq 8(N+1) \exp\left\{-N \frac{\epsilon^2}{32}\right\}.
			\end{equation}
			Applying Borel-Cantelli Lemma \cite[Theorem~4.3]{billingsley2008} to \eqref{GCbound} yields
			\begin{equation*}
				\lim\limits_{N \to \infty} \sup\limits_{z \in \mathbb{R}} \text{\textbar}P(z) - P_{\mathcal{D}_{N}}(z)\text{\textbar} = 0 \text{ with probability one.}
			\end{equation*}
			In other words, $P_{\mathcal{D}_{N}}$ converges uniformly almost surely to $P$.
		\end{theorem}
		
		Theorem \ref{glivenko_cantelli} has the flavor of VC theory results: a rate of uniform convergence of the empirical probability of a class of events to their real probability, which implies the almost sure convergence. Indeed, letting $\mathcal{S}^{\star} \subset \mathcal{S}$ be a class of events, that is not necessarily a $\sigma$-algebra, and denoting
		\begin{equation*}
			\mathbb{P}_{\mathcal{D}_{N}}(A) = \frac{1}{N} \sum_{i=1}^{N} \mathds{1}\{Z_{i} \in A\},
		\end{equation*}
		as the empirical probability of event $A \in \mathcal{S}$ under sample $\mathcal{D}_{N}$, the probability in \eqref{GCbound} can be rewritten as
		\begin{equation}
			\label{partial_convergence}
			\mathbb{P}\left(\sup\limits_{A \in \mathcal{S}^{\star}} \text{\textbar}\mathbb{P}(A) - \mathbb{P}_{\mathcal{D}_{N}}(A)\text{\textbar} > \epsilon\right),
		\end{equation}
		in which $\mathcal{S}^{\star} = \{A_{z}:z \in \mathbb{R}\}$ with $A_{z} = \{\omega \in \Omega: \omega \leq z\}$. If probability \eqref{partial_convergence} converges to zero when $N$ tends to infinity for a class $\mathcal{S}^{\star} \subsetneq \mathcal{S}$, we say there exists a \textit{partial uniform convergence} of the empirical measure to $\mathbb{P}$.
		
		Observe that in \eqref{partial_convergence} not only the class $\mathcal{S}^{\star}$ is fixed, but also the probability measure $\mathbb{P}$, hence partial uniform convergence is dependent on the class and the probability. Nevertheless, in a distribution-free framework, such as that of learning (cf. Section \ref{SecPreliminaries}), the convergences of interest should hold for any data generating distribution, which is the case, for example, of Glivenko-Cantelli Theorem, that presents a rate of convergence \eqref{GCbound} which does not depend on $P$, holding for any probability measure and random variable $Z$. Therefore, once a class $\mathcal{S}^{\star}$ of interest is fixed, partial uniform convergence should hold for any data generating distribution, a problem which can be stated as follows.
		
		Let $\mathcal{P}$ be the class of all possible probability distributions of a random variable with support in $\mathcal{Z}$, and let $\mathcal{S}^{\star}$ be a class of events. The \textit{generalized Glivenko-Cantelli problem} (GGCP) is to find a positive constant $a$ and a function $b: \mathbb{Z}_{+} \mapsto \mathbb{R}_{+}$, such that $\lim\limits_{N \to \infty} b(N)/\exp cN = 0, \forall c > 0$, satisfying, for $N$ great enough,\footnote{In the presentation of \cite[Chapter~2]{vapnik1998} it is assumed that $b$ is a positive constant, not depending on sample size $N$. Nevertheless, having $b$ as a function of $N$ of an order lesser than exponential does not change the qualitative behavior of this convergence, that is, also guarantees the almost sure converge due to Borel-Cantelli Lemma.}
		\begin{equation}
			\label{Gen_GCbound}
			\sup\limits_{P \in \mathcal{P}} \mathbb{P}\left(\sup\limits_{A \in \mathcal{S}^{\star}} \text{\textbar}\mathbb{P}(A) - \mathbb{P}_{\mathcal{D}_{N}}(A)\text{\textbar} > \epsilon\right) \leq b(N) \exp\{-a\epsilon^{2}N\},
		\end{equation}
		in which $\mathbb{P}$ is to be understood as dependent on $P$, since it is the unique probability measure on the Borel $\sigma$-algebra that equals $P$ on the events $\{\omega \in \Omega:\omega \leq z\}, z \in \mathbb{R}^{d}$. If the events are of the form $A = \{w \in \Omega: Z(w) \leq z\}, z \in \mathbb{R}$, then \eqref{Gen_GCbound} is equivalent to \eqref{GCbound}, although in the latter it is implicit that it holds for any distribution $P$.
		
		The investigation of GGCP revolves around deducing necessary and sufficient conditions on the class $\mathcal{S}^{\star}$ for \eqref{Gen_GCbound} to hold. We will study these conditions in order to establish the almost sure convergence to zero of type I estimation error (cf. \eqref{GE1}) when the loss function is binary, what may be stated as a GGCP.
		
		\subsection{Convergence to zero of type I estimation error}
		\label{ApTypeI}
		
		\subsubsection{Binary loss functions}
		
		Fix a hypotheses space $\mathcal{H}$, a binary loss function $\ell$, and consider the class $\mathcal{S}^{\star} = \{A_{h}: h \in \mathcal{H}\}$, such that $\mathds{1}\{z \in A_{h}\} = \ell(z,h) \in \{0,1\}, z \in \mathcal{Z}, h \in \mathcal{H}$, that is, if $z \in A_{h}$ the loss is one, and otherwise it is zero. For example, if $Z = (X,Y)$, the hypotheses in $\mathcal{H}$ are functions from the range of $X$ to that of $Y$, and $\ell$ is the simple loss function, then $A_{h}$ may be explicitly written as
		\begin{equation*}
			A_{h} = \{\omega: h(X(\omega)) \neq Y(\omega)\}.
		\end{equation*}
		In this instance, the probability in the left-hand side of \eqref{Gen_GCbound} may be written as
		\begin{equation}
			\label{typeIexp}
			\mathbb{P}\left(\sup\limits_{h \in \mathcal{H}} \text{\textbar}\mathbb{E}(\ell(Z,h)) - \mathbb{E}_{\mathcal{D}_{N}}(\ell(Z,h))\text{\textbar} > \epsilon\right),
		\end{equation}
		in which $\mathbb{E}$ is expectation with respect to $\mathbb{P}$ and $\mathbb{E}_{\mathcal{D}_{N}}$ is the empirical mean under $\mathcal{D}_{N}$. With the notation of Section \ref{SecPreliminaries}, this last probability equals
		\begin{equation*}
			\label{typeIappendix}
			\mathbb{P}\left(\sup\limits_{h \in \mathcal{H}} \text{\textbar}L(h) - L_{\mathcal{D}_{N}}(h)\text{\textbar} > \epsilon\right),
		\end{equation*}
		the tail probability of type I estimation error in $\mathcal{H}$.
		
		For each fixed $h \in \mathcal{H}$, we are comparing in \eqref{typeIexp} the mean of a binary function with its empirical mean, so we may apply Hoeffding's inequality \cite{hoeffding1963} to obtain 
		\begin{equation*}
			\mathbb{P}\left(\text{\textbar}\mathbb{E}(\ell(Z,h)) - \mathbb{E}_{\mathcal{D}_{N}}(\ell(Z,N))\text{\textbar} > \epsilon\right) \leq 2 \exp\{-2\epsilon^{2}N\},
		\end{equation*}
		from which follows a solution of type I estimation error GGCP when the cardinality of $\mathcal{H}$ is finite, by applying an elementary union bound:
		\begin{align*}
			\mathbb{P}\Bigg(\sup\limits_{h \in \mathcal{H}} \text{\textbar}\mathbb{E}(\ell(Z,h)) -\mathbb{E}_{\mathcal{D}_{N}}(\ell(Z,N))\text{\textbar} > \epsilon\Bigg) &\leq \sum_{h \in \mathcal{H}} \mathbb{P}\left(\text{\textbar}\mathbb{E}(\ell(Z,h)) - \mathbb{E}_{\mathcal{D}_{N}}(\ell(Z,N))\text{\textbar} > \epsilon\right) \\
			&\leq 2 \text{\textbar}\mathcal{H}\text{\textbar} \exp\{-2\epsilon^{2}N\},
		\end{align*}
		what establishes the almost sure convergence to zero of type I estimation error when $\mathcal{H}$ is finite and $\ell$ is binary.
		
		In order to treat the case when $\mathcal{H}$ has infinitely many hypotheses, we rely on a modification of Glivenko-Cantelli Theorem, which depends on the shatter coefficient of a class $\mathcal{S}^{\star} \subset \mathcal{S}$ of events in the Borel $\sigma$-algebra of $\mathbb{R}^{d}$, defined below.
		
		\begin{definition}
			\label{shatter_borel}
			Fix $\mathcal{S}^{\star} \subset \mathcal{S}$ and let
			\begin{equation*}
				\mathcal{G}_{\mathcal{S}^\star} = \{h_{A}(z) = \mathds{1}\{z \in A\}: A \in \mathcal{S}^\star\}
			\end{equation*}
			be the characteristic functions of the sets in $\mathcal{S}^{\star}$. We define the shatter coefficient of $\mathcal{S}^{\star}$ as
			\begin{equation*}
				S(\mathcal{S}^{\star},N) \coloneqq S(\mathcal{G}_{\mathcal{S}^{\star}},N),
			\end{equation*}
			in which $S(\mathcal{G}_{\mathcal{S}^{\star}},N)$ is the shatter coefficient of $\mathcal{G}_{\mathcal{S}^{\star}}$ (cf. Definition \ref{shatter}). From this definition, it follows that
			\begin{equation*}
				d_{VC}(\mathcal{S}^{\star}) = d_{VC}(\mathcal{G}_{\mathcal{S}^{\star}}).
			\end{equation*}
		\end{definition}
		
		The shatter coefficient and VC dimension of a class $\mathcal{S}^{\star}$ are related to the dichotomies this class can build with $N$ points by considering whether a point is in each set or not. From a simple modification of the proof of Theorem \ref{glivenko_cantelli} presented in \cite[Theorem~12.4]{devroye1996} follows a result due to \cite{vapnik1971uniform}.
		
		\begin{theorem}
			\label{theorem_GGCP}
			For any probability measure $\mathbb{P}$ and class of sets $\mathcal{S}^{\star} \subset \mathcal{S}$, for fixed $N \in \mathbb{Z}$ and $\epsilon > 0$, it is true that
			\begin{equation*}
				\mathbb{P}\left(\sup\limits_{A \in \mathcal{S}^{\star}} \text{\textbar}\mathbb{P}(A) - \mathbb{P}_{\mathcal{D}_{N}}(A)\text{\textbar} > \epsilon\right)  \leq 8 S(\mathcal{S}^{\star},N) \exp\left\{-N \frac{\epsilon^2}{32}\right\}.
			\end{equation*}
		\end{theorem}
		
		From this theorem follows a bound for tail probabilities of type I estimation error when $\ell$ is binary.
		
		\begin{corollary}
			\label{cor1TypeI}
			Fix a hypotheses space $\mathcal{H}$ and a loss function $\ell: \mathcal{Z} \times \mathcal{H} \mapsto \{0,1\}$. Let $\mathcal{S}^{\star} = \{A_{h}: h \in \mathcal{H}\}$, with
			\begin{equation*}
				\mathds{1}\{z \in A_{h}\} = \ell(z,h), z \in \mathcal{Z}, h \in \mathcal{H}.
			\end{equation*}
			Then,
			\begin{equation}
				\label{Ap1Eq4}
				\mathbb{P}\left(\sup\limits_{h \in \mathcal{H}} \text{\textbar}L(h) - L_{\mathcal{D}_{N}}(h)\text{\textbar} > \epsilon\right) \leq 8 \ S(\mathcal{H},N) \exp\left\{-N \frac{\epsilon^2}{32}\right\},
			\end{equation}
			with
			\begin{equation*}
				S(\mathcal{H},N) \coloneqq S(\mathcal{S}^{\star},N).
			\end{equation*}
		\end{corollary}
		
		\begin{remark}
			We remark that $S(\mathcal{S}^{\star},N) = S(\mathcal{G}_{\mathcal{H},\ell},N)$, as defined in Definition \ref{VCdimension} when the loss $\ell$ is binary. Observe that $\mathcal{S}^{\star}$ depends on $\ell$, although we omit the dependence to ease notation.
		\end{remark}
		
		The calculation of the quantities on the right-hand side of \eqref{Ap1Eq4} is not straightforward, since the shatter coefficient is not easily determined for arbitrary $N$. Nevertheless, the shatter coefficient may be bounded by a quantity depending on the VC dimension of $\mathcal{H}$. This is the content of \cite[Theorem~4.3]{vapnik1998}.
		
		\begin{theorem}
			\label{theorem_shaterDVC}
			If $d_{VC}(\mathcal{H}) < \infty$, then
			\begin{equation*}
				\ln S(\mathcal{H},N) \begin{cases}
					= N \ln 2, & \text{ if } N \leq d_{VC}(\mathcal{H})\\
					\leq d_{VC}(\mathcal{H}) \left(1 + \ln \frac{N}{d_{VC}(\mathcal{H})}\right), & \text{ if } N > d_{VC}(\mathcal{H}) 
				\end{cases}.
			\end{equation*}
		\end{theorem}
		
		\begin{remark}
			Theorem \ref{theorem_shaterDVC} is true for any loss function $\ell$, not only binary.
		\end{remark}
		
		Combining this theorem with Corollary \ref{cor1TypeI}, we obtain the following result.
		
		\begin{corollary}
			\label{cor2TypeI}
			Under the hypotheses of Corollary \ref{cor1TypeI} it holds
			\begin{equation}
				\label{Ap1Eq5}
				\mathbb{P}\left(\sup\limits_{h \in \mathcal{H}} \text{\textbar}L(h) - L_{\mathcal{D}_{N}}(h)\text{\textbar} > \epsilon\right) \leq 8 \ \exp\left\{d_{VC}(\mathcal{H}) \left(1 + \ln \frac{N}{d_{VC}(\mathcal{H})}\right) - N \frac{\epsilon^2}{32}\right\}.
			\end{equation}
			In particular, if $d_{VC}(\mathcal{H}) < \infty$, not only \eqref{Ap1Eq5} converges to zero, but also
			\begin{equation*}
				\sup\limits_{h \in \mathcal{H}} \text{\textbar}L(h) - L_{\mathcal{D}_{N}}(h)\text{\textbar} \xrightarrow[N \to \infty]{} 0,
			\end{equation*}
			with probability one by Borel-Cantelli Lemma.
		\end{corollary}
		
		From Corollary \ref{cor2TypeI} follows the convergence to zero of type I estimation error when the loss function is binary and $d_{VC}(\mathcal{H})$ is finite. We now extend this result to real-valued bounded loss functions.
		
		\subsubsection{Bounded loss functions}
		
		Assume the loss function is bounded, that is, for all $z \in \mathcal{Z}$ and $h \in \mathcal{H}$,
		\begin{equation}	
			\label{bounded}
			0 \leq \ell(z,h) \leq C < \infty,
		\end{equation}
		for a positive constant $C \in \mathbb{R}_{+}$. Throughout this section, a constant $C$ satisfying \eqref{bounded} is fixed.
		
		For any $h \in \mathcal{H}$, by definition of Lebesgue-Stieltjes integral, we have that
		\begin{equation*}
			L(h) = \int_{\mathcal{Z}} \ell(z,h) \ dP(z) = \lim\limits_{n \to \infty} \sum_{k=1}^{n-1} \frac{C}{n} \mathbb{P}\left(\ell(Z,h) > \frac{kC}{n}\right),
		\end{equation*}
		recalling that $Z$ is a random variable with distribution $P$. In the same manner, we may also write the empirical risk under $\mathcal{D}_{N}$ as
		\begin{equation*}
			L_{\mathcal{D}_{N}}(h) = \frac{1}{N} \sum_{i=1}^{N} \ell(Z_{i},h) = \lim\limits_{n \to \infty} \sum_{k=1}^{n-1} \frac{C}{n} \mathbb{P}_{\mathcal{D}_{N}}\left(\ell(Z,h) > \frac{kC}{n}\right),
		\end{equation*}
		recalling that $\mathbb{P}_{\mathcal{D}_{N}}$ is the empirical measure according to $\mathcal{D}_{N}$.
		
		From the representation of $L$ and $L_{\mathcal{D}_{N}}$ described above, we have that, for each $h \in \mathcal{H}$ fixed,
		\begin{align*}
			\text{\textbar}L(h) - L_{\mathcal{D}_{N}}(h)\text{\textbar} &= \text{\textbar}\lim\limits_{n \to \infty} \sum_{k=1}^{n-1} \frac{C}{n} \left(\mathbb{P}\left(\ell(Z,h) > \frac{kC}{n}\right) - \mathbb{P}_{\mathcal{D}_{N}}\left(\ell(Z,h) > \frac{kC}{n}\right)\right) \text{\textbar}\\
			&\leq \text{\textbar}\lim\limits_{n \to \infty} \sum_{k=1}^{n-1} \frac{C}{n} \sup\limits_{0 \leq \beta \leq C} \left(\mathbb{P}\left(\ell(Z,h) > \beta\right) - \mathbb{P}_{\mathcal{D}_{N}}\left(\ell(Z,h) > \beta\right)\right) \text{\textbar}\\
			&\leq \lim\limits_{n \to \infty} \sum_{k=1}^{n-1} \frac{C}{n} \sup\limits_{0 \leq \beta \leq C} \text{\textbar}\mathbb{P}\left(\ell(Z,h) > \beta\right) - \mathbb{P}_{\mathcal{D}_{N}}\left(\ell(Z,h) > \beta\right)\text{\textbar}\\
			&= C \sup\limits_{0 \leq \beta \leq C} \text{\textbar}\int_{\mathcal{Z}} \mathds{1}\{\ell(z,h) > \beta\} \ dP(z) - \frac{1}{N} \sum_{i=1}^{N} \ \mathds{1}\{\ell(Z_{i},h) > \beta\}\text{\textbar}.
		\end{align*}
		We conclude that
		\begin{align*}
			&\mathbb{P}\left(\sup\limits_{h \in \mathcal{H}} \text{\textbar}L(h) - L_{\mathcal{D}_{N}(h)}\text{\textbar} > \epsilon\right) \\
			&\leq \mathbb{P}\left(\sup\limits_{\substack{h \in \mathcal{H}\\ 0 \leq \beta \leq C}} \text{\textbar}\int_{\mathcal{Z}} \mathds{1}\{\ell(z,h) > \beta\} \ dP(z) - \frac{1}{N} \sum_{i=1}^{N} \ \mathds{1}\{\ell(Z_{i},h)\}\text{\textbar} > \frac{\epsilon}{C}\right).
		\end{align*}
		
		Since the right-hand side of the expression above is a GGCP with
		\begin{equation*}
			\mathcal{S}^{\star} = \{\{z \in \mathcal{Z}:\ell(z,h) > \beta\}: h \in \mathcal{H},0 \leq \beta \leq C\}
		\end{equation*}
		and, recalling the definition of shatter coefficient of $\mathcal{H}$ under a real-valued loss function $\ell$ (cf. Definition \ref{shatter}), we note that
		\begin{align*}
			S(\mathcal{G}_{\mathcal{H},\ell},N)  = S(\mathcal{S}^{\star},N) & & \text{ hence } & & d_{VC}(\mathcal{H}) = d_{VC}(\mathcal{S}^{\star}) 
		\end{align*}
		so a bound for the tail probabilities of type I estimation error when the loss function is bounded follows immediately from Theorems \ref{theorem_GGCP} and \ref{theorem_shaterDVC}.
		
		\begin{corollary}
			\label{cor3TypeI}
			Fix a hypotheses space $\mathcal{H}$ and a loss function $\ell: \mathcal{Z} \times \mathcal{H} \mapsto \mathbb{R}_{+}$, with $0 \leq \ell(z,h) \leq C$ for all $z \in \mathcal{Z}, h \in \mathcal{H}$. Then,
			\begin{equation}
				\label{Ap1Eq6}
				\mathbb{P}\left(\sup\limits_{h \in \mathcal{H}} \text{\textbar}L(h) - L_{\mathcal{D}_{N}}(h)\text{\textbar} > \epsilon\right) \leq 8 \ \exp\left\{d_{VC}(\mathcal{H}) \left(1 + \ln \frac{N}{d_{VC}(\mathcal{H})}\right) - N \frac{\epsilon^2}{32C^{2}}\right\}.
			\end{equation}
			In particular, if $d_{VC}(\mathcal{H}) < \infty$, not only \eqref{Ap1Eq6} converges to zero, but also
			\begin{equation*}
				\sup\limits_{h \in \mathcal{H}} \text{\textbar}L(h) - L_{\mathcal{D}_{N}}(h)\text{\textbar} \xrightarrow[N \to \infty]{} 0,
			\end{equation*}
			with probability one by Borel-Cantelli Lemma.
		\end{corollary}
		
		It remains to treat the case of unbounded loss functions, which requires a different approach.
		
		\subsubsection{Unbounded loss functions}
		\label{ApUnbounded}
		
		In this section, we establish conditions on $P$ and $\mathcal{H}$ for the convergence in probability to zero of type I relative estimation error when $\ell$ is unbounded. The framework treated here is that described at the beginning of Section \ref{SecUnbounded}.
		
		In order to ease notation, we denote
		\begin{equation*}
			\ell(\mathcal{D}_{N},h) \coloneqq \left(\ell(Z_{1},h),\dots,\ell(Z_{N},h)\right) \in \mathbb{R}^{N}\setminus\{0\},
		\end{equation*}
		the vector sample point loss, so it follows from \eqref{LNp} that, for $1 \leq q \leq p$,
		\begin{equation}
			\label{def_sample_pmoment}
			L_{\mathcal{D}_{N}}^{q}(h) \coloneqq \frac{\lVert \ell(\mathcal{D}_{N},h) \rVert_{q}}{N^{\frac{1}{q}}},
		\end{equation}
		in which $\lVert \cdot \rVert_{q}$ is the $q$-norm in $\mathbb{R}^{N}$. Recall that we assume that $P$ has at most heavy tails, which means there exists a $p > 1$, that can be less than 2, with
		\begin{align}
			\label{tauStarAp}
			\tau_{p} < \tau^{\star} < \infty,
		\end{align}
		and that
		\begin{align}
			\label{finite_moments}
			\sup\limits_{h \in \mathcal{H}} L_{\mathcal{D}_{N}}^{p}(h)  < \infty & & \text{ and } & & \sup\limits_{h \in \mathcal{H}} L^{p}(h) < \infty,
		\end{align}
		in which the first inequality should hold with probability one. 
		
		The first condition in \eqref{finite_moments} is more a feature of the loss function, than of distribution $P$. Actually, one can bound $L_{\mathcal{D}_{N}}^{p}(h)$ by a quantity depending on $N$ and $L_{\mathcal{D}_{N}}^{q}(h)$ with $1 \leq q < p$, for any sample $\mathcal{D}_{N}$ of any distribution $P$. This is the content of the next lemma, which will be useful later on, and that implies the following: if $\sup_{h \in \mathcal{H}} L_{\mathcal{D}_{N}}^{1}(h) = \sup_{h \in \mathcal{H}} L_{\mathcal{D}_{N}}(h) < \infty$, then $\sup_{h \in \mathcal{H}} L_{\mathcal{D}_{N}}^{p}(h) < \infty$ for any $1 < p < \infty$, for $N$ and $\mathcal{H}$ fixed. 
		
		\begin{lemma}
			\label{lemma_norm}
			For fixed $\mathcal{H}$, $N \geq 1$ and $1 \leq q < p$, it follows that
			\begin{equation*}
				1 \leq \frac{L_{\mathcal{D}_{N}}^{p}(h)}{L_{\mathcal{D}_{N}}^{q}(h)} \leq N^{\frac{1}{q} - \frac{1}{p}}
			\end{equation*}
			for all $h \in \mathcal{H}$.
		\end{lemma}
		\begin{proof}
			Recalling definition \eqref{def_sample_pmoment}, we have that
			\begin{equation*}
				\frac{L_{\mathcal{D}_{N}}^{p}(h)}{L_{\mathcal{D}_{N}}^{q}(h)} = N^{\frac{1}{q} - \frac{1}{p}} \frac{\lVert \ell(\mathcal{D}_{N},h) \rVert_{p}}{\lVert \ell(\mathcal{D}_{N},h) \rVert_{q}},
			\end{equation*}
			so it is enough to show that
			\begin{equation*}
				N^{\frac{1}{p} - \frac{1}{q}} \leq \frac{\lVert \ell(\mathcal{D}_{N},h) \rVert_{p}}{\lVert \ell(\mathcal{D}_{N},h) \rVert_{q}} \leq 1.
			\end{equation*}
			
			Now, the right inequality above is clear, since if $w \in \mathbb{R}^{N}$ is such that $\lVert w \rVert_{q} = 1$, then
			\begin{align*}
				\lVert w \rVert_{p}^{p} = \sum_{i=1}^{N} \text{\textbar}w_{i}\text{\textbar}^{p} \leq \sum_{i=1}^{N} \text{\textbar}w_{i}\text{\textbar}^{q} = 1,
			\end{align*}
			so the result follows when $\lVert w \rVert_{q} = 1$ by elevating both sided to the $1/p$ power. To conclude the proof it is enough to see that, for any $w \in \mathbb{R}^{N}\setminus\{0\}$,
			\begin{equation*}
				\lVert w \rVert_{p} = \lVert w \rVert_{q} \lVert \frac{w}{\lVert w \rVert_{q}} \rVert_{p} \leq \lVert w \rVert_{q} \lVert \frac{w}{\lVert w \rVert_{q}} \rVert_{q} = \lVert w \rVert_{q}.
			\end{equation*}
			
			The left inequality is a consequence of H\"older's inequality, since, for $w \in \mathbb{R}^{N}$,
			\begin{align*}
				\sum_{i=1}^{N} \text{\textbar}w_{i}\text{\textbar}^{q} \cdot 1 \leq \left(\sum_{i=1}^{N} \text{\textbar}w_{i}\text{\textbar}^{p}\right)^{\frac{q}{p}} N^{1 - \frac{q}{p}},
			\end{align*}
			and the result follows by taking the $1/q$ power on both sides.
		\end{proof}
		
		For unbounded losses, rather than considering the convergence of type I estimation error to zero, we will consider the convergence of the relative type I estimation error, defined as
		\begin{equation}	
			\label{Ap1Eq9}
			\textbf{(I)} \ \ \sup\limits_{h \in \mathcal{H}} \frac{\text{\textbar}L(h) - L_{\mathcal{D}_{N}}(h)\text{\textbar}}{L(h)} .
		\end{equation}
		In order to establish bounds for the tail probabilities of \eqref{Ap1Eq9} when \eqref{tauStarAp} and \eqref{finite_moments_text} hold, we rely on the following novel technical theorem.
		
		\begin{theorem}
			\label{change_denominator}
			Let $q = \sqrt{p}$. For any hypotheses space $\mathcal{H}$, loss function $\ell$ satisfying $\ell(h,z) \geq 1$, and $0 < \epsilon < 1$, it holds
			\begin{align*}
				&\mathbb{P}\left(\sup\limits_{h \in \mathcal{H}} \frac{\text{\textbar}L(h) - L_{\mathcal{D}_{N}}(h)\text{\textbar}}{L(h)}  > \tau^{\star} \epsilon\right) \leq \mathbb{P}\left(\sup\limits_{h \in \mathcal{H}} \frac{L(h) - L_{\mathcal{D}_{N}}(h)}{L^{p}(h)} > \epsilon\right)\\
				&+ \mathbb{P}\left(\sup\limits_{h \in \mathcal{H}} \frac{L_{\mathcal{D}_{N}}^{\prime}(h) - L^{\prime}(h)}{L_{\mathcal{D}_{N}}^{\prime q}(h)} > \frac{\epsilon}{N^{\frac{1}{q} - \frac{1}{p}}}\right) + \mathbb{P}\left(\sup\limits_{h \in \mathcal{H}} \frac{L_{\mathcal{D}_{N}}(h) - L(h)}{L_{\mathcal{D}_{N}}^{p}(h)} > \frac{\epsilon(1-\epsilon)}{N^{\frac{1}{q} - \frac{1}{p}}}\right)
			\end{align*}
			in which $L^{\prime},L_{\mathcal{D}_{N}}^{\prime}$ and $L_{\mathcal{D}_{N}}^{\prime k}$ are the respective risks and $k$ moments of loss function $\ell^{\prime}(z,h) \coloneqq (\ell(z,h))^{q}$.
		\end{theorem}
		\begin{proof}
			We first note that
			\begin{align*}
				&\sup\limits_{h \in \mathcal{H}} \frac{L_{\mathcal{D}_{N}}(h) - L(h)}{L(h)} > \tau^{\star}\epsilon \implies \sup\limits_{h \in \mathcal{H}} \left(\frac{L^{q}(h)}{L(h)} \frac{1}{\tau^{\star}}\right) \frac{L_{\mathcal{D}_{N}}(h) - L(h)}{L^{q}(h)} > \epsilon\\
				&\implies \sup\limits_{h \in \mathcal{H}} \frac{L_{\mathcal{D}_{N}}(h) - L(h)}{L^{q}(h)} > \epsilon
			\end{align*}
			since
			\begin{align*}
				\sup\limits_{h \in \mathcal{H}} \left(\frac{L^{q}(h)}{L(h)} \frac{1}{\tau^{\star}}\right) \leq \sup\limits_{h \in \mathcal{H}} \left(\frac{L^{p}(h)}{L(h)} \frac{1}{\tau^{\star}}\right) \leq 1
			\end{align*}
			by \eqref{tauStarAp}. With an analogous deduction, it follows that
			\begin{equation*}
				\sup\limits_{h \in \mathcal{H}} \frac{L(h) - L_{\mathcal{D}_{N}}(h)}{L(h)} > \tau^{\star}\epsilon \implies \sup\limits_{h \in \mathcal{H}} \frac{L(h) - L_{\mathcal{D}_{N}}(h)}{L^{p}(h)} > \epsilon.
			\end{equation*}
			Hence, the probability on the left hand-side of the statement is lesser or equal to
			\begin{align}
				\label{Ap1Eq10}
				\mathbb{P}\left(\sup\limits_{h \in \mathcal{H}} \frac{L(h) - L_{\mathcal{D}_{N}}(h)}{L^{p}(h)} > \epsilon\right) + \mathbb{P}\left(\sup\limits_{h \in \mathcal{H}} \frac{L_{\mathcal{D}_{N}}(h) - L(h)}{L^{q}(h)} > \epsilon\right),
			\end{align}
			so it is enough to properly bound the second probability in \eqref{Ap1Eq10}. 	
			
			In order to do so, we will intersect the event inside the probability with the following event, and its complement:
			\begin{align*}
				\sup\limits_{h \in \mathcal{H}} \frac{L_{\mathcal{D}_{N}}^{q}(h) - \delta L_{\mathcal{D}_{N}}^{p}(h)}{L^{q}(h)} \leq 1 \iff \sup\limits_{h \in \mathcal{H}} \frac{L_{\mathcal{D}_{N}}^{q}(h) - L^{q}(h)}{L_{\mathcal{D}_{N}}^{p}(h)} \leq \delta,
			\end{align*}
			in which
			\begin{equation*}
				\delta \coloneqq \frac{\epsilon}{N^{\frac{1}{q} - \frac{1}{p}}}.
			\end{equation*}
			Proceeding in this way, we conclude that
			\begin{align}
				\label{Ap1Eq12} \nonumber
				&\mathbb{P}\left(\sup\limits_{h \in \mathcal{H}} \frac{L_{\mathcal{D}_{N}}(h) - L(h)}{L^{q}(h)} > \epsilon\right) \leq \mathbb{P}\left(\sup\limits_{h \in \mathcal{H}} \frac{L_{\mathcal{D}_{N}}^{q}(h) - L^{q}(h)}{L_{\mathcal{D}_{N}}^{p}(h)} > \delta\right)\\ \nonumber
				&+ \mathbb{P}\left(\sup\limits_{h \in \mathcal{H}} \left(\frac{L_{\mathcal{D}_{N}}^{q}(h) - \delta L_{\mathcal{D}_{N}}^{p}(h)}{L^{q}(h)}\right) \frac{L_{\mathcal{D}_{N}}(h) - L(h)}{L_{\mathcal{D}_{N}}^{q}(h) - \delta L_{\mathcal{D}_{N}}^{p}(h)} > \epsilon,\sup\limits_{h \in \mathcal{H}} \frac{L_{\mathcal{D}_{N}}^{q}(h) - \delta L_{\mathcal{D}_{N}}^{p}(h)}{L^{q}(h)} \leq 1\right)\\
				&\leq \mathbb{P}\left(\sup\limits_{h \in \mathcal{H}} \frac{L_{\mathcal{D}_{N}}^{q}(h) - L^{q}(h)}{L_{\mathcal{D}_{N}}^{p}(h)} > \delta\right) + \mathbb{P}\left(\sup\limits_{h \in \mathcal{H}} \frac{L_{\mathcal{D}_{N}}(h) - L(h)}{L_{\mathcal{D}_{N}}^{q}(h) - \delta L_{\mathcal{D}_{N}}^{p}(h)} > \epsilon\right).
			\end{align}
			To bound the first probability above, we recall the definition of $L_{\mathcal{D}_{N}}^{p}(h)$ and note that $a^{\frac{1}{q}} - b^{\frac{1}{q}} \leq a - b$ if $q > 1$ and $1 \leq b \leq a$, so that
			\begin{equation}
				\label{Ap1Eq11}
				\mathbb{P}\left(\sup\limits_{h \in \mathcal{H}} \frac{L_{\mathcal{D}_{N}}^{q}(h) - L^{q}(h)}{L_{\mathcal{D}_{N}}^{p}(h)} > \delta\right) \leq \mathbb{P}\left(\sup\limits_{h \in \mathcal{H}} \frac{\left(L_{\mathcal{D}_{N}}^{q}(h)\right)^{q} - \left(L^{q}(h)\right)^{q}}{N^{-\frac{1}{p}} \lVert \ell(\mathcal{D}_{N},h)\rVert_{p}} > \delta\right).
			\end{equation}
			Define the loss function $\ell^{\prime}(z,h) \coloneqq (\ell(z,h))^{q}$, and let $L^{\prime},L^{\prime k},L_{\mathcal{D}_{N}}^{\prime}$ and $L_{\mathcal{D}_{N}}^{\prime k}$ be the risks and $k$ moments according to this new loss function. Then, the probability in \eqref{Ap1Eq11} can be written as
			\begin{equation}
				\label{Ap1Eq13}
				\mathbb{P}\left(\sup\limits_{h \in \mathcal{H}} \frac{L_{\mathcal{D}_{N}}^{\prime}(h) - L^{\prime}(h)}{\left(N^{-1}\lVert \ell^{\prime}(\mathcal{D}_{N},h) \rVert_{\frac{p}{q}}^{\frac{p}{q}}\right)^{\frac{1}{p}}} > \delta\right) \leq \mathbb{P}\left(\sup\limits_{h \in \mathcal{H}} \frac{L_{\mathcal{D}_{N}}^{\prime}(h) - L^{\prime}(h)}{L_{\mathcal{D}_{N}}^{\prime q}(h)} > \delta\right)
			\end{equation}
			since $\frac{p}{q} = q$ and $N^{\frac{1}{p}} \leq N^{\frac{1}{q}}$.
			
			We turn to the second probability in \eqref{Ap1Eq12}. By applying Lemma \ref{lemma_norm}, and recalling the definition of $\delta$, we have that $L_{\mathcal{D}_{N}}^{q}(h) - \delta L_{\mathcal{D}_{N}}^{p}(h)$ is equal to
			\begin{align*}
				L_{\mathcal{D}_{N}}^{p}(h) \left(\frac{L_{\mathcal{D}_{N}}^{q}(h)}{L_{\mathcal{D}_{N}}^{p}(h)} - \delta\right) \geq L_{\mathcal{D}_{N}}^{p}(h) \left(\frac{1}{N^{\frac{1}{q} - \frac{1}{p}}} - \frac{\epsilon}{N^{\frac{1}{q} - \frac{1}{p}}}\right) = \frac{L_{\mathcal{D}_{N}}^{p}(h)}{N^{\frac{1}{q} - \frac{1}{p}}} \left(1 - \epsilon\right),
			\end{align*}
			from which follows
			\begin{equation}
				\label{Ap1Eq14}
				\mathbb{P}\left(\sup\limits_{h \in \mathcal{H}} \frac{L_{\mathcal{D}_{N}}(h) - L(h)}{L_{\mathcal{D}_{N}}^{q}(h) - \delta L_{\mathcal{D}_{N}}^{p}(h)} > \epsilon\right) \leq \mathbb{P}\left(\sup\limits_{h \in \mathcal{H}} \frac{L_{\mathcal{D}_{N}}(h) - L(h)}{L_{\mathcal{D}_{N}}^{p}(h)} > \frac{\epsilon(1-\epsilon)}{N^{\frac{1}{q} - \frac{1}{p}}}\right).
			\end{equation}
			The result follows by combining \eqref{Ap1Eq10}, \eqref{Ap1Eq12}, \eqref{Ap1Eq13} and \eqref{Ap1Eq14}.
		\end{proof}
		
		An exponential bound for relative type I estimation error \eqref{Ap1Eq9} depending on $p$, $\tau^{\star}$ and $d_{VC}(\mathcal{H})$ is a consequence of Theorems \ref{theorem_shaterDVC} and \ref{change_denominator}, and results in \cite{cortes2019}, which we now state. Define, for a $0 < \varsigma < 1$ fixed,
		\begin{align*}
			\Gamma(p,\epsilon) &= \frac{p-1}{p}(1+\varsigma)^{\frac{1}{p}} + \frac{1}{p} \left(\frac{p}{p-1}\right)^{p-1} \left(1 + \left(\frac{p-1}{p}\right)^{p} \varsigma^{\frac{1}{p}}\right)^{\frac{1}{p}} \left[1 + \frac{\ln(1/\epsilon)}{\left(\frac{p}{p-1}\right)^{p-1}}\right]^{\frac{p-1}{p}},		
		\end{align*}
		for $0 < \epsilon < 1, 1 < p \leq 2$, and
		\begin{align*}
			\Lambda(p) = \left(\frac{1}{2}\right)^{\frac{2}{p}} \left(\frac{p}{p-2}\right)^{\frac{p-1}{p}} + \frac{p}{p-1} \varsigma^{\frac{p-2}{2p}}
		\end{align*}
		for $p > 2$.
		
		\begin{theorem}
			\label{theorem_unbounded}
			Fix a hypotheses space $\mathcal{H}$ and an unbounded loss function $\ell$, and assume that \eqref{finite_moments} is in force. Then, the following holds:
			\begin{itemize}
				\item If $P$ has light tails, so that \eqref{tauStarAp} holds for a $p > 2$ fixed, then		
				\begin{align*}
					\mathbb{P}\Bigg(\sup\limits_{h \in \mathcal{H}} & \frac{L(h) - L_{\mathcal{D}_{N}}(h)}{\sqrt[p]{\left(L^{p}(h)\right)^{p} + \varsigma}} > \Lambda(p) \epsilon\Bigg)\\& < 4 \exp\left\{d_{VC}(\mathcal{H})\left(1 + \ln \frac{2N}{d_{VC}(\mathcal{H})}\right) - \frac{\epsilon^{2}N}{4}\right\}
				\end{align*}
				and
				\begin{align*}
					\mathbb{P}\Bigg(\sup\limits_{h \in \mathcal{H}} & \frac{L_{\mathcal{D}_{N}}(h) - L(h)}{\sqrt[p]{\left(L_{\mathcal{D}_{N}}^{p}(h)\right)^{p} + \varsigma}} > \Lambda(p) \epsilon\Bigg) \\&< 4 \exp\left\{d_{VC}(\mathcal{H})\left(1 + \ln \frac{2N}{d_{VC}(\mathcal{H})}\right) - \frac{\epsilon^{2}N}{4}\right\},
				\end{align*}
				for $0 < \epsilon < 1$ and $0 < \varsigma < \epsilon^{2}$.
				
				\item If $P$ has heavy tails, so that \eqref{tauStarAp} holds only for a $1 < p \leq 2$ fixed, then
				\begin{align*}
					\mathbb{P}\Bigg(\sup\limits_{h \in \mathcal{H}} & \frac{L(h) - L_{\mathcal{D}_{N}}(h)}{\sqrt[p]{\left(L^{p}(h)\right)^{p} + \varsigma}} > \Gamma(p,\epsilon) \epsilon\Bigg) \\&< 4 \exp\left\{d_{VC}(\mathcal{H})\left(1 + \ln \frac{2N}{d_{VC}(\mathcal{H})}\right) - \frac{\epsilon^{2}N^{\frac{2(p-1)}{p}}}{2^{\frac{p+2}{2}}}\right\}
				\end{align*}
				and
				\begin{align*}
					\mathbb{P}\Bigg(\sup\limits_{h \in \mathcal{H}} & \frac{L_{\mathcal{D}_{N}}(h) - L(h)}{\sqrt[p]{\left(L_{\mathcal{D}_{N}}^{p}(h)\right)^{p} + \varsigma}} > \Gamma(p,\epsilon) \epsilon\Bigg) \\&< 4 \exp\left\{d_{VC}(\mathcal{H})\left(1 + \ln \frac{2N}{d_{VC}(\mathcal{H})}\right) - \frac{\epsilon^{2}N^{\frac{2(p-1)}{p}}}{2^{\frac{p+2}{2}}}\right\},
				\end{align*}
				for $0 < \epsilon < 1$ and $0 < \varsigma^{\frac{p-1}{p}} < \epsilon^{\frac{p}{p-1}}$.
			\end{itemize}
		\end{theorem}
		
		Theorem \ref{theorem_unbounded}, together with Theorem \ref{change_denominator}, imply the following corollary, which is an exponential bound for relative type I estimation error when $P$ has heavy or light tails. The value of $\varsigma$ in the definitions of $\Lambda(p)$ and $\Gamma(p,\epsilon)$ below can be arbitrarily small.
		
		\begin{corollary}
			\label{convergence_relativeTI}
			Fix a hypotheses space $\mathcal{H}$, an unbounded loss function $\ell$ and $\epsilon > 0$. The following holds:
			\begin{itemize}
				\item If \eqref{tauStarAp} holds for a $p \geq 4$ fixed, then
				\begin{align*}
					\mathbb{P}&\Bigg(\sup\limits_{h \in \mathcal{H}} \frac{\text{\textbar}L(h) - L_{\mathcal{D}_{N}}(h)\text{\textbar}}{L(h)} > \tau^{\star} \Lambda(\sqrt{p}) \epsilon\Bigg) \\
					&< 12 \exp\left\{d_{VC}(\mathcal{H})\left(1 + \ln \frac{2N}{d_{VC}(\mathcal{H})}\right) - \frac{\epsilon^{2}(1-\epsilon)^{2}N^{1 - \frac{2}{\sqrt{p}} + \frac{2}{p}}}{4}\right\}
				\end{align*}
				
				\item If \eqref{tauStarAp} holds for a $1 < p < 4$ fixed, then
				\begin{align*}
					\mathbb{P}&\Bigg(\sup\limits_{h \in \mathcal{H}} \frac{\text{\textbar}L(h) - L_{\mathcal{D}_{N}}(h)\text{\textbar}}{L(h)} > \tau^{\star} \Gamma\left(\sqrt{p},\frac{\epsilon}{N^{\frac{1}{\sqrt{p}} - \frac{1}{p}}}\right)\epsilon\Bigg) \\
					&< 12 \exp\left\{d_{VC}(\mathcal{H})\left(1 + \ln \frac{2N}{d_{VC}(\mathcal{H})}\right) - \frac{\epsilon^{2}N^{\frac{2(\sqrt{p} - 1)}{\sqrt{p}} - \frac{2}{\sqrt{p}} + \frac{2}{p}}}{2^{\frac{\sqrt{p} + 2}{2}}} \right\}.
				\end{align*}
			\end{itemize}
			In both cases, if $d_{VC}(\mathcal{H}) < \infty$, then, by Borel-Cantelli Lemma,
			\begin{equation*}
				\sup\limits_{h \in \mathcal{H}} \frac{\text{\textbar}L(h) - L_{\mathcal{D}_{N}}(h)\text{\textbar}}{L(h)} \xrightarrow[N \to \infty]{} 0,
			\end{equation*}
			with probability one.
		\end{corollary}
		
		Corollary \ref{convergence_relativeTI} establishes the convergence to zero of relative type I estimation error, and concludes our study of type I estimation error convergence in classical VC theory.
		
		\begin{remark}
			We simplified the bounds in Corollary \ref{convergence_relativeTI} since, by combining Theorems \ref{change_denominator} with \ref{theorem_unbounded}, we obtain a bound with three terms of different orders in $N$, where the exponential in each of them is multiplied by four. The term of the greatest order is that we show in Corollary \ref{convergence_relativeTI}, with the exponential multiplied by twelve, since we can bound the two terms of lesser order by the one of the greatest order. This worsens the bound for fixed $N$, but eases notation and has the same qualitative effect of presenting an exponential bound for relative type I estimation error, which implies its almost sure convergence to zero.
		\end{remark}
		
		\begin{remark}
			Observe that $d_{VC}(\mathcal{M},\ell) = d_{VC}(\mathcal{M},\ell^{\prime})$, in which $\ell^{\prime}(z,h) = \left(\ell(z,h)\right)^{q}, q = \sqrt{p},$ as defined in Theorem \ref{change_denominator}, hence we can bound all three terms of the inequality in this theorem by $d_{VC}(\mathcal{M},\ell)$, as is done in Corollary \ref{convergence_relativeTI}. The VC dimension equality is true since $\mathcal{G}_{\mathcal{M},\ell} = \mathcal{G}_{\mathcal{M},\ell^{\prime}}$ (cf. Definition \ref{VCdimension}), as each function $g_{\beta,h}(z) = \mathds{1}\{\ell(z,h) \geq \beta\}$ in $\mathcal{G}_{\mathcal{M},\ell}$ has a correspondent $g^{\prime}_{\beta^{q},h}(z) = \mathds{1}\{\ell^{\prime}(z,h) \geq \beta^{q}\}$ in $\mathcal{G}_{\mathcal{M},\ell^{\prime}}$ that is such that $g \equiv g^{\prime}$. 
		\end{remark}
		
		\begin{remark}
			\label{remark_finite_moments}
			Condition \eqref{finite_moments} is not actually satisfied by many $\mathcal{H}$, for instance it does not hold for linear regression under the quadratic loss function. However, one can actually consider a $\mathcal{M} \subset \mathcal{H}$ such that \eqref{finite_moments} is true, with $d_{VC}(\mathcal{M}) = d_{VC}(\mathcal{H})$ and $L(h^{\star}) = L(h^{\star}_{\mathcal{M}})$, without loss of generality. For example, in linear regression one could consider only hypotheses with parameters bounded by a \textit{very large} constant $\gamma$, excluding hypotheses that are unlikely to be the target one. Observe that, in this example, it is better to consider the bounds for relative type I estimation error of unbounded loss functions, rather than consider that the loss function is bounded by a very large constant $C = \mathcal{O}(\gamma^2)$, which would generate bad bounds when applying Corollary \ref{cor3TypeI}. The results for unbounded loss functions hold for bounded ones, with $p$ arbitrarily large.
		\end{remark}
		
		\begin{remark}
			\label{remark_geq1}
			The main reason we assume that $\ell(z,h) \geq 1$, for all $z \in \mathcal{Z}$ and $h \in \mathcal{H}$, is to simplify the argument before \eqref{Ap1Eq11}, which could fail if the losses were lesser than one. We believe this assumption could be dropped at the cost of more technical results. Nevertheless, the results in Corollary \ref{convergence_relativeTI} present an exponential bound for
			\begin{equation*}
				\mathbb{P}\left(\sup\limits_{h \in \mathcal{H}} \frac{\text{\textbar}L(h) - L_{\mathcal{D}_{N}}(h)}{L(h) + 1}\text{\textbar} > \epsilon\right)
			\end{equation*}
			for any unbounded loss function $\ell$. We note that, if we had not imposed this constraint in the loss function, we would have to deal with the denominators in the estimation errors, which could then be zero. This could have been easily accomplished by summing a constant to the denominators and then making it go to zero after the bounds are established, that is, find bounds for
			\begin{equation*}
				\mathbb{P}\left(\sup\limits_{h \in \mathcal{H}} \frac{\text{\textbar}L(h) - L_{\mathcal{D}_{N}}(h)}{L(h) + \varsigma}\text{\textbar} > \epsilon\right),
			\end{equation*}
			and then make $\varsigma \to 0$. This is done in \cite{cortes2019}. By considering loss functions greater or equal to one, we have avoided the need to have heavier notations and more technical details when establishing the convergence of relative estimation errors. 
		\end{remark}
		
		\subsection{Convergence to zero of type II estimation error}
		
		A bound for type II estimation error \eqref{GE2} and relative type II estimation error, defined as
		\begin{equation*}
			\textbf{(II)} \ \ \frac{L(\hat{h}^{\mathcal{D}_{N}}) - L(h^{\star})}{L(\hat{h}^{\mathcal{D}_{N}})}
		\end{equation*}
		follow immediately from a bound obtained for the respective type I error. This is a consequence of the following elementary inequality, which can be found in part in \cite[Lemma~8.2]{devroye1996}.
		
		\begin{lemma}
			\label{lemmaTypeItoII}
			For any hypotheses space $\mathcal{H}$ and possible sample $\mathcal{D}_{N}$,
			\begin{align}
				\label{first_inequality}
				L(\hat{h}^{\mathcal{D}_{N}}) - L(h^{\star}) \leq 2 \ \sup\limits_{h \in \mathcal{H}} \text{\textbar}L(h) - L_{\mathcal{D}_{N}}(h)\text{\textbar},
			\end{align}
			and, if $\ell(z,h) \geq 1$, for all $z \in \mathcal{Z}$ and $h \in \mathcal{H}$, then
			\begin{equation}
				\label{second_inequality}
				\frac{L(\hat{h}^{\mathcal{D}_{N}}) - L(h^{\star})}{L(\hat{h}^{\mathcal{D}_{N}})} \leq 2 \ \sup\limits_{h \in \mathcal{H}} \frac{\text{\textbar}L(h) - L_{\mathcal{D}_{N}}(h)\text{\textbar}}{L(h)}.
			\end{equation}
			These inequalities yield
			\begin{align}	
				\label{firstPinequality}
				\mathbb{P}\left(L(\hat{h}^{\mathcal{D}_{N}}) - L(h^{\star}) > \epsilon\right) &\leq \mathbb{P}\left(\sup\limits_{h \in \mathcal{H}} \text{\textbar}L(h) - L_{\mathcal{D}_{N}}(h)\text{\textbar} > \epsilon/2\right)
			\end{align}
			and
			\begin{equation}	
				\label{secondPinequality}
				\mathbb{P}\left(\frac{L(\hat{h}^{\mathcal{D}_{N}}) - L(h^{\star})}{L(\hat{h}^{\mathcal{D}_{N}})} > \epsilon\right) \leq \mathbb{P}\left(\sup\limits_{h \in \mathcal{H}} \frac{\text{\textbar}L(h) - L_{\mathcal{D}_{N}}(h)\text{\textbar}}{L(h)} > \epsilon/2\right).
			\end{equation}
		\end{lemma}
		\begin{proof}
			The first inequality follows from
			\begin{align*}
				L(\hat{h}^{\mathcal{D}_{N}}) - L(h^{\star}) &= L(\hat{h}^{\mathcal{D}_{N}}) - L_{\mathcal{D}_{N}}(\hat{h}^{\mathcal{D}_{N}}) + L_{\mathcal{D}_{N}}(\hat{h}^{\mathcal{D}_{N}}) - L(h^{\star})\\
				&\leq  L(\hat{h}^{\mathcal{D}_{N}}) - L_{\mathcal{D}_{N}}(\hat{h}^{\mathcal{D}_{N}}) + L_{\mathcal{D}_{N}}(h^{\star}) - L(h^{\star})\\
				&\leq 2 \ \sup\limits_{h \in \mathcal{H}} \text{\textbar}L(h) - L_{\mathcal{D}_{N}}(h)\text{\textbar}.
			\end{align*}
			For the second one, analogous to the deduction above, we have that
			\begin{align*}
				\frac{L(\hat{h}^{\mathcal{D}_{N}}) - L(h^{\star})}{L(\hat{h}^{\mathcal{D}_{N}})} &\leq  \frac{L(\hat{h}^{\mathcal{D}_{N}}) - L_{\mathcal{D}_{N}}(\hat{h}^{\mathcal{D}_{N}})}{L(\hat{h}^{\mathcal{D}_{N}})} + \frac{L_{\mathcal{D}_{N}}(h^{\star}) - L(h^{\star})}{L(h^{\star})}\\
				&\leq 2 \ \sup\limits_{h \in \mathcal{H}} \frac{\text{\textbar}L(h) - L_{\mathcal{D}_{N}}(h)\text{\textbar}}{L(h)},
			\end{align*}
			since $L(\hat{h}^{\mathcal{D}_{N}}) \geq L(h^{\star})$. The inequalities \eqref{firstPinequality} and \eqref{secondPinequality} are direct from \eqref{first_inequality} and \eqref{second_inequality}.
		\end{proof}
		
		Combining Lemma \ref{lemmaTypeItoII} with Corollaries \ref{cor2TypeI} and \ref{cor3TypeI} we obtain the consistency of type II estimation error, when $d_{VC}(\mathcal{H}) < \infty$ and the loss function is bounded, what also concerns binary loss functions.
		
		\begin{corollary}
			\label{cor1TypeII}
			Fix a hypotheses space $\mathcal{H}$ and a loss function $\ell: \mathcal{Z} \times \mathcal{H} \mapsto \mathbb{R}_{+}$, with $0 \leq \ell(z,h) \leq C$ for all $z \in \mathcal{Z}, h \in \mathcal{H}$. Then,
			\begin{equation}
				\label{Ap1Eq7}
				\mathbb{P}\left(L(\hat{h}^{\mathcal{D}_{N}}) - L(h^{\star}) > \epsilon\right) \leq 8 \ \exp\left\{d_{VC}(\mathcal{H}) \left(1 + \ln \frac{N}{d_{VC}(\mathcal{H})}\right) - N \frac{\epsilon^2}{128C^{2}}\right\}.
			\end{equation}
			In particular, if $d_{VC}(\mathcal{H}) < \infty$, not only \eqref{Ap1Eq7} converges to zero, but also
			\begin{equation*}
				L(\hat{h}^{\mathcal{D}_{N}}) - L(h^{\star}) \xrightarrow[N \to \infty]{} 0,
			\end{equation*}
			with probability one by Borel-Cantelli Lemma.
		\end{corollary}
		
		Finally, combining Lemma \ref{lemmaTypeItoII} with Corollary \ref{convergence_relativeTI}, we obtain the consistency of relative type II estimation error when $d_{VC}(\mathcal{H}) < \infty$, the loss function is unbounded, and $P$ satisfies \eqref{tauStarAp}.
		
		\begin{corollary}
			\label{cor2TypeII}
			Fix a hypotheses space $\mathcal{H}$, an unbounded loss function $\ell$ and $\epsilon > 0$. The following holds:
			\begin{itemize}
				\item If \eqref{tauStarAp} holds for a $p \geq 4$ fixed, then
				\begin{align*}
					\mathbb{P}\Bigg(&\frac{L(\hat{h}^{\mathcal{D}_{N}}) - L(h^{\star})}{L(\hat{h}^{\mathcal{D}_{N}})} > \tau^{\star} \Lambda(\sqrt{p}) \epsilon\Bigg) \\
					&< 12 \exp\left\{d_{VC}(\mathcal{H})\left(1 + \ln \frac{2N}{d_{VC}(\mathcal{H})}\right) - \frac{\epsilon^{2}(1-\epsilon/2)^{2}N^{1 - \frac{2}{\sqrt{p}} + \frac{2}{p}}}{16}\right\}
				\end{align*}
				
				\item If \eqref{tauStarAp} holds for a $1 < p < 4$ fixed, then
				\begin{align*}
					\mathbb{P}\Bigg(&\frac{L(\hat{h}^{\mathcal{D}_{N}}) - L(h^{\star})}{L(\hat{h}^{\mathcal{D}_{N}})} > \tau^{\star} \Gamma\left(\sqrt{p},\frac{\epsilon}{N^{\frac{1}{\sqrt{p}} - \frac{1}{p}}}\right)\epsilon\Bigg) \\
					&< 12 \exp\left\{d_{VC}(\mathcal{H})\left(1 + \ln \frac{2N}{d_{VC}(\mathcal{H})}\right) - \frac{\epsilon^{2}N^{\frac{2(\sqrt{p} - 1)}{\sqrt{p}} - \frac{2}{\sqrt{p}} + \frac{2}{p}}}{2^{\frac{\sqrt{p} + 6}{2}}} \right\}.
				\end{align*}
			\end{itemize}
			In any case, if $d_{VC}(\mathcal{H}) < \infty$, then, by Borel-Cantelli Lemma,
			\begin{equation*}
				\sup\limits_{h \in \mathcal{H}} \frac{L(\hat{h}^{\mathcal{D}_{N}}) - L(h^{\star})}{L(\hat{h}^{\mathcal{D}_{N}})} \xrightarrow[N \to \infty]{} 0,
			\end{equation*}
			with probability one.
		\end{corollary}
		
		This ends the study of type II estimation error convergence.
		
		\section{Results of the experiments}
		\label{AppResults}
		
		
		\begin{table}[ht]
			\centering
			\caption{Joint distributions considered in each example in Section \ref{SecPLLS}.} \label{jointDist}
			\resizebox{\linewidth}{!}{\begin{tabular}{|c|cc|cc|cc|cc|cc|cc|cc|}
					\hline
					\multirow{2}{*}{$x$} & \multicolumn{2}{c|}{Example 1} & \multicolumn{2}{c|}{Example 2}  & \multicolumn{2}{c|}{Example 3}  & \multicolumn{2}{c|}{Example 4} & \multicolumn{2}{c|}{Example 5} & \multicolumn{2}{c|}{Example 6} & \multicolumn{2}{c|}{Example 7}  \\ 
					& $p(0,x)$ & $p(1,x)$ & $p(0,x)$ & $p(1,x)$ & $p(0,x)$ & $p(1,x)$ & $p(0,x)$ & $p(1,x)$ & $p(0,x)$ & $p(1,x)$ & $p(0,x)$ & $p(1,x)$ & $p(0,x)$ & $p(1,x)$ \\ 
					\hline
					1 & 0.0612 & 0.0638 & 0.0600 & 0.0650 & 0.0563 & 0.0688 & 0.0500 & 0.0750 & 0.0375 & 0.0875 & 0.0250 & 0.1000 & 0.0125 & 0.1125 \\ 
					2 & 0.0638 & 0.0612 & 0.0650 & 0.0600 & 0.0688 & 0.0563 & 0.0750 & 0.0500 & 0.0875 & 0.0375 & 0.1000 & 0.0250 & 0.1125 & 0.0125 \\ 
					3 & 0.0612 & 0.0638 & 0.0600 & 0.0650 & 0.0563 & 0.0688 & 0.0500 & 0.0750 & 0.0375 & 0.0875 & 0.0250 & 0.1000 & 0.0125 & 0.1125 \\ 
					4 & 0.0638 & 0.0612 & 0.0650 & 0.0600 & 0.0688 & 0.0563 & 0.0750 & 0.0500 & 0.0875 & 0.0375 & 0.1000 & 0.0250 & 0.1125 & 0.0125 \\ 
					5 & 0.0612 & 0.0638 & 0.0600 & 0.0650 & 0.0563 & 0.0688 & 0.0500 & 0.0750 & 0.0375 & 0.0875 & 0.0250 & 0.1000 & 0.0125 & 0.1125 \\ 
					6 & 0.0638 & 0.0612 & 0.0650 & 0.0600 & 0.0688 & 0.0563 & 0.0750 & 0.0500 & 0.0875 & 0.0375 & 0.1000 & 0.0250 & 0.1125 & 0.0125 \\ 
					7 & 0.0612 & 0.0638 & 0.0600 & 0.0650 & 0.0563 & 0.0688 & 0.0500 & 0.0750 & 0.0375 & 0.0875 & 0.0250 & 0.1000 & 0.0125 & 0.1125 \\ 
					8 & 0.0638 & 0.0612 & 0.0650 & 0.0600 & 0.0688 & 0.0563 & 0.0750 & 0.0500 & 0.0875 & 0.0375 & 0.1000 & 0.0250 & 0.1125 & 0.0125 \\ 
					\hline $\epsilon^{\star}$ & \multicolumn{2}{c|}{0.0025}  & \multicolumn{2}{c|}{0.005}  & \multicolumn{2}{c|}{0.0125}  & \multicolumn{2}{c|}{0.025}  & \multicolumn{2}{c|}{0.050}  & \multicolumn{2}{c|}{0.075}  & \multicolumn{2}{c|}{0.100}  \\ 
					CE & \multicolumn{2}{c|}{0.693}  & \multicolumn{2}{c|}{0.692}  & \multicolumn{2}{c|}{0.688}  & \multicolumn{2}{c|}{0.673}  & \multicolumn{2}{c|}{0.611}  & \multicolumn{2}{c|}{0.500}  & \multicolumn{2}{c|}{0.325}  \\ 
					$L(h^{\star})$ & \multicolumn{2}{c|}{0.490}  & \multicolumn{2}{c|}{0.480}  & \multicolumn{2}{c|}{0.450}  & \multicolumn{2}{c|}{0.400}  & \multicolumn{2}{c|}{0.300}  & \multicolumn{2}{c|}{0.200}  & \multicolumn{2}{c|}{0.100}  \\ 
					\hline
					\multicolumn{15}{l}{\footnotesize $p(0,x) = \mathbb{P}(Y = 0,X = x)$, $p(1,x) = \mathbb{P}(Y = 1,X = x)$, CE: Conditional Entropy}
			\end{tabular}}
		\end{table}
		
		\begin{figure}[ht]
			\centering
			\includegraphics[width=\linewidth]{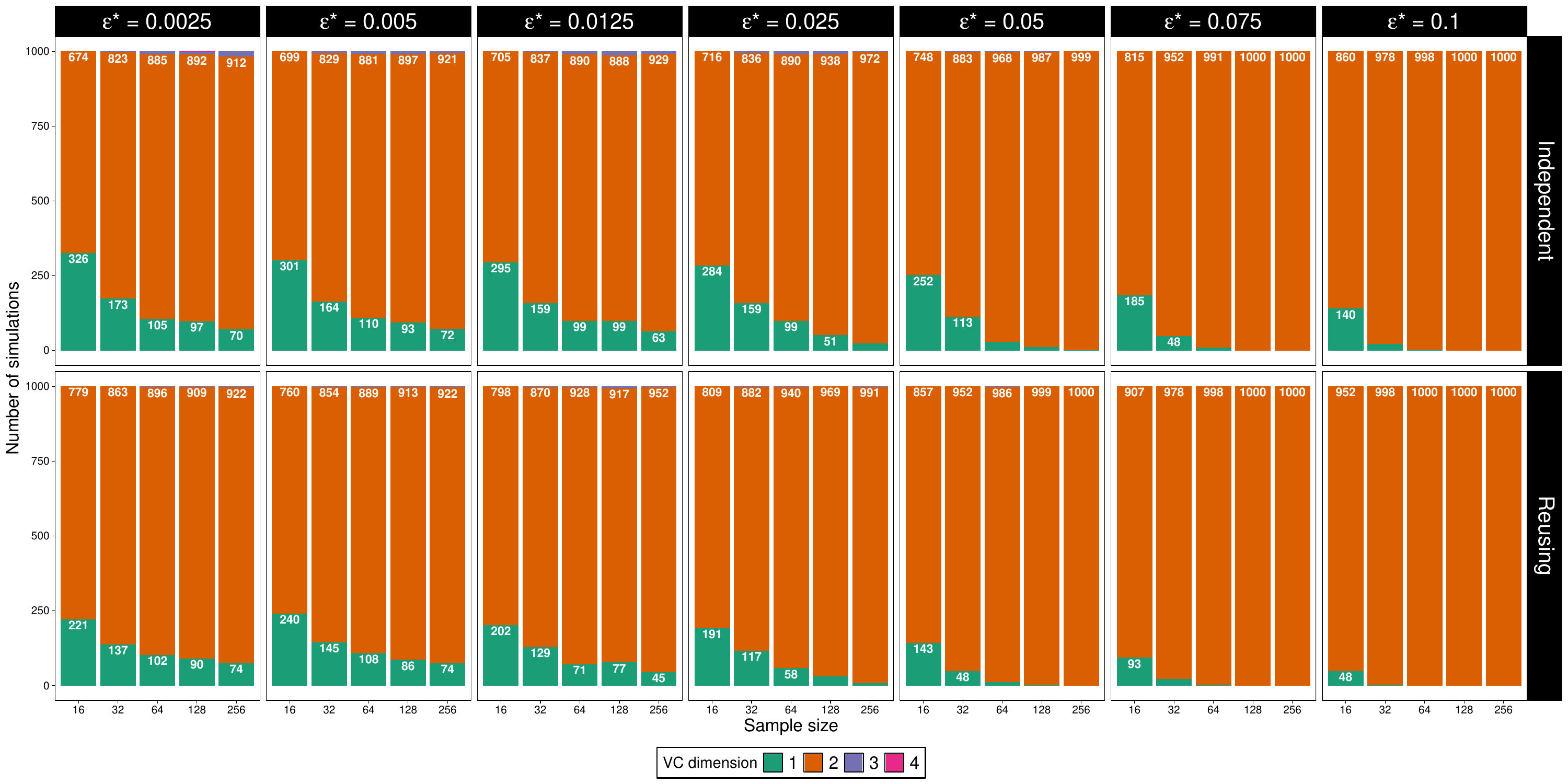}
			\caption{The frequency of $d_{VC}(\hat{\mathcal{M}})$ over the 1,000 samples in each case when learning via model selection in the whole partition lattice Learning Space in the example of Section \ref{SecPLLS}.} \label{fig_PLLS4}
		\end{figure}
		
		\begin{figure}[ht]
			\centering
			\includegraphics[width=\linewidth]{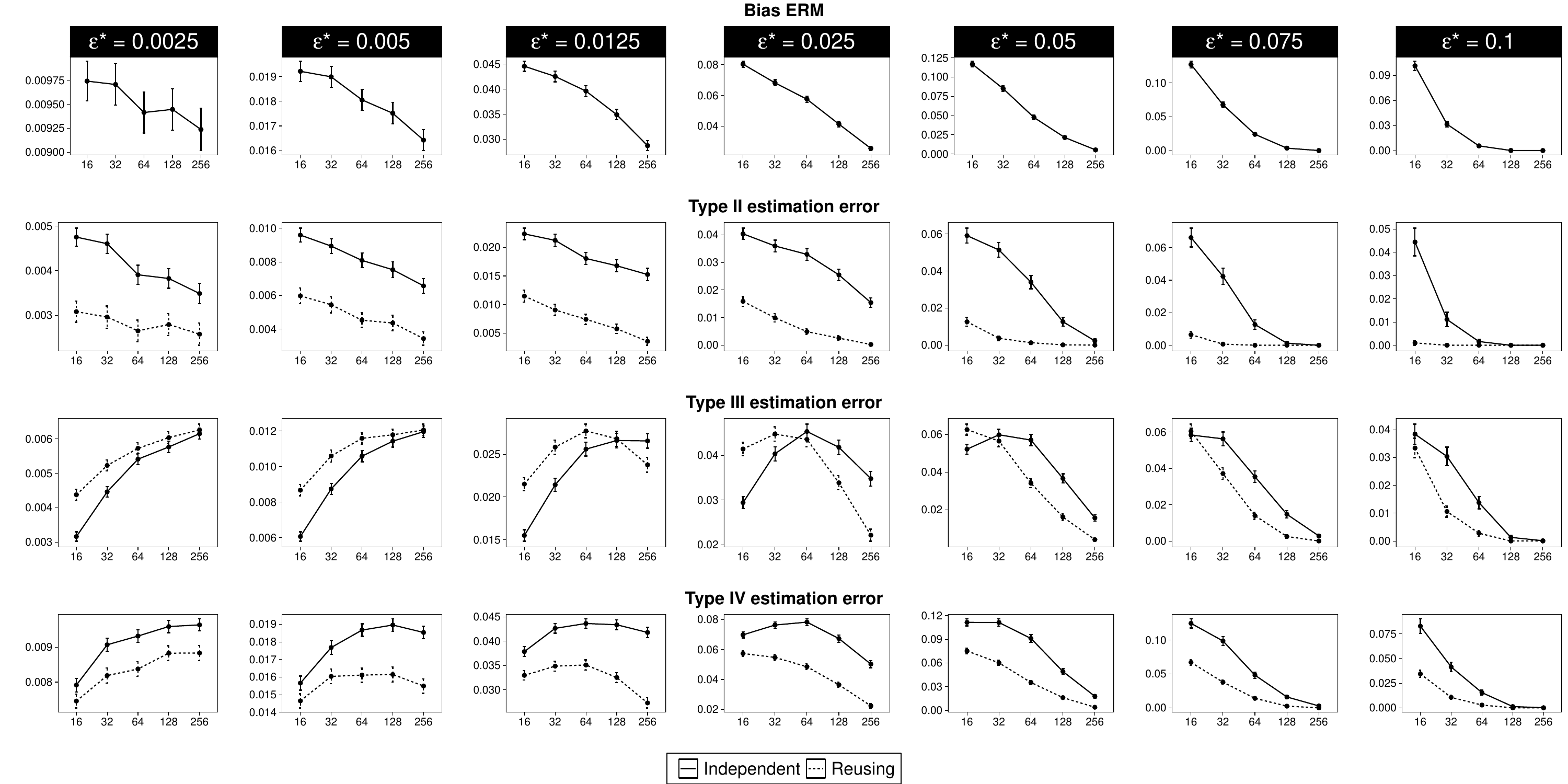}
			\caption{The average bias and 95\% confidence interval of the ERM hypotheses and types II, III, and IV estimation error over the 1,000 simulated samples for each case when learning via model selection in the models with VC dimension 2 in the partition lattice Learning Space in the example of Section \ref{SecPLLS}. These results are also presented in Table \ref{res_PLLS2}.} \label{fig_PLLS5}
		\end{figure}
		
		\footnotesize
		\begin{longtable}{lllllll}	
			\caption{\normalsize Average and standard error of the ERM hypotheses bias and estimation errors over the 1,000 samples simulated in Section \ref{SecPLLS} for learning in the whole partition lattice Learning Space.} \label{res_PLLS} \\
			\hline
			$\epsilon^{\star}$ & Learning & Size & Bias ERM & Type II & Type III & Type IV \\ 
			\hline
			\multirow{10}{*}{0.0025} & \multirow{5}{*}{Independent} & 16 & 0.0097 (1e-04) & 0.003 (1e-04) & 0.0059 (1e-04) & 0.0089 (1e-04) \\ 
			&  & 32 & 0.0097 (1e-04) & 0.0036 (1e-04) & 0.0059 (1e-04) & 0.0095 (1e-04) \\ 
			&  & 64 & 0.0094 (1e-04) & 0.0032 (1e-04) & 0.0063 (1e-04) & 0.0095 (1e-04) \\ 
			&  & 128 & 0.0094 (1e-04) & 0.0031 (1e-04) & 0.0066 (1e-04) & 0.0097 (1e-04) \\ 
			&  & 256 & 0.0092 (1e-04) & 0.0029 (1e-04) & 0.0067 (1e-04) & 0.0097 (1e-04) \\  \cline{2-7}
			& \multirow{5}{*}{Reusing} & 16 & 0.0097 (1e-04) & 0.0025 (1e-04) & 0.0059 (1e-04) & 0.0084 (1e-04) \\ 
			&  & 32 & 0.0097 (1e-04) & 0.0028 (1e-04) & 0.0062 (1e-04) & 0.009 (1e-04) \\ 
			&  & 64 & 0.0094 (1e-04) & 0.0027 (1e-04) & 0.0065 (1e-04) & 0.0092 (1e-04) \\ 
			&  & 128 & 0.0094 (1e-04) & 0.0025 (1e-04) & 0.0068 (1e-04) & 0.0093 (1e-04) \\ 
			&  & 256 & 0.0092 (1e-04) & 0.0024 (1e-04) & 0.0068 (1e-04) & 0.0093 (1e-04) \\ 
			\hline
			\multirow{10}{*}{0.005} & \multirow{5}{*}{Independent} & 16 & 0.0192 (2e-04) & 0.0064 (2e-04) & 0.0112 (2e-04) & 0.0177 (2e-04) \\ 
			&  & 32 & 0.019 (2e-04) & 0.007 (2e-04) & 0.0116 (2e-04) & 0.0185 (2e-04) \\ 
			&  & 64 & 0.0181 (2e-04) & 0.0067 (2e-04) & 0.0124 (2e-04) & 0.0191 (2e-04) \\ 
			&  & 128 & 0.0175 (2e-04) & 0.0062 (2e-04) & 0.013 (2e-04) & 0.0192 (2e-04) \\ 
			&  & 256 & 0.0164 (2e-04) & 0.0057 (2e-04) & 0.0131 (2e-04) & 0.0188 (2e-04) \\  \cline{2-7}
			& \multirow{5}{*}{Reusing} & 16 & 0.0192 (2e-04) & 0.0048 (2e-04) & 0.0116 (2e-04) & 0.0164 (2e-04) \\ 
			&  & 32 & 0.019 (2e-04) & 0.0056 (2e-04) & 0.0122 (2e-04) & 0.0178 (2e-04) \\ 
			&  & 64 & 0.0181 (2e-04) & 0.0046 (2e-04) & 0.013 (2e-04) & 0.0177 (2e-04) \\ 
			&  & 128 & 0.0175 (2e-04) & 0.0043 (2e-04) & 0.0133 (2e-04) & 0.0176 (2e-04) \\ 
			&  & 256 & 0.0164 (2e-04) & 0.0035 (2e-04) & 0.0134 (2e-04) & 0.017 (2e-04) \\ 
			\hline
			\multirow{10}{*}{0.0125} & \multirow{5}{*}{Independent} & 16 & 0.0446 (5e-04) & 0.0154 (5e-04) & 0.0279 (5e-04) & 0.0433 (4e-04) \\ 
			&  & 32 & 0.0426 (5e-04) & 0.0166 (6e-04) & 0.028 (4e-04) & 0.0446 (5e-04) \\ 
			&  & 64 & 0.0396 (6e-04) & 0.0154 (5e-04) & 0.0297 (4e-04) & 0.0451 (5e-04) \\ 
			&  & 128 & 0.0349 (5e-04) & 0.0141 (5e-04) & 0.0307 (4e-04) & 0.0448 (5e-04) \\ 
			&  & 256 & 0.0287 (5e-04) & 0.0136 (5e-04) & 0.0293 (4e-04) & 0.0429 (5e-04) \\  \cline{2-7}
			& \multirow{5}{*}{Reusing} & 16 & 0.0446 (5e-04) & 0.0102 (5e-04) & 0.0283 (5e-04) & 0.0385 (5e-04) \\ 
			&  & 32 & 0.0426 (5e-04) & 0.0104 (5e-04) & 0.0297 (4e-04) & 0.04 (5e-04) \\ 
			&  & 64 & 0.0396 (6e-04) & 0.0087 (4e-04) & 0.0305 (4e-04) & 0.0392 (5e-04) \\ 
			&  & 128 & 0.0349 (5e-04) & 0.0059 (4e-04) & 0.0309 (4e-04) & 0.0368 (5e-04) \\ 
			&  & 256 & 0.0287 (5e-04) & 0.0042 (3e-04) & 0.0282 (4e-04) & 0.0324 (5e-04) \\ 
			\hline
			\multirow{10}{*}{0.025} & \multirow{5}{*}{Independent} & 16 & 0.0804 (0.001) & 0.0287 (0.0011) & 0.0535 (0.0011) & 0.0822 (9e-04) \\ 
			&  & 32 & 0.0682 (0.001) & 0.0288 (0.0011) & 0.0538 (9e-04) & 0.0825 (0.001) \\ 
			&  & 64 & 0.0575 (0.001) & 0.0283 (0.0011) & 0.0535 (9e-04) & 0.0819 (0.0011) \\ 
			&  & 128 & 0.0413 (9e-04) & 0.024 (0.001) & 0.046 (9e-04) & 0.07 (0.0012) \\ 
			&  & 256 & 0.0255 (7e-04) & 0.0148 (9e-04) & 0.0365 (9e-04) & 0.0513 (0.0012) \\  \cline{2-7}
			& \multirow{5}{*}{Reusing} & 16 & 0.0804 (0.001) & 0.0165 (8e-04) & 0.0527 (0.001) & 0.0692 (0.0011) \\ 
			&  & 32 & 0.0682 (0.001) & 0.0119 (7e-04) & 0.0537 (9e-04) & 0.0656 (0.001) \\ 
			&  & 64 & 0.0575 (0.001) & 0.0093 (6e-04) & 0.0508 (9e-04) & 0.0601 (0.001) \\ 
			&  & 128 & 0.0413 (9e-04) & 0.0052 (4e-04) & 0.0434 (9e-04) & 0.0486 (9e-04) \\ 
			&  & 256 & 0.0255 (7e-04) & 0.0024 (2e-04) & 0.0334 (8e-04) & 0.0358 (9e-04) \\ 
			\hline
			\multirow{10}{*}{0.05} & \multirow{5}{*}{Independent} & 16 & 0.1169 (0.0019) & 0.0464 (0.002) & 0.095 (0.0022) & 0.1414 (0.0022) \\ 
			&  & 32 & 0.085 (0.0018) & 0.044 (0.002) & 0.0787 (0.002) & 0.1227 (0.0025) \\ 
			&  & 64 & 0.0475 (0.0015) & 0.0321 (0.0018) & 0.0616 (0.0017) & 0.0937 (0.0025) \\ 
			&  & 128 & 0.0215 (0.001) & 0.0124 (0.0012) & 0.0384 (0.0014) & 0.0508 (0.002) \\ 
			&  & 256 & 0.0054 (5e-04) & 0.0023 (5e-04) & 0.0158 (9e-04) & 0.0181 (0.0011) \\  \cline{2-7}
			& \multirow{5}{*}{Reusing} & 16 & 0.1169 (0.0019) & 0.0169 (0.001) & 0.0836 (0.002) & 0.1005 (0.0021) \\ 
			&  & 32 & 0.085 (0.0018) & 0.0146 (9e-04) & 0.0703 (0.0018) & 0.0849 (0.0019) \\ 
			&  & 64 & 0.0475 (0.0015) & 0.0068 (6e-04) & 0.0562 (0.0016) & 0.063 (0.0017) \\ 
			&  & 128 & 0.0215 (0.001) & 0.0026 (4e-04) & 0.0352 (0.0012) & 0.0378 (0.0013) \\ 
			&  & 256 & 0.0054 (5e-04) & 3e-04 (1e-04) & 0.0155 (9e-04) & 0.0159 (9e-04) \\ 
			\hline
			\multirow{10}{*}{0.075} & \multirow{5}{*}{Independent} & 16 & 0.1273 (0.0025) & 0.0552 (0.0029) & 0.106 (0.0034) & 0.1612 (0.0039) \\ 
			&  & 32 & 0.0677 (0.002) & 0.0403 (0.0025) & 0.0679 (0.0025) & 0.1082 (0.0036) \\ 
			&  & 64 & 0.0241 (0.0013) & 0.0124 (0.0015) & 0.0377 (0.0018) & 0.05 (0.0024) \\ 
			&  & 128 & 0.0036 (5e-04) & 0.0013 (4e-04) & 0.0148 (0.001) & 0.0161 (0.0012) \\ 
			&  & 256 & 1e-04 (1e-04) & 0 (0) & 0.0029 (5e-04) & 0.0029 (5e-04) \\  \cline{2-7}
			& \multirow{5}{*}{Reusing} & 16 & 0.1273 (0.0025) & 0.0173 (0.0011) & 0.0842 (0.0029) & 0.1014 (0.003) \\ 
			&  & 32 & 0.0677 (0.002) & 0.0119 (9e-04) & 0.0575 (0.0022) & 0.0695 (0.0023) \\ 
			&  & 64 & 0.0241 (0.0013) & 0.0038 (5e-04) & 0.0352 (0.0016) & 0.0389 (0.0017) \\ 
			&  & 128 & 0.0036 (5e-04) & 8e-04 (2e-04) & 0.0146 (0.001) & 0.0153 (0.001) \\ 
			&  & 256 & 1e-04 (1e-04) & 0 (0) & 0.0029 (5e-04) & 0.0029 (5e-04) \\ 
			\hline
			\multirow{10}{*}{0.1} & \multirow{5}{*}{Independent} & 16 & 0.1014 (0.0028) & 0.0383 (0.0029) & 0.0864 (0.0043) & 0.1247 (0.0051) \\ 
			&  & 32 & 0.0316 (0.0017) & 0.0105 (0.0015) & 0.0382 (0.0024) & 0.0487 (0.003) \\ 
			&  & 64 & 0.0056 (7e-04) & 0.0016 (5e-04) & 0.0144 (0.0013) & 0.016 (0.0014) \\ 
			&  & 128 & 1e-04 (1e-04) & 0 (0) & 0.0013 (4e-04) & 0.0013 (4e-04) \\ 
			&  & 256 & 0 (0) & 0 (0) & 1e-04 (1e-04) & 1e-04 (1e-04) \\  \cline{2-7}
			& \multirow{5}{*}{Reusing} & 16 & 0.1014 (0.0028) & 0.0128 (0.0011) & 0.053 (0.003) & 0.0658 (0.0032) \\ 
			&  & 32 & 0.0316 (0.0017) & 0.0038 (6e-04) & 0.0296 (0.0017) & 0.0334 (0.0018) \\ 
			&  & 64 & 0.0056 (7e-04) & 0.0013 (4e-04) & 0.0137 (0.0011) & 0.015 (0.0012) \\ 
			&  & 128 & 1e-04 (1e-04) & 0 (0) & 0.0013 (4e-04) & 0.0013 (4e-04) \\ 
			&  & 256 & 0 (0) & 0 (0) & 1e-04 (1e-04) & 1e-04 (1e-04) \\ 
			\hline
		\end{longtable}
		
		\begin{longtable}{lllllll}
			\caption{\normalsize Average and standard error of the ERM hypotheses bias and estimation errors over the 1,000 samples simulated in Section \ref{SecPLLS} for learning in the partition lattice Learning Space restricted to the models with VC dimension 2.} \label{res_PLLS2} \\
			\hline
			$\epsilon^{\star}$ & Learning & Size & Bias ERM & Type II & Type III & Type IV \\ 
			\hline
			\multirow{10}{*}{0.0025} & \multirow{5}{*}{Independent} & 16 & 0.0097 (1e-04) & 0.0048 (1e-04) & 0.0032 (1e-04) & 0.0079 (1e-04) \\ 
			&  & 32 & 0.0097 (1e-04) & 0.0046 (1e-04) & 0.0045 (1e-04) & 0.0091 (1e-04) \\ 
			&  & 64 & 0.0094 (1e-04) & 0.0039 (1e-04) & 0.0054 (1e-04) & 0.0093 (1e-04) \\ 
			&  & 128 & 0.0094 (1e-04) & 0.0038 (1e-04) & 0.0058 (1e-04) & 0.0096 (1e-04) \\ 
			&  & 256 & 0.0092 (1e-04) & 0.0035 (1e-04) & 0.0062 (1e-04) & 0.0096 (1e-04) \\ \cline{2-7}
			& \multirow{5}{*}{Reusing} & 16 & 0.0097 (1e-04) & 0.0031 (1e-04) & 0.0044 (1e-04) & 0.0075 (1e-04) \\ 
			&  & 32 & 0.0097 (1e-04) & 0.003 (1e-04) & 0.0052 (1e-04) & 0.0082 (1e-04) \\ 
			&  & 64 & 0.0094 (1e-04) & 0.0026 (1e-04) & 0.0057 (1e-04) & 0.0084 (1e-04) \\ 
			&  & 128 & 0.0094 (1e-04) & 0.0028 (1e-04) & 0.006 (1e-04) & 0.0088 (1e-04) \\ 
			&  & 256 & 0.0092 (1e-04) & 0.0026 (1e-04) & 0.0063 (1e-04) & 0.0088 (1e-04) \\ 
			\hline
			\multirow{10}{*}{0.005} & \multirow{5}{*}{Independent} & 16 & 0.0192 (2e-04) & 0.0096 (2e-04) & 0.0061 (1e-04) & 0.0157 (2e-04) \\ 
			&  & 32 & 0.019 (2e-04) & 0.0089 (2e-04) & 0.0087 (2e-04) & 0.0177 (2e-04) \\ 
			&  & 64 & 0.0181 (2e-04) & 0.0081 (2e-04) & 0.0106 (2e-04) & 0.0187 (2e-04) \\ 
			&  & 128 & 0.0175 (2e-04) & 0.0075 (2e-04) & 0.0114 (2e-04) & 0.019 (2e-04) \\ 
			&  & 256 & 0.0164 (2e-04) & 0.0066 (2e-04) & 0.012 (2e-04) & 0.0185 (2e-04) \\  \cline{2-7}
			& \multirow{5}{*}{Reusing} & 16 & 0.0192 (2e-04) & 0.006 (2e-04) & 0.0087 (2e-04) & 0.0147 (2e-04) \\ 
			&  & 32 & 0.019 (2e-04) & 0.0055 (2e-04) & 0.0106 (2e-04) & 0.016 (2e-04) \\ 
			&  & 64 & 0.0181 (2e-04) & 0.0045 (2e-04) & 0.0116 (2e-04) & 0.0161 (2e-04) \\ 
			&  & 128 & 0.0175 (2e-04) & 0.0044 (2e-04) & 0.0118 (2e-04) & 0.0162 (2e-04) \\ 
			&  & 256 & 0.0164 (2e-04) & 0.0034 (2e-04) & 0.0121 (2e-04) & 0.0155 (2e-04) \\ 
			\hline
			\multirow{10}{*}{0.0125} & \multirow{5}{*}{Independent} & 16 & 0.0446 (5e-04) & 0.0224 (5e-04) & 0.0155 (3e-04) & 0.0379 (5e-04) \\ 
			&  & 32 & 0.0426 (5e-04) & 0.0212 (6e-04) & 0.0214 (4e-04) & 0.0427 (5e-04) \\ 
			&  & 64 & 0.0396 (6e-04) & 0.0181 (5e-04) & 0.0256 (4e-04) & 0.0437 (5e-04) \\ 
			&  & 128 & 0.0349 (5e-04) & 0.0168 (5e-04) & 0.0266 (4e-04) & 0.0434 (5e-04) \\ 
			&  & 256 & 0.0287 (5e-04) & 0.0152 (5e-04) & 0.0266 (4e-04) & 0.0418 (5e-04) \\  \cline{2-7}
			& \multirow{5}{*}{Reusing} & 16 & 0.0446 (5e-04) & 0.0115 (5e-04) & 0.0215 (4e-04) & 0.033 (5e-04) \\ 
			&  & 32 & 0.0426 (5e-04) & 0.009 (5e-04) & 0.0258 (4e-04) & 0.0349 (5e-04) \\ 
			&  & 64 & 0.0396 (6e-04) & 0.0074 (5e-04) & 0.0277 (4e-04) & 0.0351 (5e-04) \\ 
			&  & 128 & 0.0349 (5e-04) & 0.0057 (4e-04) & 0.0268 (4e-04) & 0.0325 (5e-04) \\ 
			&  & 256 & 0.0287 (5e-04) & 0.0035 (4e-04) & 0.0238 (4e-04) & 0.0273 (5e-04) \\ 
			\hline
			\multirow{10}{*}{0.025} & \multirow{5}{*}{Independent} & 16 & 0.0804 (0.001) & 0.0404 (0.001) & 0.0294 (7e-04) & 0.0698 (0.0011) \\ 
			&  & 32 & 0.0682 (0.001) & 0.036 (0.0011) & 0.0404 (8e-04) & 0.0763 (0.0011) \\ 
			&  & 64 & 0.0575 (0.001) & 0.0329 (0.0011) & 0.0454 (8e-04) & 0.0783 (0.0011) \\ 
			&  & 128 & 0.0413 (9e-04) & 0.0255 (0.001) & 0.0418 (8e-04) & 0.0673 (0.0012) \\ 
			&  & 256 & 0.0255 (7e-04) & 0.0154 (9e-04) & 0.0348 (8e-04) & 0.0502 (0.0012) \\  \cline{2-7}
			& \multirow{5}{*}{Reusing} & 16 & 0.0804 (0.001) & 0.0159 (9e-04) & 0.0414 (8e-04) & 0.0573 (0.001) \\ 
			&  & 32 & 0.0682 (0.001) & 0.0099 (8e-04) & 0.0448 (8e-04) & 0.0547 (0.001) \\ 
			&  & 64 & 0.0575 (0.001) & 0.0049 (6e-04) & 0.0436 (9e-04) & 0.0485 (0.001) \\ 
			&  & 128 & 0.0413 (9e-04) & 0.0026 (4e-04) & 0.0339 (8e-04) & 0.0364 (9e-04) \\ 
			&  & 256 & 0.0255 (7e-04) & 3e-04 (1e-04) & 0.0221 (7e-04) & 0.0224 (7e-04) \\ 
			\hline
			\multirow{10}{*}{0.05} & \multirow{5}{*}{Independent} & 16 & 0.1169 (0.0019) & 0.0591 (0.0021) & 0.0522 (0.0013) & 0.1113 (0.0023) \\ 
			&  & 32 & 0.085 (0.0018) & 0.0514 (0.002) & 0.0598 (0.0015) & 0.1112 (0.0024) \\ 
			&  & 64 & 0.0475 (0.0015) & 0.034 (0.0018) & 0.057 (0.0016) & 0.091 (0.0024) \\ 
			&  & 128 & 0.0215 (0.001) & 0.0127 (0.0012) & 0.0366 (0.0013) & 0.0493 (0.0019) \\ 
			&  & 256 & 0.0054 (5e-04) & 0.0023 (5e-04) & 0.0156 (9e-04) & 0.0179 (0.0011) \\  \cline{2-7}
			& \multirow{5}{*}{Reusing} & 16 & 0.1169 (0.0019) & 0.0126 (0.0012) & 0.0626 (0.0015) & 0.0752 (0.0018) \\ 
			&  & 32 & 0.085 (0.0018) & 0.0037 (6e-04) & 0.0565 (0.0015) & 0.0602 (0.0017) \\ 
			&  & 64 & 0.0475 (0.0015) & 0.0013 (4e-04) & 0.0341 (0.0013) & 0.0354 (0.0013) \\ 
			&  & 128 & 0.0215 (0.001) & 2e-04 (1e-04) & 0.0161 (9e-04) & 0.0162 (9e-04) \\ 
			&  & 256 & 0.0054 (5e-04) & 0 (0) & 0.0041 (4e-04) & 0.0041 (4e-04) \\ 
			\hline
			\multirow{10}{*}{0.075} & \multirow{5}{*}{Independent} & 16 & 0.1273 (0.0025) & 0.0662 (0.0029) & 0.0584 (0.0019) & 0.1246 (0.0034) \\ 
			&  & 32 & 0.0677 (0.002) & 0.0423 (0.0025) & 0.0563 (0.0019) & 0.0986 (0.0034) \\ 
			&  & 64 & 0.0241 (0.0013) & 0.0127 (0.0015) & 0.0355 (0.0016) & 0.0482 (0.0023) \\ 
			&  & 128 & 0.0036 (5e-04) & 0.0013 (4e-04) & 0.0148 (0.001) & 0.0161 (0.0012) \\ 
			&  & 256 & 1e-04 (1e-04) & 0 (0) & 0.0029 (5e-04) & 0.0029 (5e-04) \\  \cline{2-7}
			& \multirow{5}{*}{Reusing} & 16 & 0.1273 (0.0025) & 0.0065 (0.0011) & 0.0605 (0.002) & 0.0671 (0.0022) \\ 
			&  & 32 & 0.0677 (0.002) & 7e-04 (3e-04) & 0.0372 (0.0016) & 0.0379 (0.0016) \\ 
			&  & 64 & 0.0241 (0.0013) & 0 (0) & 0.014 (0.001) & 0.014 (0.001) \\ 
			&  & 128 & 0.0036 (5e-04) & 0 (0) & 0.0026 (4e-04) & 0.0026 (4e-04) \\ 
			&  & 256 & 1e-04 (1e-04) & 0 (0) & 1e-04 (1e-04) & 1e-04 (1e-04) \\ 
			\hline
			\multirow{10}{*}{0.1} & \multirow{5}{*}{Independent} & 16 & 0.1014 (0.0028) & 0.0444 (0.003) & 0.0384 (0.0018) & 0.0828 (0.0038) \\ 
			&  & 32 & 0.0316 (0.0017) & 0.0111 (0.0016) & 0.0304 (0.0017) & 0.0415 (0.0025) \\ 
			&  & 64 & 0.0056 (7e-04) & 0.0016 (5e-04) & 0.0138 (0.0011) & 0.0154 (0.0013) \\ 
			&  & 128 & 1e-04 (1e-04) & 0 (0) & 0.0013 (4e-04) & 0.0013 (4e-04) \\ 
			&  & 256 & 0 (0) & 0 (0) & 1e-04 (1e-04) & 1e-04 (1e-04) \\  \cline{2-7}
			& \multirow{5}{*}{Reusing} & 16 & 0.1014 (0.0028) & 0.001 (5e-04) & 0.0334 (0.0018) & 0.0344 (0.0019) \\ 
			&  & 32 & 0.0316 (0.0017) & 0 (0) & 0.0106 (0.001) & 0.0106 (0.001) \\ 
			&  & 64 & 0.0056 (7e-04) & 0 (0) & 0.0028 (5e-04) & 0.0028 (5e-04) \\ 
			&  & 128 & 1e-04 (1e-04) & 0 (0) & 0 (0) & 0 (0) \\ 
			&  & 256 & 0 (0) & 0 (0) & 0 (0) & 0 (0) \\
			\hline
		\end{longtable}
		\normalsize 
		
		\begin{landscape}
		\tiny
		\begin{longtable}{llrllllllll}
			\hline
			Sce. & $d$ & $n$ & Noise & ERM & LASSO & Ridge & VSLS & PLLS & Steps (PLLS) & Steps (VSLS) \\ 
			\hline
			\multirow{27}{*}{S1} & \multirow{9}{*}{100} & 140 & Easy & 0.000194 (3.75e-06) & 0.00102 (1.3e-05) & 0.0109 (0.000216) & \textbf{3.64e-05 (1.96e-06)} & 8.06e-05 (2.99e-06) & 510 & 323 \\ 
			 &  & 140 & Moderate & 0.0019 (3.24e-05) & 0.00076 (2.68e-05) & 0.0124 (0.000268) & \textbf{0.000386 (2.69e-05)} & 0.000758 (2.53e-05) & 510 & 326 \\ 
			 &  & 140 & Hard & 0.0189 (0.000284) & 0.00735 (0.00025) & 0.0291 (0.000539) & \textbf{0.00357 (0.000183)} & 0.0118 (0.000449) & 505 & 328 \\ 
			 \cline{3-11} &  & 175 & Easy & 0.000137 (2.48e-06) & 0.000981 (1.34e-05) & 0.0079 (0.000144) & \textbf{2.99e-05 (1.6e-06)} & 5.42e-05 (2.23e-06) & 522 & 298 \\ 
			 &  & 175 & Moderate & 0.0013 (2.06e-05) & 0.000633 (2.35e-05) & 0.00964 (0.00018) & \textbf{0.000277 (1.62e-05)} & 0.000533 (2.1e-05) & 535 & 303 \\ 
			 &  & 175 & Hard & 0.0136 (0.000223) & 0.00609 (0.000208) & 0.0215 (0.000375) & \textbf{0.00252 (0.000143)} & 0.00746 (0.000278) & 525 & 295 \\ 
			 \cline{3-11} &  & 210 & Easy & 0.000106 (1.79e-06) & 0.000964 (1.29e-05) & 0.00656 (9.51e-05) & \textbf{2.39e-05 (1.1e-06)} & 4.1e-05 (1.35e-06) & 504 & 271 \\ 
			 &  & 210 & Moderate & 0.00107 (1.9e-05) & 0.000532 (1.74e-05) & 0.00771 (0.000146) & \textbf{0.000232 (1.36e-05)} & 0.00042 (1.53e-05) & 491 & 278 \\ 
			 &  & 210 & Hard & 0.0106 (0.000166) & 0.00525 (0.000177) & 0.0181 (0.000365) & \textbf{0.00241 (0.000148)} & 0.00583 (0.000231) & 502 & 269 \\ 
			 \cline{2-11} &  \multirow{9}{*}{500} & 674 & Easy & 7.82e-06 (6.12e-08) & 0.00559 (3.33e-05) & 0.0145 (0.000194) & 1.94 (0.23) & \textbf{2.92e-06 (5.79e-08)} & 132 & 46 \\ 
			 &  & 674 & Moderate & 7.98e-05 (6.71e-07) & 0.00548 (3.46e-05) & 0.0147 (0.000239) & 2.24 (0.222) & \textbf{7.01e-05 (1.54e-05)} & 127 & 45 \\ 
			 &  & 674 & Hard & 0.000784 (6.35e-06) & 0.00451 (2.9e-05) & 0.0154 (0.000259) & 2.34 (0.246) & \textbf{0.000408 (1.25e-05)} & 139 & 44 \\ 
			 \cline{3-11} &  & 842 & Easy & 5.65e-06 (3.8e-08) & 0.00549 (3.38e-05) & 0.0107 (9.09e-05) & 2.07 (0.222) & \textbf{1.98e-06 (4.17e-08)} & 119 & 43 \\ 
			 &  & 842 & Moderate & 5.65e-05 (4.59e-07) & 0.00543 (3.29e-05) & 0.0105 (9.54e-05) & 1.94 (0.22) & \textbf{3.18e-05 (6.91e-06)} & 116 & 44 \\ 
			 &  & 842 & Hard & 0.000567 (3.88e-06) & 0.00435 (2.77e-05) & 0.0108 (0.000128) & 2.12 (0.227) & \textbf{0.000244 (9.07e-06)} & 132 & 43 \\ 
			 \cline{3-11} &  & 1010 & Easy & \textbf{4.45e-06 (3.14e-08)} & 0.00553 (3.52e-05) & 0.00919 (6.27e-05) & 2.2 (0.219) & 2.94e-05 (1.26e-05) & 113 & 40 \\ 
			 &  & 1010 & Moderate & 4.35e-05 (3.15e-07) & 0.00541 (3.47e-05) & 0.00914 (5.93e-05) & 2.18 (0.215) & \textbf{2.13e-05 (4.24e-06)} & 111 & 40 \\ 
			 &  & 1010 & Hard & 0.00044 (3.09e-06) & 0.00432 (2.87e-05) & 0.00865 (8.66e-05) & 1.93 (0.21) & \textbf{0.000201 (2.09e-05)} & 130 & 42 \\ 
			 \cline{2-11} &  \multirow{9}{*}{1000} & 1340 & Easy & \textbf{1.99e-06 (1.23e-08)} & 0.011 (6.33e-05) & 0.0254 (0.000171) & 8.63 (0.136) & 0.00115 (0.000231) & 115 & 28 \\ 
			 &  & 1340 & Moderate & \textbf{1.99e-05 (1.24e-07)} & 0.011 (6.12e-05) & 0.0254 (0.00015) & 7.78 (0.0936) & 0.000578 (7.53e-05) & 129 & 37 \\ 
			 &  & 1340 & Hard & \textbf{0.000198 (1.22e-06)} & 0.0107 (6.08e-05) & 0.025 (0.000162) & 8.3 (0.121) & 0.000482 (3.43e-05) & 137 & 32 \\ 
			 \cline{3-11} &  & 1675 & Easy & \textbf{1.42e-06 (8.27e-09)} & 0.0109 (6.18e-05) & 0.0205 (0.000133) & 7.76 (0.0961) & 0.000509 (6.3e-05) & 115 & 37 \\ 
			 &  & 1675 & Moderate & \textbf{1.43e-05 (8.43e-08)} & 0.011 (6.87e-05) & 0.0203 (0.000113) & 8.05 (0.101) & 0.000646 (5.96e-05) & 109 & 34 \\ 
			 &  & 1675 & Hard & \textbf{0.000143 (6.83e-07)} & 0.0106 (6.18e-05) & 0.0201 (0.00012) & 7.66 (0.0765) & 0.000524 (3.85e-05) & 118 & 38 \\ 
			 \cline{3-11} &  & 2010 & Easy & \textbf{1.09e-06 (5.97e-09)} & 0.0109 (6.15e-05) & 0.0178 (9.64e-05) & 8.46 (0.0947) & 0.00146 (0.000418) & 95 & 29 \\ 
			 &  & 2010 & Moderate & \textbf{1.1e-05 (5.28e-08)} & 0.0108 (6.2e-05) & 0.0178 (9.5e-05) & 8.52 (0.0961) & 0.000834 (0.000108) & 97 & 28 \\ 
			 &  & 2010 & Hard & \textbf{0.000111 (6.22e-07)} & 0.0106 (6.09e-05) & 0.0175 (9.67e-05) & 8.21 (0.106) & 0.000816 (6.69e-05) & 97 & 30 \\ 
			\hline \multirow{27}{*}{S2} & \multirow{9}{*}{100} & 140 & Easy & 0.00019 (3.37e-06) & 7.71e-05 (2.43e-06) & 0.000363 (7.94e-06) & \textbf{3.59e-05 (2.12e-06)} & 5.57e-05 (3.03e-06) & 557 & 324 \\ 
			 &  & 140 & Moderate & 0.00185 (3.48e-05) & 0.000776 (2.57e-05) & 0.0022 (4.61e-05) & \textbf{0.000333 (1.84e-05)} & 0.000533 (4.1e-05) & 532 & 321 \\ 
			 &  & 140 & Hard & 0.019 (0.000341) & 0.00765 (0.000246) & 0.0158 (0.000373) & \textbf{0.00587 (0.000443)} & 0.0175 (0.000649) & 533 & 338 \\ 
			 \cline{3-11} &  & 175 & Easy & 0.000134 (2.31e-06) & 6.24e-05 (2.12e-06) & 0.000266 (4.99e-06) & \textbf{2.79e-05 (1.3e-06)} & 3.24e-05 (1.81e-06) & 540 & 298 \\ 
			 &  & 175 & Moderate & 0.00131 (2.3e-05) & 0.000621 (1.95e-05) & 0.0016 (2.84e-05) & \textbf{0.000273 (1.4e-05)} & 0.000343 (1.61e-05) & 538 & 303 \\ 
			 &  & 175 & Hard & 0.0135 (0.000231) & 0.00642 (0.00023) & 0.0136 (0.000307) & \textbf{0.00393 (0.000286)} & 0.0117 (0.000433) & 533 & 303 \\ 
			 \cline{3-11} &  & 210 & Easy & 0.000107 (1.92e-06) & 5.11e-05 (1.78e-06) & 0.000215 (3.72e-06) & 2.16e-05 (1.2e-06) & \textbf{2.16e-05 (9.61e-07)} & 537 & 274 \\ 
			 &  & 210 & Moderate & 0.00103 (1.62e-05) & 0.000504 (1.55e-05) & 0.00128 (2.62e-05) & \textbf{0.000236 (1.31e-05)} & 0.000258 (1.58e-05) & 547 & 276 \\ 
			 &  & 210 & Hard & 0.0111 (0.00017) & 0.00561 (0.000212) & 0.0113 (0.000275) & \textbf{0.00299 (0.00025)} & 0.00778 (0.000335) & 531 & 272 \\ 
			 \cline{2-11} &  \multirow{9}{*}{500} & 674 & Easy & 7.96e-06 (6.6e-08) & 2.53e-06 (4.93e-08) & 1.85e-05 (1.98e-07) & 0.00238 (0.000266) & \textbf{7.88e-07 (3.28e-08)} & 173 & 45 \\ 
			 &  & 674 & Moderate & 7.97e-05 (6.27e-07) & 2.54e-05 (4.49e-07) & 8.82e-05 (8.37e-07) & 0.00233 (0.000248) & \textbf{8.4e-06 (3.28e-07)} & 160 & 46 \\ 
			 &  & 674 & Hard & 0.000789 (6.39e-06) & \textbf{0.000249 (3.98e-06)} & 0.000778 (7.51e-06) & 0.00082 (8.13e-05) & 0.000277 (1.28e-05) & 122 & 53 \\ 
			 \cline{3-11} &  & 842 & Easy & 5.64e-06 (4.42e-08) & 2e-06 (3.84e-08) & 1.29e-05 (1.4e-07) & 0.00235 (0.000241) & \textbf{4.52e-07 (1.63e-08)} & 141 & 43 \\ 
			 &  & 842 & Moderate & 5.72e-05 (4.48e-07) & 1.98e-05 (3.9e-07) & 6.49e-05 (6.01e-07) & 0.00186 (0.000206) & \textbf{4.78e-06 (1.97e-07)} & 142 & 46 \\ 
			 &  & 842 & Hard & 0.000567 (4.24e-06) & 0.000195 (3.22e-06) & 0.000584 (5.8e-06) & 0.00179 (0.000196) & \textbf{9.05e-05 (6.86e-06)} & 143 & 46 \\ 
			 \cline{3-11} &  & 1010 & Easy & 4.41e-06 (3.26e-08) & 1.73e-06 (3.28e-08) & 1.03e-05 (1.08e-07) & 0.00208 (0.000214) & \textbf{3.27e-07 (1.25e-08)} & 168 & 43 \\ 
			 &  & 1010 & Moderate & 4.43e-05 (2.89e-07) & 1.7e-05 (2.91e-07) & 5.11e-05 (4.05e-07) & 0.0022 (0.000221) & \textbf{3.62e-06 (1.4e-07)} & 134 & 41 \\ 
			 &  & 1010 & Hard & 0.000436 (3.02e-06) & 0.000169 (3.15e-06) & 0.000472 (4.14e-06) & 0.00201 (0.000192) & \textbf{4e-05 (2.68e-06)} & 129 & 43 \\ 
			 \cline{2-11} &  \multirow{9}{*}{1000} & 1340 & Easy & 2e-06 (1.1e-08) & 7.06e-07 (8.09e-09) & 5.17e-06 (4.59e-08) & 0.00237 (3.88e-05) & \textbf{1.66e-07 (4.49e-09)} & 150 & 30 \\ 
			 &  & 1340 & Moderate & 2.01e-05 (1.29e-07) & 5.9e-06 (7.6e-08) & 2.26e-05 (1.59e-07) & 0.0023 (2.97e-05) & \textbf{1.74e-06 (5.11e-08)} & 145 & 33 \\ 
			 &  & 1340 & Hard & 0.000198 (1.1e-06) & 5.91e-05 (6.83e-07) & 0.000198 (1.22e-06) & 0.00199 (6.8e-06) & \textbf{2.57e-05 (1.47e-06)} & 159 & 43 \\ 
			 \cline{3-11} &  & 1675 & Easy & 1.43e-06 (7.92e-09) & 5.65e-07 (6.89e-09) & 3.54e-06 (2.59e-08) & 0.00223 (2.47e-05) & \textbf{9.53e-08 (2.84e-09)} & 130 & 34 \\ 
			 &  & 1675 & Moderate & 1.42e-05 (7.57e-08) & 4.64e-06 (5.97e-08) & 1.61e-05 (9.63e-08) & 0.00215 (1.68e-05) & \textbf{1.16e-06 (3.78e-08)} & 134 & 37 \\ 
			 &  & 1675 & Hard & 0.000142 (7.27e-07) & 4.63e-05 (6.37e-07) & 0.000146 (1.1e-06) & 0.00225 (2.62e-05) & \textbf{9.25e-06 (2.84e-07)} & 129 & 34 \\ 
			 \cline{3-11} &  & 2010 & Easy & 1.11e-06 (5.74e-09) & 4.64e-07 (6.81e-09) & 2.71e-06 (2.07e-08) & 0.00233 (2.95e-05) & \textbf{6.83e-08 (2.29e-09)} & 118 & 31 \\ 
			 &  & 2010 & Moderate & 1.1e-05 (5.61e-08) & 3.81e-06 (4.82e-08) & 1.27e-05 (7.18e-08) & 0.0024 (2.77e-05) & \textbf{7.98e-07 (2.8e-08)} & 119 & 29 \\ 
			 &  & 2010 & Hard & 0.00011 (5.87e-07) & 3.93e-05 (5.13e-07) & 0.000115 (7.34e-07) & 0.0023 (2.35e-05) & \textbf{6.54e-06 (2.33e-07)} & 121 & 32 \\ 
			\hline \multirow{27}{*}{S3} & \multirow{9}{*}{100} & 140 & Easy & 0.000195 (3.85e-06) & 0.000699 (1.27e-05) & 0.00135 (3.32e-05) & 0.182 (0.0186) & \textbf{5.56e-05 (2.84e-06)} & 490 & 243 \\ 
			 &  & 140 & Moderate & 0.00192 (3.27e-05) & 0.00222 (3.9e-05) & 0.00306 (6.17e-05) & 0.181 (0.0172) & \textbf{0.000615 (4.83e-05)} & 488 & 257 \\ 
			 &  & 140 & Hard & \textbf{0.0186 (0.000352)} & 0.023 (0.000509) & 0.0222 (0.000538) & 0.26 (0.0111) & 0.0361 (0.000994) & 492 & 283 \\ 
			 \cline{3-11} &  & 175 & Easy & 0.000134 (2.27e-06) & 0.000492 (9.92e-06) & 0.000925 (2.19e-05) & 0.15 (0.0155) & \textbf{3.51e-05 (1.76e-06)} & 478 & 221 \\ 
			 &  & 175 & Moderate & 0.00136 (2.34e-05) & 0.00158 (2.83e-05) & 0.00222 (4.26e-05) & 0.147 (0.0152) & \textbf{0.000318 (1.51e-05)} & 481 & 222 \\ 
			 &  & 175 & Hard & \textbf{0.0137 (0.000269)} & 0.0167 (0.000353) & 0.0165 (0.000336) & 0.167 (0.0136) & 0.0235 (0.000776) & 474 & 237 \\ 
			 \cline{3-11} &  & 210 & Easy & 0.000105 (1.69e-06) & 0.000401 (8.08e-06) & 0.000703 (1.92e-05) & 0.0921 (0.0105) & \textbf{2.42e-05 (1.21e-06)} & 478 & 206 \\ 
			 &  & 210 & Moderate & 0.00106 (1.63e-05) & 0.0013 (2.31e-05) & 0.00177 (3.41e-05) & 0.134 (0.0139) & \textbf{0.000256 (1.15e-05)} & 463 & 200 \\ 
			 &  & 210 & Hard & \textbf{0.0105 (0.000179)} & 0.0126 (0.000251) & 0.0129 (0.000268) & 0.16 (0.0136) & 0.0149 (0.000546) & 463 & 210 \\ 
			 \cline{2-11} &  \multirow{9}{*}{500} & 674 & Easy & 7.73e-06 (6.71e-08) & 0.000187 (1.55e-06) & 0.000179 (1.55e-06) & 0.082 (0.000189) & \textbf{8.55e-07 (3.26e-08)} & 149 & 50 \\ 
			 &  & 674 & Moderate & 7.88e-05 (6.16e-07) & 0.000113 (1.11e-06) & 0.00015 (1.77e-06) & 0.082 (0.000187) & \textbf{9.18e-06 (3.74e-07)} & 165 & 46 \\ 
			 &  & 674 & Hard & 0.000781 (6.08e-06) & 0.000835 (7.16e-06) & 0.000858 (7.57e-06) & 0.0819 (0.000199) & \textbf{0.000579 (2.89e-05)} & 182 & 45 \\ 
			 \cline{3-11} &  & 842 & Easy & 5.64e-06 (4.42e-08) & 0.000144 (1.09e-06) & 0.00014 (1.03e-06) & 0.0807 (0.000221) & \textbf{5.39e-07 (2.17e-08)} & 139 & 45 \\ 
			 &  & 842 & Moderate & 5.61e-05 (4.44e-07) & 6.99e-05 (6.9e-07) & 0.000102 (1.06e-06) & 0.081 (0.000255) & \textbf{5.46e-06 (2.27e-07)} & 152 & 42 \\ 
			 &  & 842 & Hard & 0.000567 (3.79e-06) & 0.000608 (4.32e-06) & 0.000622 (5.35e-06) & 0.0809 (0.000219) & \textbf{0.000146 (1.04e-05)} & 159 & 44 \\ 
			 \cline{3-11} &  & 1010 & Easy & 4.37e-06 (3.43e-08) & 0.000124 (8.64e-07) & 0.00012 (7.14e-07) & 0.0802 (0.000262) & \textbf{3.84e-07 (1.45e-08)} & 149 & 41 \\ 
			 &  & 1010 & Moderate & 4.35e-05 (3.27e-07) & 5.35e-05 (4.87e-07) & 7.83e-05 (7.56e-07) & 0.0803 (0.000231) & \textbf{3.69e-06 (1.36e-07)} & 125 & 41 \\ 
			 &  & 1010 & Hard & 0.000443 (3.24e-06) & 0.000485 (3.95e-06) & 0.000494 (3.61e-06) & 0.0806 (0.000237) & \textbf{4.55e-05 (2.5e-06)} & 141 & 42 \\ 
			 \cline{2-11} &  \multirow{9}{*}{1000} & 1340 & Easy & 1.99e-06 (1.24e-08) & 9.76e-05 (6.14e-07) & 9.25e-05 (5.47e-07) & 0.0415 (2.94e-05) & \textbf{1.85e-07 (5.13e-09)} & 143 & 29 \\ 
			 &  & 1340 & Moderate & 1.98e-05 (1.05e-07) & 7.04e-05 (5.68e-07) & 6.64e-05 (4.94e-07) & 0.0415 (2.87e-05) & \textbf{1.86e-06 (5.46e-08)} & 143 & 30 \\ 
			 &  & 1340 & Hard & 0.000199 (1.14e-06) & 0.00021 (1.27e-06) & 0.000216 (1.3e-06) & 0.0415 (2.59e-05) & \textbf{3.78e-05 (2.25e-06)} & 150 & 37 \\ 
			 \cline{3-11} &  & 1675 & Easy & 1.43e-06 (8.2e-09) & 7.6e-05 (4.72e-07) & 7.39e-05 (4.17e-07) & 0.0413 (3.38e-05) & \textbf{1.17e-07 (3.1e-09)} & 126 & 34 \\ 
			 &  & 1675 & Moderate & 1.43e-05 (6.77e-08) & 4.95e-05 (4.11e-07) & 4.82e-05 (3.64e-07) & 0.0413 (3.03e-05) & \textbf{1.17e-06 (3.22e-08)} & 127 & 34 \\ 
			 &  & 1675 & Hard & 0.000142 (8.3e-07) & 0.000152 (9.41e-07) & 0.000156 (9.02e-07) & 0.0412 (3.39e-05) & \textbf{1.18e-05 (3.56e-07)} & 138 & 38 \\ 
			 \cline{3-11} &  & 2010 & Easy & 1.11e-06 (5.76e-09) & 6.63e-05 (3.84e-07) & 6.42e-05 (3.86e-07) & 0.0411 (3.37e-05) & \textbf{8.51e-08 (2.06e-09)} & 122 & 34 \\ 
			 &  & 2010 & Moderate & 1.11e-05 (5.65e-08) & 3.89e-05 (2.72e-07) & 3.83e-05 (2.85e-07) & 0.0412 (3.62e-05) & \textbf{7.98e-07 (2.28e-08)} & 114 & 27 \\ 
			 &  & 2010 & Hard & 0.00011 (5.69e-07) & 0.000119 (7.09e-07) & 0.000122 (7.83e-07) & 0.0412 (3.48e-05) & \textbf{7.73e-06 (2e-07)} & 116 & 30 \\ 
			\hline \multirow{27}{*}{S4} & \multirow{9}{*}{100} & 140 & Easy & \textbf{1.85e-06 (3.3e-08)} & 0.0259 (0.000381) & 0.034 (0.000825) & 5.61 (0.579) & 0.0167 (0.00226) & 257 & 238 \\ 
			 &  & 140 & Moderate & \textbf{1.87e-05 (3.43e-07)} & 0.0261 (0.000381) & 0.0354 (0.00112) & 5.47 (0.591) & 0.0161 (0.00263) & 252 & 242 \\ 
			 &  & 140 & Hard & \textbf{0.000189 (3.25e-06)} & 0.0261 (0.00042) & 0.0347 (0.000884) & 5.3 (0.578) & 0.0236 (0.00341) & 254 & 240 \\ 
			 \cline{3-11} &  & 175 & Easy & \textbf{1.32e-06 (2.12e-08)} & 0.0201 (0.000262) & 0.023 (0.000497) & 4.22 (0.528) & 0.00333 (0.00055) & 223 & 220 \\ 
			 &  & 175 & Moderate & \textbf{1.33e-05 (2.16e-07)} & 0.0203 (0.00028) & 0.0229 (0.000447) & 5.12 (0.524) & 0.00273 (0.000391) & 220 & 222 \\ 
			 &  & 175 & Hard & \textbf{0.000137 (2.4e-06)} & 0.0204 (0.000267) & 0.0226 (0.000402) & 5.12 (0.526) & 0.00522 (0.00066) & 220 & 221 \\ 
			 \cline{3-11} &  & 210 & Easy & \textbf{1.05e-06 (1.59e-08)} & 0.0178 (0.0002) & 0.0191 (0.000286) & 3.9 (0.494) & 0.00117 (0.000167) & 203 & 201 \\ 
			 &  & 210 & Moderate & \textbf{1.04e-05 (1.76e-07)} & 0.0178 (0.000206) & 0.0185 (0.000284) & 4.37 (0.514) & 0.000611 (9.44e-05) & 209 & 199 \\ 
			 &  & 210 & Hard & \textbf{0.000107 (1.8e-06)} & 0.0175 (0.000189) & 0.019 (0.000324) & 2.49 (0.354) & 0.00125 (0.000129) & 198 & 207 \\ 
			 \cline{2-11} &  \multirow{9}{*}{500} & 674 & Easy & \textbf{7.89e-08 (6.29e-10)} & 0.13 (0.00099) & 0.127 (0.000875) & 55.5 (0.0734) & 0.274 (0.00842) & 121 & 49 \\ 
			 &  & 674 & Moderate & \textbf{8.03e-07 (6.6e-09)} & 0.131 (0.00101) & 0.126 (0.000992) & 55.7 (0.0624) & 0.285 (0.00768) & 117 & 48 \\ 
			 &  & 674 & Hard & \textbf{7.77e-06 (6.39e-08)} & 0.13 (0.00104) & 0.127 (0.00099) & 55.7 (0.054) & 0.242 (0.00641) & 142 & 45 \\ 
			 \cline{3-11} &  & 842 & Easy & \textbf{5.67e-08 (4.31e-10)} & 0.103 (0.000731) & 0.102 (0.000565) & 55.2 (0.105) & 0.129 (0.00497) & 116 & 46 \\ 
			 &  & 842 & Moderate & \textbf{5.7e-07 (4.44e-09)} & 0.103 (0.000822) & 0.103 (0.000633) & 55.2 (0.112) & 0.135 (0.00479) & 105 & 46 \\ 
			 &  & 842 & Hard & \textbf{5.66e-06 (3.71e-08)} & 0.103 (0.000696) & 0.103 (0.000668) & 55.3 (0.0924) & 0.0938 (0.00171) & 142 & 44 \\ 
			 \cline{3-11} &  & 1010 & Easy & \textbf{4.42e-08 (3.39e-10)} & 0.0899 (0.000624) & 0.0898 (0.000521) & 55 (0.0963) & 0.0761 (0.00311) & 106 & 42 \\ 
			 &  & 1010 & Moderate & \textbf{4.47e-07 (3.21e-09)} & 0.0897 (0.000568) & 0.0889 (0.000484) & 54.8 (0.113) & 0.0709 (0.00261) & 110 & 39 \\ 
			 &  & 1010 & Hard & \textbf{4.41e-06 (3.27e-08)} & 0.0909 (0.000613) & 0.0887 (0.000563) & 54.8 (0.119) & 0.0696 (0.00272) & 114 & 41 \\ 
			 \cline{2-11} &  \multirow{9}{*}{1000} & 1340 & Easy & \textbf{2e-08 (1.41e-10)} & 0.265 (0.00185) & 0.255 (0.00175) & 111 (0.057) & 0.569 (0.0095) & 136 & 38 \\ 
			 &  & 1340 & Moderate & \textbf{1.97e-07 (1.06e-09)} & 0.265 (0.00185) & 0.255 (0.00158) & 111 (0.0683) & 0.597 (0.0138) & 129 & 33 \\ 
			 &  & 1340 & Hard & \textbf{2e-06 (1.22e-08)} & 0.266 (0.00187) & 0.257 (0.00173) & 111 (0.046) & 0.633 (0.00976) & 110 & 28 \\ 
			 \cline{3-11} &  & 1675 & Easy & \textbf{1.43e-08 (7.54e-11)} & 0.209 (0.0012) & 0.205 (0.00133) & 111 (0.0712) & 0.27 (0.00401) & 124 & 33 \\ 
			 &  & 1675 & Moderate & \textbf{1.4e-07 (7.9e-10)} & 0.213 (0.00139) & 0.206 (0.00121) & 111 (0.0658) & 0.271 (0.00374) & 122 & 34 \\ 
			 &  & 1675 & Hard & \textbf{1.41e-06 (7.47e-09)} & 0.209 (0.0013) & 0.202 (0.00117) & 111 (0.0686) & 0.284 (0.00539) & 115 & 34 \\ 
			 \cline{3-11} &  & 2010 & Easy & \textbf{1.11e-08 (5.17e-11)} & 0.182 (0.00117) & 0.179 (0.00108) & 111 (0.0701) & 0.162 (0.00241) & 112 & 31 \\ 
			 &  & 2010 & Moderate & \textbf{1.11e-07 (5.52e-10)} & 0.183 (0.00112) & 0.178 (0.00102) & 111 (0.0842) & 0.177 (0.00429) & 106 & 29 \\ 
			 &  & 2010 & Hard & \textbf{1.1e-06 (5.87e-09)} & 0.181 (0.00112) & 0.178 (0.000988) & 111 (0.0678) & 0.183 (0.00557) & 103 & 26 \\ 
			\hline
			\caption{\normalsize Mean (SE) of Type IV estimation error by method, across replications, for each scenario, sample size $n$, and noise level. The lowest (best) Type IV error in each row is shown in bold. The last two columns report the mean number of search steps taken by VSLS and PLLS within the allotted computation time of 120 seconds.} 
			\label{tab.typeIV}
		\end{longtable}
	\end{landscape}
	\normalsize	
		
	\section*{Acknowledgments}
	
	D. Marcondes was funded by grant \#2022/06211-2, São Paulo Research Foundation (FAPESP), during the writing of this paper. We thank J. Barrera and A. Simonis for many fruitful conversations about Learning Spaces. During the preparation of this work, the authors used Claude AI to assist in writing and refining the code of the simulation in Section \ref{SecRegression}. After using this tool, the authors reviewed and carefully edited the code, verified its output, and take full responsibility for the content of this publication. This work was supported by computational resources provided by the Australian Government through the National Computational Infrastructure (NCI) under the ANU Startup Scheme.
	
	\bibliographystyle{plain}
	\bibliography{ref}
\end{document}